\documentclass{article} 
\usepackage{iclr2027_conference,times}

\usepackage{amsmath,amsfonts,bm}

\def\eqref#1{equation~\ref{#1}}

\def\1{\bm{1}}

\DeclareMathAlphabet{\mathsfit}{\encodingdefault}{\sfdefault}{m}{sl}
\SetMathAlphabet{\mathsfit}{bold}{\encodingdefault}{\sfdefault}{bx}{n}

\usepackage{hyperref}
\usepackage{url}

\usepackage{graphicx}     
\usepackage{booktabs}     
\usepackage{multirow}

\usepackage{amsmath}
\usepackage{amssymb}      
\usepackage{amsthm}       
\usepackage{thm-restate}  
\usepackage{mathtools}

\usepackage{algorithm}
\usepackage{algorithmic}  

\usepackage{capt-of}      

\newtheorem{theorem}{Theorem}[section]
\newtheorem{proposition}[theorem]{Proposition}

\theoremstyle{definition}

\theoremstyle{remark}

\newcommand{\theoremfootnote}[1]{\footnote{#1}}

\definecolor{cite_color}{HTML}{114083}
\definecolor{link_color}{RGB}{153, 0,0}  
\definecolor{url_color}{RGB}{153, 102,  0}
\definecolor{emp_color}{RGB}{0,0,255}
\hypersetup{
	colorlinks=true,
	linkcolor=link_color,
	citecolor=cite_color,
	urlcolor=black,
	linktoc=section
}

\usepackage{etoc}
\etocdepthtag.toc{mtchapter}
\etocsettagdepth{mtchapter}{subsection}
\etocsettagdepth{mtappendix}{none}

\title{Advancing Optimal Subset Oracle via \\ Learning Relaxation of Neural Set Functions}

\author{
	Yongquan Shi$^{1}$, Zijing Ou$^{2}$, Shiping Wang$^{1}$,  Yatao Bian$^{3}$ \\
	$^1$Fuzhou University, \ $^2$Imperial College London, \ $^3$National University of Singapore \\
	\texttt{losparksayoji@outlook.com} 
}

\iclrfinalcopy 
\begin{document}

\maketitle
\lhead{Preprint}

\begin{abstract}
	Learning neural set functions is pivotal to a wide range of important
	applications, including compound selection in AI-driven drug discovery
	and product recommendation.
	Recent work has introduced optimal subset oracles to implicitly learn
	set functions under practical weakly supervised settings, where model
	parameters are optimized through mean-field variational inference.
	However, these frameworks rely on Monte Carlo sampling to estimate
	gradients of the evidence lower bound when updating the variational
	distribution.
	Repeated sampling across iterations incurs substantial computational
	overhead, while the resulting stochasticity can destabilize the
	optimization trajectory.
	In this work, we reinterpret the evidence lower bound as a continuous
	relaxation of the set function and learn a surrogate objective that
	replaces sampling-based gradient estimation during variational
	optimization.
	The learned surrogate provides stable and efficient gradients
	throughout the continuous domain, thereby reducing computational
	overhead and accelerating inference.
	Furthermore, we establish an approximation guarantee for the proposed
	framework under submodular maximization and characterize its connection
	to variational free energy.
	Experiments on a variety of real-world tasks demonstrate consistent
	improvements over existing baselines.
\end{abstract}

\section{Introduction}
Set-value prediction has a wide range of applications in real-world scenarios and plays a crucial role in many tasks.
For example, recommendation systems select products that are likely to
interest a user \citep{coppolillo2024relevance}, anomaly detection identifies outliers from the
majority of observations \citep{zhang2020set}, and AI-driven drug discovery prioritizes
promising compounds from large candidate databases \citep{gimeno2019light}.
These tasks require explicitly or implicitly learning a set function
\citep{rezatofighi2017deepsetnet, zaheer2017deep} that assigns a utility value to each candidate subset, with more
desirable subsets receiving higher values.

More formally, our objective is to select an optimal subset $S^*$ from a given large ground set $V$, such that it attains the highest utility value among all candidate subsets according to a set function $F_\theta(S;V)$ parameterized by $\theta$. This process can be understood as optimizing the following criteria:
\begin{equation}\label{argmax F}
	S^* = \operatorname*{arg\,max}_{S \in 2^V} F_{\theta}(S; V).
\end{equation}
A direct approach is to learn $F_\theta(S;V)$ in a supervised manner from tuples ${(V_i,S_i,U_i)}_{i=1}^N$, where $U_i$ denotes the utility of the candidate subset $S_i$, a setting commonly referred to as a function-value (FV) oracle \citep{balcan2018submodular}.
However, this training paradigm is often prohibitively expensive, as it requires collecting utility annotations for a sufficiently large and diverse set of candidate subsets \citep{ou2022learning}.
To overcome this limitation, an alternative approach to optimizing objective (\ref{argmax F}) is to implicitly learn the set function from a probabilistic perspective \citep{tschiatschek2018differentiable}.
This approach estimates the parameter $\theta$ in a supervised manner using pairs $\{(V_i, S^*_i)\}_{i=1}^N$, where $S^*_i$ denotes the optimal subset corresponding to $V_i$, serving as an optimal subset (OS) oracle. 
With limited data sampled from the underlying distribution $\mathbb{P}(S, V)$, the OS oracle learns latent patterns by maximizing the empirical log-likelihood $\sum_{i=1}^{N} \log p_\theta(S_i^* | V_i)$ over the observed data.
Compared with the FV oracle, the OS oracle is generally more practical, as it avoids explicit utility labeling for candidate subsets, thereby reducing annotation costs in practice.

Although the OS oracle is conceptually appealing, the distribution $p(S | V)$ is generally intractable in practice. 
Consequently, most existing OS oracles adopt a variational inference framework \citep{blei2017variational}, in which a variational distribution $q_{\bm{\psi}}(S|V)$ is fitted to approximate the target distribution. 
Here, $\bm{\psi}\in[0,1]^{|V|}$ parameterizes $|V|$ independent Bernoulli
variables, with $\psi_i$ representing the probability that element $i$
is included in the predicted subset $S$.
Within this framework, optimizing the variational distribution requires maximizing the \textit{evidence lower bound} (ELBO), which is equivalent to minimizing the divergence between the two distributions.
Unfortunately, exact evaluation of the ELBO gradient is computationally prohibitive, as it involves an exponential summation over all possible subsets $S$. When the ground set $V$ becomes moderately large (e.g., $|V| \geq 30$), enumerating the entire $2^{|V|}$-subset space becomes intractable.

\begin{figure}[!t]
	\centering
	\includegraphics[width=1.0\linewidth]{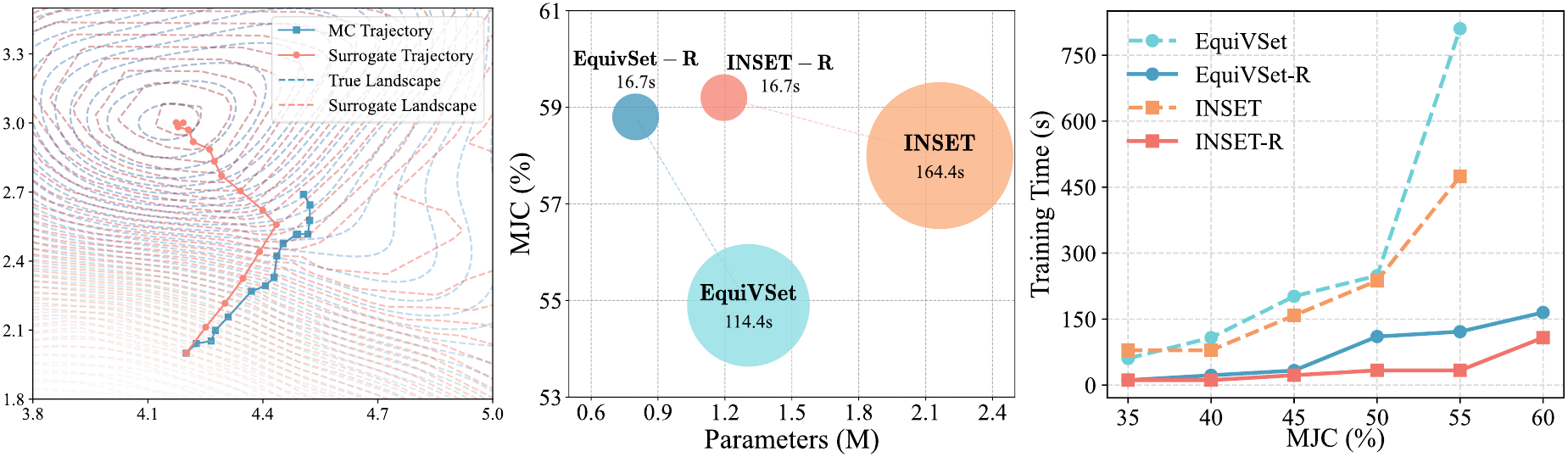}
	\caption{
		\textbf{(left)} Optimization trajectories over the same number of steps on a black-box objective: Monte Carlo sampling (blue) introduces stochastic noise, while surrogate-based optimization (red) converges faster and more smoothly.
		\textbf{(middle)} Comparison of different baselines on the CelebA dataset in terms of performance, per-iteration training time, and model size.
		\textbf{(right)} Time required by different baselines to reach the same validation performance on the CelebA dataset.
		Here, ``$-\operatorname{R}$'' denotes variants that replace Monte Carlo sampling with a learned set function relaxation.
	}
	\label{intro_fig}
	\vspace{-5mm}
\end{figure}

Existing approaches attempt to address this challenge through the Monte Carlo (MC) sampling-based approximation \citep{xie2024enhancing, xie2024horse, ozcan2025learning}; however, this strategy comes with limitations.
Accompanied by the unrolled multi-step gradient ascent, extensive sampling incurs substantial computational overhead. Meanwhile, the induced stochasticity destabilizes the optimization trajectory, making the model prone to suboptimal solutions.

In practice, the ELBO can be viewed as a relaxation of the underlying set function. 
Motivated by this, our core idea is to learn a differentiable surrogate objective that replaces sampling-based ELBO gradient estimation during variational optimization.
In this paper, we propose to learn a \textbf{Re}laxation of \textbf{Set} functions ($\operatorname{ReSet}$) within optimal subset oracles, thereby replacing MC–based approximations with a flexible and learnable paradigm.
By leveraging the learned relaxation,
$\operatorname{ReSet}$ substantially reduces optimization overhead and
enables faster updates of the variational distribution.
We provide illustrative comparisons in Figure~\ref{intro_fig} to demonstrate this.
This optimization process can be interpreted as a greedy strategy for continuous submodular maximization, and we derive an approximation ratio that provides a theoretical lower-bound guarantee.
Furthermore, we explore the intrinsic relationship between the ELBO and variational free energy, revealing that our method offers advantages beyond merely optimizing a fixed objective by adaptively learning the trade-off between expected energy and entropy.
Our main contributions can be summarized as follows:
\begin{itemize}
	\item We propose to learn a set function relaxation that replaces sampling-based ELBO gradient estimation, enabling efficient and stable learning of neural set functions.
	
	\item We guarantee that the proposed method converges to a stationary point and derive a constant approximation ratio for the solution, ensuring theoretical soundness.
	
	\item We revisit the variational inference process through the lens of free energy, demonstrating that our method is able to learn an adaptive trade-off between expected energy and entropy.
	
	\item Experimental results indicate that the proposed method consistently outperforms baseline methods across diverse real-world tasks.
\end{itemize}

\section{Preliminaries}\label{Preliminary}
We begin by introducing the optimization framework for optimal subset oracles.
From a probabilistic perspective, neural set function learning under the OS oracle can be formulated as a maximum likelihood problem \citep{stigler1990history}:
\begin{equation}
	\label{OS oracle}
	\begin{aligned}
		&\arg\max_{\theta} \mathbb{E}_{\mathbb{P}(V,S)}[\log p_{\theta}(S | V)], \\
		&\text{s. t. } p_{\theta}(S | V) \propto \exp(F_{\theta}(S ; V)), \forall S \in 2^{V},
	\end{aligned}
\end{equation}
where $p_{\theta}(S | V)$ denotes the target set distribution we aim to identify.
By treating the utility function $F_{\theta}(S;V)$ as the negative energy, the set distribution can be formulated via an energy-based model:
\begin{equation}\label{EBM}
	p_{\theta}(S | V)  =  \frac{\exp(F_{\theta}(S;V))}{Z}, 
	\qquad
	Z  =  \sum_{S' \subseteq V} \exp(F_{\theta}(S';V)).
\end{equation}
Such formulation follows the principle of maximum entropy \citep{guiasu1985principle}, thereby embodying the least prior assumptions. This endows the model with the ability to maintain the least possible bias toward unknown information, consistent with the notion of \textit{non-informative priors} in Bayesian modeling \citep{jeffreys1946invariant}.
Although this approach is conceptually attractive, learning the parameters $\theta$ is notoriously difficult.
Following \citet{ou2022learning}, the training of energy-based model is cast in a variational inference framework, where a variational distribution $q(\bm{\psi})$ is introduced as an approximation to the set mass function $p_\theta(S|V)$, the fitting target can be expressed as\footnote{We use the shorthand $q(\bm{\psi})$ to denote $q_{\bm{\psi}}(S | V)$ here.}
\begin{equation}\label{variational approximation}
	\bm{\psi}^* = \mathop {\arg\min}\limits_{\bm{\psi}} D[q(\bm{\psi}) || p_\theta(S | V)],
\end{equation}
where $D(\cdot||\cdot)$ is a discrepancy measure between two distributions, $\bm{\psi} \in [0, 1]^{|V|}$ represents the probability that each element $s \in V$ is included in the optimal subset $S^*$.
By achieving this objective, the variational distribution $q(\bm{\psi})$ can be regarded as a function of $\theta$, which allows the parameter $\theta$ to be effectively learned in a straightforward manner through the cross-entropy loss:
\begin{equation}\label{cross-entropy loss}
	\begin{split}
		\mathcal{L}(\theta; \bm{\psi}^*) &= \mathbb{E}_{\mathbb P(V,S)}[-\log q(\bm{\psi}^*)] \\
		&\approx   -\sum_{i \in S^{*}} \log \psi_{i}^{*} - \sum_{i \in V \setminus S^{*}} \log (1 - \psi_{i}^{*}).
	\end{split}
\end{equation}
In this way, the entire framework can be trained in a cooperative learning fashion \citep{xie2018cooperative}: with $\theta$ fixed, the variational distribution $q(\bm{\psi})$ is optimized to approximate the energy model $p_\theta(S | V)$ via (\ref{variational approximation}); then, $\theta$ is updated through gradient flow on the cross-entropy loss in (\ref{cross-entropy loss}).

\section{Method}
Although Section~\ref{Preliminary} outlines a feasible framework for learning set functions under energy-based modeling, the resulting optimization remains nontrivial.
In this section, we show how this challenge can be effectively addressed through a straightforward end-to-end learnable paradigm.
\subsection{Problem Setup}
In order to optimize objective (\ref{variational approximation}), we need to specify both the divergence measure $D(\cdot||\cdot)$ and the form of the variational distribution $q(\bm{\psi})$, while ensuring that $\bm{\psi}^*$ remains differentiable with respect to $\theta$.
A common choice is to define the variational distribution $q(\bm{\psi})$ as a product of $|V|$ independent Bernoulli variables, \textit{i.e.}, $q(\bm{\psi})=\prod_{i\in S}\psi_i
\prod_{i\in V\setminus S}(1-\psi_i)$, which implies that $q(\bm{\psi})$ is fully factorized and can be regarded as a mean-field approximation of $p_\theta(S | V)$. By further choosing the Kullback-Leibler divergence as the measure, the objective (\ref{variational approximation}) can be reformulated as maximizing the evidence lower bound (ELBO):
\begin{equation}\label{ELBO}
	\min_{\psi} \mathbb{KL}(q(\bm{\psi})||p_{\theta}(S|V)) \iff \max_{\bm{\psi}} \underbrace{f_{\text{mt}}^{F_\theta}(\bm{\psi}) + \mathbb{H}(q(\bm{\psi}))}_{\mathrm{ELBO}},
\end{equation}
where $f_{\text{mt}}^{F_\theta}(\bm{\psi})$ denotes the multilinear extension of $F_\theta(S)$ \citep{calinescu2007maximizing}, defined as
\begin{equation}\label{f_mt}
	f_{\text{mt}}^{F_\theta}(\bm{\psi}):= \sum_{S \subseteq V} F_{\theta}(S) \prod_{i \in S} \psi_{i} \prod_{i \notin S} (1 - \psi_{i}),
\end{equation}
and $\mathbb{H}(\cdot)$ is the entropy of a given distribution.
To maximize the ELBO in (\ref{ELBO}), a straightforward approach is to apply fixed-point iteration, where the stationary condition is
$\psi_i = \sigma ( \nabla_{\psi_i} f_{\text{mt}}^{F_{\theta}}(\boldsymbol{\psi}))$\footnote{The detailed derivation of this stationary condition is provided in Appendix~\ref{appendix:Derivations of the Fixed Point Iteration}.},
here $\sigma(\cdot)$ denotes the sigmoid function. 
In this way, $\bm{\psi}$ can be updated via the following procedure:
\begin{align}\label{MFVI}
	\boldsymbol{\psi}^{(0)} &\leftarrow \text{Initialize in } [0, 1]^{|V|},\\
	\label{FPI}
	\boldsymbol{\psi}^{(k)} &\leftarrow (1 + \exp(-\nabla_{\boldsymbol{\psi}^{(k-1)}} f_{\text{mt}}^{F_{\theta}}(\boldsymbol{\psi}^{(k-1)})))^{-1}, \\
	\boldsymbol{\psi}^\ast &\leftarrow \boldsymbol{\psi}^{(K)},
\end{align}
which was referred to as differentiable mean-field variational inference by \citet{ou2022learning}.

The crux of the matter is that computing the derivative of the multilinear extension $\nabla f_{\text{mt}}^{F_{\theta}}(\boldsymbol{\psi})$ entails summing over all possible subsets, which is computationally intractable.
Existing methods \citep{xie2024enhancing, xie2024horse, ozcan2025learning} approximate using Monte Carlo (MC) sampling\footnote{Further details on gradient approximation via Monte Carlo sampling are provided in Appendix~\ref{appendix:Details of Monte Carlo Gradient Estimation}.}:
\begin{equation}
	\nabla_{\psi_i} f_{\text{mt}}^{F_{\theta}}(\boldsymbol{\psi}) \approx  \frac{1}{M} \sum_{m=1}^M \left[ F_{\theta}(S^{(m)} + i) -  F_{\theta}(S^{(m)}) \right],
\end{equation}
where $S ^{(m)}$ is the $m$-th independent sample from $q(S;(\boldsymbol{\psi}|\psi_i \leftarrow 0))$, and $S^{(m)} + i$ denotes the union of the set $S^{(m)} \cup {i}$. 
However, the reliance on extensive sampling incurs substantial computational overhead, and the induced stochasticity can further destabilize the optimization trajectory.

\subsection{Learning Adaptive Set Function Relaxation}
During variational optimization, the entropy term depends only
on $\bm{\psi}$, and its gradient can be evaluated in closed form.
For a fixed $\bm{\psi}$, the multilinear extension
$f_{\mathrm{mt}}^{F_\theta}(\bm{\psi})$ depends only on $\theta$,
while evaluating its gradient
remains the computational bottleneck.
Motivated by this observation, our core idea is to learn a
differentiable surrogate objective
$g_\theta:[0,1]^{|V|}\to\mathbb{R}_{+}$
that plays the role of the multilinear extension, replacing
sampling-based ELBO gradient estimation during variational optimization.\footnote{Note that we do not introduce an additional neural network here. The original OS framework already uses a network to approximate $F_\theta$ for MC-based gradient estimation, while our method directly learns the relaxation.} 
The surrogate enables stable gradient computation throughout the continuous domain without requiring repeated MC estimation at each optimization step.
Notably, constructing this surrogate function is easy to implement in practice.
According to \citet{karalias2022neural}, the extensions of set function, including $g_\theta(\cdot)$ here, can be modeled using simple neural networks.
Consequently, the variational parameter $\bm{\psi}$ can be effectively optimized through gradient ascent
\begin{equation}\label{gradient descent}
	\bm{\psi}^{(k+1)}
	=
	\bm{\psi}^{(k)}+\alpha\left(
	\nabla g_\theta(\bm{\psi}^{(k)})
	+\log\frac{\bm{1}-\bm{\psi}^{(k)}}{\bm{\psi}^{(k)}}
	\right),
\end{equation}
where $\alpha$ is the step size. 
While the above equation illustrates a basic setting, in implementation, the continuous optimization can be carried out using different methods, such as projected gradient ascent or continuous greedy \citep{calinescu2007maximizing}.
Subsequently, the parameters $\theta$ can be updated through the marginal-based loss \citep{domke2013learning} defined in Equation (\ref{cross-entropy loss}).
Note that, in this case, the loss is propagated through the entire optimization process, which involves computing second-order derivatives (i.e., gradients of gradients). These derivatives can be efficiently evaluated using Hessian–vector products, which scale linearly with the model size \citep{dagreou2024compute}, similar to standard first-order backpropagation in feed-forward architectures.
A brief description of our method is given in Algorithm~\ref{algorithm1}, while a more comprehensive procedure is detailed in Appendix~\ref{appendix:Detailed Pseudo Code of ReSet}.

\begin{algorithm}[t!]
	\caption{Variational Optimization with $\operatorname{ReSet}$}
	\label{algorithm1}
	\begin{algorithmic}[1]
		\STATE Initialize the variational parameter via the auxiliary recognition network\\
		$\bm{\psi}^{(0)} \leftarrow \mathrm{RecNet}_{\phi}(V)$
		\FOR{$k = 0$ to $K-1$}
		\STATE Compute the surrogate $g_{\theta}(\bm{\psi}^{(k)})$ with Eq.~(\ref{Network Architecture})
		
		\STATE Perform gradient ascent on $\bm{\psi}$ using Eq.~(\ref{add noise})\\
		
		$\bm{\psi}^{(k+1)} \leftarrow \bm{\psi}^{(k)} + \alpha(
		\nabla g_\theta(\bm{\psi}^{(k)})
		+\log\frac{\bm{1}-\bm{\psi}^{(k)}}{\bm{\psi}^{(k)}}
		) + \bm{\epsilon}_k$
		\ENDFOR
		\STATE Set $\bm{\psi}^{*} \leftarrow \bm{\psi}^{(K)}$
	\end{algorithmic}
\end{algorithm}

In fact, maximizing the surrogate leads to the following
continuous optimization problem:
\begin{equation}
	\max_{\bm{\psi}} g_\theta(\bm{\psi}), \quad \text{s.t.} \quad \bm{\psi} \in \mathcal{X},
\end{equation}
where the constraint set $\mathcal{X}$ is a box constraint, forming a standard compact convex set
\begin{equation}
	\mathcal{X} =  \left\{\bm{\psi} \in \mathbb{R}^{|V|} \mid 0 \le \psi_i \le 1, \forall i \in \{1, \dots, |V|\} \right\}.
\end{equation}
The gradient $\nabla g_\theta(\cdot)$ of the differentiable surrogate
objective provides continuous surrogate marginal-gain signals,
indicating which elements (probability components) should be increased to improve the objective.
The entropy term in Eq.~(\ref{gradient descent}) balances these
signals with uncertainty in the variational distribution.
Since enforcing exact submodularity is difficult with neural parameterizations \citep{dolhansky2016deep}, we instead consider weak DR-submodularity, under which the learned relaxation admits the following guarantees \citep{pedramfar2024unifiedapproachmaximizingcontinuous, hassani2017gradient}.
\begin{restatable}{theorem}{MainTheorem} \label{theorem1}
	Suppose the relaxation surrogate
	$g_\theta:[0,1]^{|V|}\to\mathbb{R}_{+}$ is $L$-smooth\theoremfootnote{$L$-smoothness is a standard regularity assumption widely adopted in first-order optimization and continuous DR-submodular maximization.}
	and $\gamma$-weakly DR-submodular \theoremfootnote{This assumption holds in a broad range of practical settings, as \(\gamma \in (0,1]\) provides a continuous measure of the degree of submodularity, with \(\gamma=1\) corresponding to exact DR-submodularity.}.
	Then the following results hold:
	
	1) For a solution $\bm{\psi}^* \in [0, 1]^{|V|}$ of continuous greedy, we have
	\begin{equation}
		g_\theta(\bm{\psi}^*) \ge \gamma e^{-\gamma} g_\theta(\bm{\psi}_{\mathrm{OPT}}),
	\end{equation}
	where $g_\theta(\bm{\psi}_{\mathrm{OPT}})
	= \max_{\bm{\psi}} g_\theta(\bm{\psi})$ is the global optimum.
	
	2) Projected gradient ascent with a fixed step size $0<\alpha<\frac{1}{L}$ converges to first-order stationarity.
\end{restatable}
The above theorem provides theoretical support for our learning
framework. It establishes convergence to first-order stationarity
and shows that the relaxed objective can be optimized with a
constant-factor approximation guarantee of $\gamma e^{-\gamma}$.
In practice, we can refine the continuous-greedy solution by projected gradient ascent, preserving this approximation guarantee while converging to first-order stationarity.
With entropy regularization, the utility bound includes an
entropy-dependent correction, as detailed in
Appendix~\ref{appendix:Entropy-Regularized Refinement}.
As the final solution \(\bm{\psi}^*\) is typically fractional, we can obtain a discrete solution by sampling \(S^*\sim q(\bm{\psi}^*)\). Define the induced set function
$
F_\theta(S):=g_\theta(\mathbf{1}_S),
$
where \(\mathbf{1}_S\) is the indicator vector of \(S\). The definition of the multilinear extension gives
\begin{equation}
	\mathbb{E}_{S^*\sim q(\bm{\psi}^*)}[F_\theta(S^*)]
	= f_{\mathrm{mt}}^{F_\theta}(\bm{\psi}^*),
\end{equation}
and when directional consistency between $g_\theta$ and the
multilinear objective holds along the continuous-greedy trajectory,
the resulting discrete solution satisfies\footnote{That is, the directions induced by $g_\theta$ also maximize the linearized multilinear objective; see Appendix~\ref{appendix:Derivation of the Discrete Approximation Guarantee} for the derivation.}
\begin{equation}
	\mathbb{E}_{S^*\sim q(\bm{\psi}^*)}[F_\theta(S^*)]
	\ge \gamma e^{-\gamma}F_\theta(S_{\mathrm{OPT}}),
	\label{eq:discrete_bound_with_solution_error}
\end{equation}
where $S_{\mathrm{OPT}}\in
\arg\max_{S\subseteq V}F_\theta(S)$.

\subsection{Connection to Variational Free Energy}
The energy function associated with Eq.~(\ref{EBM}) evaluates the
plausibility of a specific state, reflecting a localized perspective.
In contrast, the Helmholtz free energy provides a global measure:
\begin{equation}\label{Helmholtz free energy}
	\mathcal{F}_{\mathrm{H}}=-k_BT\log Z_T,
\end{equation}
where $k_B$ is the Boltzmann constant, $T>0$ denotes the absolute temperature,
and
$Z_T$
is the partition function, with
$E_\theta(S;V)=-F_\theta(S;V)$.
It aggregates all possible subsets of $V$.
At unit temperature, $k_BT=1$, $Z_T$ coincides with the partition
function $Z$ in Eq.~(\ref{EBM}).
Within the variational framework, for a given ground set $V$ and
a fixed temperature $T$, the corresponding variational free energy
is defined as
\begin{equation}\label{reformulated Helmholtz free energy}
	\mathcal{F}_\theta(\bm{\psi})
	=
	\mathbb{E}_{S\sim q(\bm{\psi})}
	\left[E_\theta(S;V)\right]
	+
	k_BT\sum_{i\in V}
	\left[
	\psi_i\log\psi_i
	+(1-\psi_i)\log(1-\psi_i)
	\right].
\end{equation}
Here, the first term denotes the expected energy under the mean-field
variational distribution, while the second term is its negative
entropy weighted by $k_BT$. The temperature therefore controls
the trade-off between expected energy and entropy.
\begin{restatable}[Variational Free Energy]{proposition}{VariationalFreeEnergy}
	\label{Variational Free Energy}
	Consider an energy function $E_\theta(S;V)$ and a mean-field
	variational distribution $q(\bm{\psi})$. The associated
	variational free energy admits a decomposition into
	expected energy and entropy:
	\begin{equation}
		\mathcal{F}_\theta(\bm{\psi})
		=
		\langle E\rangle_q-k_BT\mathcal{H}_q,
	\end{equation}
	where
	$\langle E\rangle_q
	=\mathbb{E}_{S\sim q(\bm{\psi})}[E_\theta(S;V)]$
	denotes the expected energy, and
	$\mathcal{H}_q=\mathbb{H}(q(\bm{\psi}))$
	denotes the entropy of the variational distribution.
\end{restatable}
Notably, the set utility function naturally plays the role of negative energy here, \textit{i.e.}, $F_\theta(\cdot)=-E_\theta(\cdot)$, which implies the expected energy term in (\ref{reformulated Helmholtz free energy}) can be interpreted as the negative counterpart of the multilinear extension $f_{\text{mt}}^{F_\theta}$ in (\ref{ELBO}).
Consequently, by setting $k_B T = 1$, the variational free energy becomes equivalent to the negative ELBO.
\begin{restatable}{corollary}{ELBOCorollary}\label{corollary1}
	When the set utility is interpreted as negative energy,
	i.e., $E_\theta=-F_\theta$, and $k_BT=1$, the variational
	free energy equals the negative evidence lower bound,
	$\mathcal{F}_\theta(\bm{\psi})=-\mathrm{ELBO}$.
\end{restatable}
Therefore, maximizing the ELBO in
Eq.~(\ref{ELBO}) is equivalent to minimizing the variational
free energy:
\begin{equation}\label{min F}
	\begin{aligned}
		\arg\max_{\bm{\psi}\in\mathcal{X}}
		\mathrm{ELBO}
		=
		\arg\min_{\bm{\psi}\in\mathcal{X}}
		\left[
		f_{\mathrm{mt}}^{E_\theta}(\bm{\psi})
		-\mathbb{H}(q(\bm{\psi}))
		\right]
		=
		\arg\min_{\bm{\psi}\in\mathcal{X}}
		\mathcal{F}_\theta(\bm{\psi}).
	\end{aligned}
\end{equation}
Conceptually, this equivalence indicates that variational inference
can be interpreted as an energy minimization process that balances
fitting the observed data with maintaining uncertainty.
Minimizing $\mathcal{F}_\theta(\bm{\psi})$ favors low-energy
configurations, while the entropy term discourages excessive
concentration of the variational distribution.
More importantly, this formulation provides flexibility beyond the objective in (\ref{ELBO}). 
In contrast to optimizing the ELBO under a fixed balancing relation (temperature), we can directly 
treat $T$ in (\ref{reformulated Helmholtz free energy}) as a learnable parameter, thereby allowing the model to identify an adaptive trade-off between expected energy and entropy.

\section{Network Architecture}
Although we have established theoretical guarantees above, realizing a smooth and well-behaved energy landscape remains notoriously difficult in practice \citep{gladstone2025energy}, as real-world problems often involve high-dimensional objective surfaces.
To address this, we leverage Langevin Dynamics \citep{du2019implicit} by injecting a stochastic noise term, thereby promoting exploration of the energy landscape:
\begin{equation}\label{add noise}
	\bm{\psi}^{(k+1)}  =  \bm{\psi}^{(k)}  +  
	\alpha\left(
	\nabla g_\theta(\bm{\psi}^{(k)})
	+\log\frac{\bm{1}-\bm{\psi}^{(k)}}{\bm{\psi}^{(k)}}
	\right)
	 +  \bm{\epsilon}_k,  \qquad  \bm{\epsilon}_k \sim \mathcal{N}(0, \tau),
\end{equation}
where $\tau$ is the magnitude of the noise $\bm{\epsilon}$. Besides this, to ensure that the optimization trajectory of the predictive distribution can adapt to diverse problem settings, we make the gradient ascent step size \(\alpha\) learnable, which improves generalization.
Taken together, these improvements equip the model with a stronger capacity to navigate complex and rugged objective landscapes in real-world tasks.
In addition, to further accelerate inference and improve scalability, we introduce an auxiliary recognition network\footnote{Following the setup of \citet{ou2022learning}; see Appendix~\ref{appendix:Accelerating Inference with Equivariant Neural Networks} for details.} that amortizes inference by providing a reasonable initialization of the variational distribution $q(\bm{\psi})$ for a given set $V$ before optimizing the objective in (\ref{ELBO}):
\begin{equation}
	\bm{\psi}^{(0)} = \operatorname{RecNet}_\phi(V).
\end{equation}
Furthermore, modeling the surrogate $g_\theta$ requires permutation invariance, which cannot be achieved by classic feed-forward networks. We adopt the DeepSets architecture \citep{zaheer2017deep}, which overcomes this limitation through the following proposition.
\begin{proposition}
	\label{Deep Set}
	All permutation invariant set functions can be decomposed in the form $f(S)=\rho\Big(\sum_{s_i\in S}\kappa(s_i)\Big)$, for suitable transformations $\kappa$ and $\rho$.
\end{proposition}
Following the DeepSets framework, we further construct permutation-invariant sufficient representations\footnote{Following \citet{xie2024enhancing}, we incorporate background information from $V$ into the set function; see Appendix~\ref{appendix:Incorporate Sufficient Invariant Statistics of Background Information} for details.}
for both the superset $V$ and the subset $S$ by employing the soft assignment $\bm{\psi}$. These representations can be approximated by empirical estimations $R(s_i)=\sum_{s_i \in S} \delta(s_i)$ \cite{bloem2020probabilistic}, where $\delta(s_i)$ represents an atom of unit mass located at $s_i$, such as one-hot embeddings.
Hence, the neural network construction can be outlined as follows:
\begin{equation}\label{Network Architecture}
	g_\theta(\bm{\psi})  =  \sigma \bigg(
	\theta_1 \sum_{i=1}^{|V|} \zeta_1(\psi_i s_i) + \theta_2 \sum_{i=1}^{|V|} \zeta_2((1-\psi_i) s_i)
	\bigg).
\end{equation}
Here, the feed-forward modules $\theta_1$ and $\theta_2$ are accompanied by a non-linear activation layer denoted by $\sigma(\cdot)$. Due to the limitations of the mean-field assumption, the variational distribution $q(\bm{\psi})$ is insufficient for capturing interactions among elements within the input set. To remedy this, we design the architecture of function $\zeta$ as the multi-head attention \citep{vaswani2017attention}, enabling the capture of interactions among elements.
The overall architecture of the proposed network is illustrated in Appendix~\ref{appendix:Overall Framework}.
Leveraging this network architecture, we show in the following section that our proposed framework empirically outperforms the baselines.

\section{Empirical Studies}
We evaluate the proposed method on a diverse range of tasks, including product recommendation, set anomaly detection, compound selection, and synthetic experiments. All experiments are conducted with five different random seeds, and we report the corresponding mean performance along with standard deviation. Detailed descriptions of the model architectures and training configurations are provided in Appendix~\ref{appendix:Experimental Details}.
Additional evaluations including ablation studies, hyperparameter sensitivity analyses, and experiments on varying ground set sizes are provided in Appendix~\ref{appendix:Additional Experiments}.

\textbf{Evaluation Metric.} 
The goal of optimal subset selection is to identify an optimal subset $S^*$
from a given ground set $V$. In experiments, we adopt the mean Jaccard coefficient (MJC) as the evaluation metric to assess the performance of different methods. For each ground set, the Jaccard coefficient is
defined as $\operatorname{JC}(S', S^*) = \frac{|S' \cap S^*|}{|S' \cup S^*|}$, which evaluates the similarity between the predicted subset $S^{\prime}$ and the ground-truth optimal subset $S^*$. Subsequently, the MJC is obtained by averaging the Jaccard coefficients across all test instances. 

\textbf{Baselines.} 
Our method is designed as a plug-in module that can be integrated into existing optimal subset oracle frameworks. 
The resulting variants are denoted by the suffix ``-$\operatorname{R}$''.
We evaluate these variants against four standard baselines: Random, Probabilistic Greedy Model (PGM) \citep{tschiatschek2018differentiable}, DeepSet \citep{zaheer2017deep} and Set Transformer \citep{lee2019set}; as well as two optimal subset oracles: EquiVSet \citep{ou2022learning} and INSET \citep{xie2024enhancing}. Detailed descriptions of these baselines are provided in Appendix \ref{Baselines}.

\begin{figure}[h]
	\centering
	\begin{minipage}[t]{0.42\linewidth}
		\centering
		\captionof{table}{Results on the synthetic datasets.}
		\vspace{1pt}
		\resizebox{\linewidth}{!}{
			\begin{tabular}{l|cc}
				\toprule
				Method & Two Moons & Gaussian Mixture  \\
				\midrule
				Random & 5.5 & 5.5 \\
				PGM & 36.0 $\pm$ 2.0 & 43.8 $\pm$ 0.9 \\
				DeepSet & 47.2 $\pm$ 0.3 & 44.6 $\pm$ 0.2 \\
				Set Transformer & 57.4 $\pm$ 0.2 & 90.5 $\pm$ 0.2 \\
				\midrule
				EquiVSet & 58.7 $\pm$ 0.2 & 90.9 $\pm$ 0.2 \\
				EquiVSet-$\operatorname{R}$
				& 99.2 $\pm$ 0.3 
				& \textbf{91.0 $\pm$ 0.1} \\
				\midrule
				INSET & 59.0 $\pm$ 0.3 & 90.9 $\pm$ 0.2 \\ 
				INSET-$\operatorname{R}$
				& \textbf{99.7 $\pm$ 0.1} 
				& 90.9 $\pm$ 0.1 \\
				\bottomrule
		\end{tabular}}
		\label{Synthetic}
	\end{minipage}
	\hfill
	\begin{minipage}[t]{0.55\linewidth}
		\centering
		\vspace{0pt}
		\includegraphics[width=\linewidth]{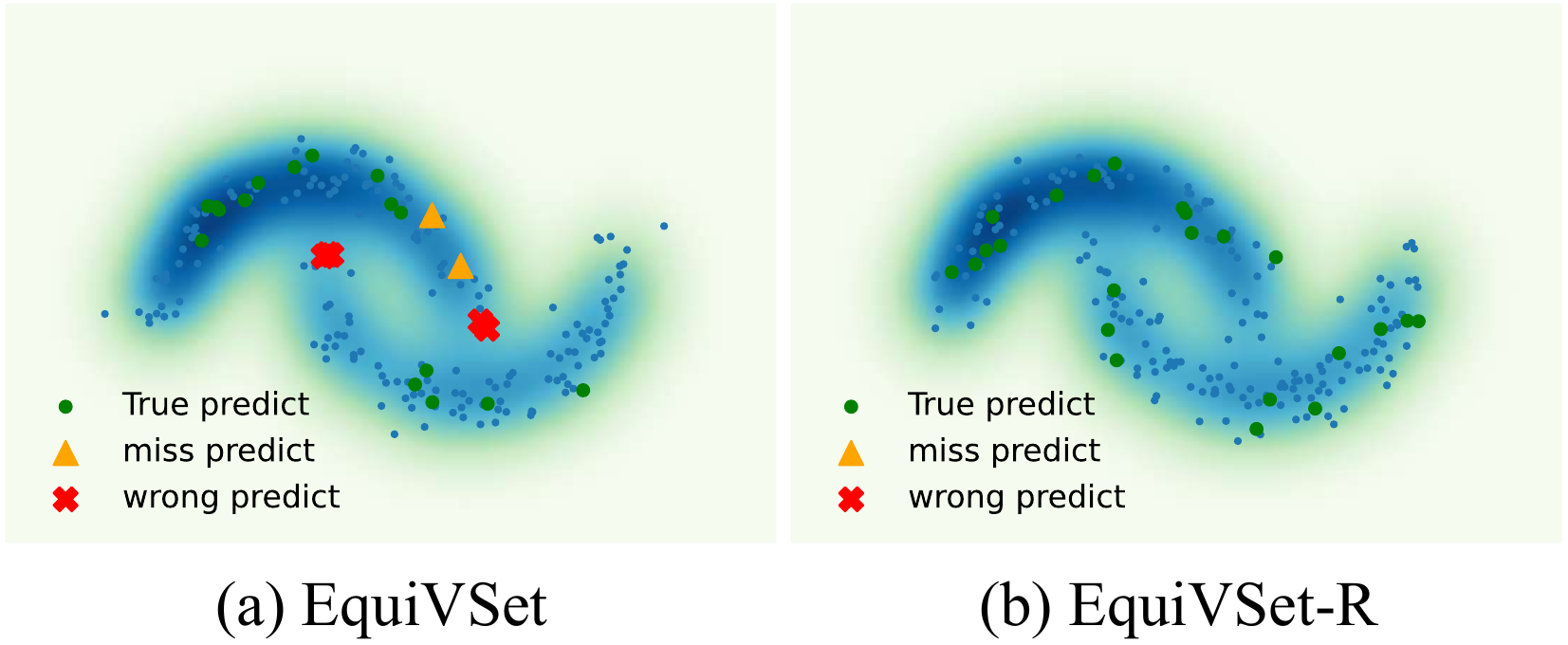}
		\vspace{-5mm}
		\captionof{figure}{Visualization on Two-Moons.}
		\label{moons_fig}
	\end{minipage}
\vspace{-4mm}
\end{figure}

\textbf{Synthetic Experiments.}
We first evaluate the methods on two synthetic datasets: the two-moons dataset \citep{pedregosa2011scikit} with additional noise of variance $\sigma^2=0.1$, and the mixture of Gaussians represented as $\frac{1}{2}\mathcal{N}(\boldsymbol{\mu}_{0}, \boldsymbol{\Sigma}) + \frac{1}{2}\mathcal{N}(\boldsymbol{\mu}_{1}, \boldsymbol{\Sigma})$, where $\boldsymbol{\mu}_{0} = [\frac{1}{\sqrt{2}}, \frac{1}{\sqrt{2}}]^{\top}, \boldsymbol{\mu}_{1} = -\boldsymbol{\mu}_{0}, \boldsymbol{\Sigma} = \frac{1}{4}\mathbf{I}.$
Taking the Gaussian Mixture dataset as an example, the data are generated according to the following procedure: i) Initially, we sample an index $b \sim \operatorname{Bernoulli}(\frac{1}{2})$; ii) we then sample 10 points from the Gaussian distribution $\mathcal{N}(\boldsymbol{\mu}_{b}, \boldsymbol{\Sigma})$ to construct the ground-truth subset $S^*$; iii) Finally, we sample additional 90 points from $\mathcal{N}(\boldsymbol{\mu}_{1-b}, \boldsymbol{\Sigma})$ to construct $V \backslash S^*$.
For both datasets, we generate 1,000 samples for training, validation, and test, respectively.

From the quantitative results in Table \ref{Synthetic}, it can be observed that the proposed method achieves a significant improvement of nearly 70\% over the baseline on the more challenging Two Moons dataset. 
In addition, we provide visualization results of the model predictions after training for 50 epochs in Figure~\ref{moons_fig}.
The green dots represent correct model predictions, the red crosses are incorrect model predictions, and the yellow triangles represent data points in subset oracle $S^*$ that are missed by the model. 
One can see that EquiVSet tends to misclassify more points at the intersection, while also missing some in-distribution points.
In comparison, EquiVSet-$\operatorname{R}$ learns a more cohesive and clearer distribution boundary, while correctly identifying all the difficult samples in the overlapping region.

\begin{table*}[h]
	\vspace{-3mm}
	\caption{Experimental results (\%) on product recommendation across categories in the Amazon dataset, where ``Set-T'' denotes Set Transformer.}
	\centering
	\label{Product Recommendation}
	\resizebox{1\textwidth}{!}{
		\begin{tabular}{c|c|ccc|cc|cc}
			\toprule
			Categories & Random & PGM & DeepSet & Set-T 
			& EquiVSet & EquiVSet-$\operatorname{R}$ 
			& INSET & INSET-$\operatorname{R}$ \\
			\midrule
			Toys 
			& 8.3 
			& 44.1 $\pm$ 0.4 
			& 42.9 $\pm$ 0.5 
			& 62.6 $\pm$ 2.0 
			& 68.0 $\pm$ 2.0 
			& \textbf{77.9 $\pm$ 1.0} 
			& 76.9 $\pm$ 0.5 
			& \textbf{79.0 $\pm$ 0.6} \\
			
			Furniture 
			& 6.5 
			& 17.5 $\pm$ 0.7 
			& 17.6 $\pm$ 0.7 
			& 17.6 $\pm$ 0.8 
			& 17.2 $\pm$ 0.9 
			& \textbf{35.9 $\pm$ 1.3} 
			& 16.9 $\pm$ 5.0 
			& \textbf{36.5 $\pm$ 1.5} \\
			
			Gear 
			& 7.7 
			& 47.1 $\pm$ 0.4 
			& 38.1 $\pm$ 0.2 
			& 64.7 $\pm$ 0.6 
			& 70.0 $\pm$ 0.2 
			& \textbf{78.7 $\pm$ 2.6} 
			& 80.8 $\pm$ 1.2 
			& \textbf{82.1 $\pm$ 1.1} \\
			
			Carseats 
			& 6.6 
			& 23.0 $\pm$ 1.0 
			& 21.0 $\pm$ 1.0 
			& 22.0 $\pm$ 1.0 
			& 21.0 $\pm$ 1.0 
			& \textbf{26.9 $\pm$ 1.0} 
			& 23.1 $\pm$ 3.4 
			& \textbf{30.2 $\pm$ 3.9} \\
			
			Bath 
			& 7.6 
			& 56.4 $\pm$ 0.8 
			& 42.4 $\pm$ 0.6 
			& 71.6 $\pm$ 0.5 
			& 75.7 $\pm$ 0.9 
			& \textbf{84.7 $\pm$ 0.9} 
			& 86.2 $\pm$ 0.5 
			& \textbf{86.7 $\pm$ 0.6} \\
			
			Health 
			& 7.6 
			& 44.9 $\pm$ 0.2 
			& 44.8 $\pm$ 0.4 
			& 69.0 $\pm$ 1.0 
			& 70.0 $\pm$ 2.0 
			& \textbf{81.4 $\pm$ 0.6} 
			& 81.2 $\pm$ 0.5 
			& \textbf{82.2 $\pm$ 0.2} \\
			
			Diaper 
			& 8.4 
			& 58.0 $\pm$ 0.9 
			& 45.7 $\pm$ 0.5 
			& 78.9 $\pm$ 0.5 
			& 83.0 $\pm$ 1.0 
			& \textbf{90.4 $\pm$ 0.6} 
			& 88.0 $\pm$ 0.7 
			& \textbf{90.4 $\pm$ 0.5} \\
			
			Bedding 
			& 7.9 
			& 48.0 $\pm$ 0.6 
			& 48.2 $\pm$ 0.8 
			& 76.0 $\pm$ 2.0 
			& 77.0 $\pm$ 1.0 
			& \textbf{86.8 $\pm$ 1.2} 
			& 85.7 $\pm$ 1.0 
			& \textbf{87.6 $\pm$ 1.7} \\
			
			Safety 
			& 6.5 
			& 25.0 $\pm$ 0.6 
			& 22.1 $\pm$ 0.4 
			& 23.4 $\pm$ 0.9 
			& 25.0 $\pm$ 3.0 
			& \textbf{32.1 $\pm$ 9.4} 
			& 23.8 $\pm$ 1.5 
			& \textbf{38.3 $\pm$ 4.0} \\
			
			Feeding 
			& 9.3 
			& 56.0 $\pm$ 0.8 
			& 43.0 $\pm$ 0.2 
			& 75.3 $\pm$ 0.6 
			& 81.0 $\pm$ 0.7 
			& \textbf{87.2 $\pm$ 1.4} 
			& 88.2 $\pm$ 1.0 
			& \textbf{88.6 $\pm$ 0.5} \\
			
			Apparel 
			& 9.0 
			& 53.3 $\pm$ 0.5 
			& 50.7 $\pm$ 0.4 
			& 68.0 $\pm$ 2.0 
			& 75.0 $\pm$ 1.0 
			& \textbf{82.2 $\pm$ 0.9} 
			& \textbf{83.7 $\pm$ 0.3} 
			& 83.3 $\pm$ 1.3 \\
			
			Media 
			& 9.4 
			& 44.1 $\pm$ 0.9 
			& 42.0 $\pm$ 1.0 
			& 53.0 $\pm$ 2.0 
			& 57.0 $\pm$ 1.0 
			& \textbf{67.0 $\pm$ 1.0} 
			& 62.0 $\pm$ 2.3 
			& \textbf{68.1 $\pm$ 1.1} \\
			\bottomrule
		\end{tabular}
	}
\end{table*}

\textbf{Product Recommendation.}
We use the Amazon baby registry dataset \citep{gillenwater2014expectation} as a benchmark for the product recommendation task, which contains multiple categories of sub-products selected by different customers. 
Each product in the dataset is registered under a specific category, such as “Diaper” or “Feeding”, and is represented by a brief textual description. 
In practice, each product is represented by a 768-dimensional embedding obtained from a pre-trained BERT model \citep{lee2018pre}.
Following \citet{ou2022learning}, we partition the data into training, validation, and test sets using a 1:1:1 ratio. The objective is to identify the subset of products that a customer finds most relevant within a candidate set of size 30.
The performance of all methods across different categories is reported in Table~\ref{Product Recommendation}.
Overall, our approach outperforms the two optimal subset oracle baselines on most evaluated tasks.
Notably, the performance improvements are more evident on challenging categories, such as Furniture, Carseats, Safety, and Media.

\textbf{Set Anomaly Detection.}
Experiments on set anomaly detection are conducted using four widely adopted benchmarks: Double MNIST \citep{sun2019multi}, F-MNIST \citep{xiao2017fashion}, CelebA \citep{liu2015deep}, and CIFAR-10 \citep{krizhevsky2009learning}.\footnote{Due to the absence of public code and the relative ambiguity in the protocol descriptions of \citet{ou2022learning}, we provide a reconstructed and well-defined experimental standard for Double MNIST, F-MNIST, and CIFAR-10; further discussions are provided in Appendix~\ref{Experimental Settings for Set Anomaly Detection}.}
According to the protocol of \citet{zaheer2017deep, ou2022learning}, we split each dataset into 10,000 training, 1,000 validation, and 1,000 test set instances.
Taking the Double MNIST dataset as an example, we form the OS oracle $S^*$ by sampling $n \in \{2,3,4\}$ images of the same class, with the complement $V \backslash S^*$ comprising $20 - |S^*|$ images sampled from non-target classes.
\begin{table*}[t]
	\centering
	\caption{Experimental results (\%) on anomaly detection and compound selection tasks.}
	\label{tab:AD-CS-results}
	\resizebox{0.88\textwidth}{!}{
		\begin{tabular}{l|cccc|cc}
			\toprule
			& \multicolumn{4}{c|}{{Anomaly Detection}} 
			& \multicolumn{2}{c}{{Compound Selection}} \\
			& \textbf{Double MNIST} 
			& \textbf{CelebA} 
			& \textbf{F-MNIST} 
			& \textbf{CIFAR10} 
			& \textbf{PDBBind} 
			& \textbf{BindingDB} \\
			\midrule
			Random 
			& 9.3 
			& 21.9 
			& 20.9 
			& 20.1 
			& 9.9 
			& 0.9 \\
			
			PGM 
			& - 
			& 48.1 $\pm$ 0.6 
			& - 
			& - 
			& 91.0 $\pm$ 1.0 
			& 69.0 $\pm$ 2.0 \\
			
			DeepSet 
			& 18.5 $\pm$ 1.0 
			& 44.0 $\pm$ 0.5 
			& 59.7 $\pm$ 1.2 
			& 29.0 $\pm$ 0.3 
			& 90.1 $\pm$ 1.1 
			& 71.0 $\pm$ 2.0 \\
			
			Set Transformer 
			& 22.0 $\pm$ 2.6 
			& 52.7 $\pm$ 0.8 
			& 69.8 $\pm$ 2.9 
			& 27.0 $\pm$ 1.7 
			& 91.9 $\pm$ 1.5 
			& 71.5 $\pm$ 1.0 \\
			
			\midrule
			EquiVSet 
			& 38.6 $\pm$ 2.5 
			& 54.9 $\pm$ 0.5 
			& 58.4 $\pm$ 1.2 
			& 30.4 $\pm$ 0.7 
			& 92.4 $\pm$ 1.1 
			& 72.1 $\pm$ 0.9 \\
			
			EquiVSet-$\operatorname{R}$ 
			& \textbf{71.0 $\pm$ 1.6} 
			& \textbf{58.8 $\pm$ 0.6} 
			& \textbf{72.6 $\pm$ 2.1} 
			& \textbf{32.4 $\pm$ 2.2} 
			& \textbf{93.0 $\pm$ 2.6} 
			& \textbf{78.3 $\pm$ 0.8} \\
			
			\midrule
			INSET 
			& 46.2 $\pm$ 1.1 
			& 58.0 $\pm$ 1.2 
			& 66.0 $\pm$ 2.8 
			& 32.0 $\pm$ 0.9 
			& 93.5 $\pm$ 0.8 
			& 73.4 $\pm$ 1.0 \\
			
			INSET-$\operatorname{R}$ 
			& \textbf{73.2 $\pm$ 1.8} 
			& \textbf{59.2 $\pm$ 0.2} 
			& \textbf{75.9 $\pm$ 1.4} 
			& \textbf{35.8 $\pm$ 1.3} 
			& \textbf{96.0 $\pm$ 0.6} 
			& \textbf{77.8 $\pm$ 1.2} \\
			\bottomrule
	\end{tabular}}
\vspace{-5mm}
\end{table*}
From Table~\ref{tab:AD-CS-results}, it is evident that our proposed variants consistently outperform the baseline methods with notable improvements across all datasets.\footnote{Results for PGM are unavailable as the official implementation has not been released.}
For instance, on the Double MNIST dataset, our proposed variant EquiVSet-$\operatorname{R}$ achieves a relative improvement of approximately 83.9\% compared to the original framework. 

\textbf{Compound Selection in AI-aided Drug Discovery.}
In the drug discovery process, screening compounds with high bioactivity \citep{wallach2015atomnet, li2021structure, ji2023drugood}, structural diversity, and favorable
ADME
 properties \citep{gimeno2019light} is a critical step.
Virtual screening pipelines typically proceed through multiple sequential filtering stages. 
For example, researchers first select compounds with high bioactivity, then extract diverse subsets, and finally remove candidates with unfavorable ADME properties.
After these steps, a final subset of compounds is obtained.
Due to privacy constraints in pharmaceutical practice, intermediate supervision signals are often inaccessible or costly to obtain. Consequently, models are expected to learn the composite screening process end-to-end, which aligns well with the OS oracle.
Nevertheless, full end-to-end prediction is often more complex and requires additional domain expertise. 
Here, we focus on the high-bioactivity selection stage, following the experimental setting of \citet{xie2024enhancing}.
In the experiments, we evaluate the methods on PDBBind \citep{liu2015pdb} and BindingDB \citep{liu2007bindingdb} datasets using the same high-bioactivity filter.
As shown in Table~\ref{tab:AD-CS-results}, our method consistently
improves upon both OS-oracle baselines across the two datasets. The
gains are particularly pronounced on the more challenging BindingDB
benchmark.
This suggests that our model effectively captures essential patterns directly from the data. Consequently, it appears less sensitive to potential misguidance from inductive biases that might not perfectly align with the real-world complex distributions.
Collectively, these results underscore the promise of the proposed framework for drug discovery applications.

\section{Related Work}

Set function learning is commonly studied through function-value (FV) oracles, which learn to predict the utility of a given set \citep{wendler2019powerset, wendler2021learning, de2022neural}. Representative architectures include DeepSets \citep{zaheer2017deep}, Set Transformer \citep{lee2019set}, and neural models for submodular functions \citep{bilmes2017deep, dolhansky2016deep, balcan2018submodular}. However, these methods typically require dense utility supervision, limiting their practicality. Optimal subset (OS) oracles address this issue by learning latent set functions from optimal subsets under an energy-based variational inference framework \citep{ou2022learning}, with subsequent extensions incorporating contextual information, hierarchical attention, and implicit differentiation \citep{xie2024enhancing, xie2024horse, ozcan2025learning}. Nevertheless, existing OS methods use Monte Carlo sampling for gradient estimation, leading to inefficient training. Our work builds on this line by learning a relaxation of neural set functions, while also connecting OS oracle learning with energy minimization and weakly DR-submodular maximization.
Further discussion is provided in Appendix~\ref{appendix:Related Work}.

\section{Conclusion}

In this paper, we propose $\operatorname{ReSet}$, which learns a relaxation of neural set functions to advance optimal subset oracles.
The central idea of our approach is to replace Monte Carlo sampling over the fixed landscape with a learnable paradigm.
By constructing a surrogate of the evidence lower bound, our method enables efficient and stable gradient estimation across the continuous domain, thereby accelerating and improving the variational inference process.
We further establish approximation guarantees for the resulting solutions, characterize their convergence to stationary points, and show that the proposed method can adaptively learn the trade-off between expected utility and entropy.
As a general plug-in module, our method can be integrated into various optimal subset oracle frameworks, with extensive experiments demonstrating its effectiveness and scalability.

\subsection*{AI use statement}

In this work, we used generative AI tools for language polishing and proofreading. All AI-assisted edits were carefully reviewed by the authors and did not alter the technical content or scientific claims of the paper. We take full responsibility for the final content of this work.

\subsection*{Reproducibility statement}
We provide detailed model descriptions, training procedures, and experimental settings in the main paper and appendices. The assumptions and complete proofs of the theoretical results are included in the appendix. Dataset preprocessing, evaluation protocols, and hyperparameter settings are also provided to facilitate reproduction of the experiments. Anonymous source code is provided in the supplementary materials.

\bibliography{iclr2027_conference}

@book{stigler1990history,
	title={The history of statistics: The measurement of uncertainty before 1900},
	author={Stigler, Stephen M},
	year={1990},
	publisher={Harvard University Press}
}

@article{guiasu1985principle,
	title={The principle of maximum entropy},
	author={Guiasu, Silviu and Shenitzer, Abe},
	journal={The mathematical intelligencer},
	volume={7},
	number={1},
	pages={42--48},
	year={1985},
}

@article{jeffreys1946invariant,
	title={An invariant form for the prior probability in estimation problems},
	author={Jeffreys, Harold},
	journal={Proceedings of the Royal Society of London. Series A. Mathematical and Physical Sciences},
	volume={186},
	number={1007},
	pages={453--461},
	year={1946},
}

@inproceedings{ou2022learning,
	title={Learning neural set functions under the optimal subset oracle},
	author={Ou, Zijing and Xu, Tingyang and Su, Qinliang and Li, Yingzhen and Zhao, Peilin and Bian, Yatao},
	booktitle={Advances in Neural Information Processing Systems},
	pages={35021--35034},
	year={2022}
}

@inproceedings{xie2018cooperative,
	title={Cooperative learning of energy-based model and latent variable model via MCMC teaching},
	author={Xie, Jianwen and Lu, Yang and Gao, Ruiqi and Wu, Ying Nian},
	booktitle={Proceedings of the AAAI Conference on Artificial Intelligence},
	year={2018}
}

@article{he2025feat,
	title={FEAT: Free energy Estimators with Adaptive Transport},
	author={He, Jiajun and Du, Yuanqi and Vargas, Francisco and Wang, Yuanqing and Gomes, Carla P and Hern{\'a}ndez-Lobato, Jos{\'e} Miguel and Vanden-Eijnden, Eric},
	journal={arXiv preprint arXiv:2504.11516},
	year={2025}
}

@inproceedings{dagreou2024compute,
	title={How to compute Hessian-vector products?},
	author={Dagr{\'e}ou, Mathieu and Ablin, Pierre and Vaiter, Samuel and Moreau, Thomas},
	booktitle={The Third Blogpost Track at ICLR},
	year={2024}
}

@article{domke2013learning,
	title={Learning graphical model parameters with approximate marginal inference},
	author={Domke, Justin},
	journal={IEEE transactions on pattern analysis and machine intelligence},
	volume={35},
	number={10},
	pages={2454--2467},
	year={2013},
}

@inproceedings{calinescu2007maximizing,
	title={Maximizing a Submodular Set Function Subject to a Matroid Constraint},
	author={Calinescu, Gruia and Chekuri, Chandra and P{\'a}l, Martin and Vondr{\'a}k, Jan},
	booktitle={Integer Programming and Combinatorial Optimization},
	pages={182--196},
	year={2007},
}

@article{gladstone2025energy,
	title={Energy-Based Transformers are Scalable Learners and Thinkers},
	author={Gladstone, Alexi and Nanduru, Ganesh and Islam, Md Mofijul and Han, Peixuan and Ha, Hyeonjeong and Chadha, Aman and Du, Yilun and Ji, Heng and Li, Jundong and Iqbal, Tariq},
	journal={arXiv preprint arXiv:2507.02092},
	year={2025}
}

@article{du2019implicit,
	title={Implicit generation and modeling with energy based models},
	author={Du, Yilun and Mordatch, Igor},
	journal={Advances in neural information processing systems},
	volume={32},
	year={2019}
}

@article{zaheer2017deep,
	title={Deep sets},
	author={Zaheer, Manzil and Kottur, Satwik and Ravanbakhsh, Siamak and Poczos, Barnabas and Salakhutdinov, Russ R and Smola, Alexander J},
	journal={Advances in neural information processing systems},
	volume={30},
	year={2017}
}

@inproceedings{xie2024enhancing,
	title={Enhancing neural subset selection: Integrating background information into set representations},
	author={Xie, Binghui and Bian, Yatao and Chen, Yongqiang and Zhao, Peilin and Han, Bo and Meng, Wei and Cheng, James and others},
	booktitle={International Conference on Learning Representations},
	year={2024}
}

@article{bloem2020probabilistic,
	title={Probabilistic symmetries and invariant neural networks},
	author={Bloem-Reddy, Benjamin and Teh, Yee Whye},
	journal={Journal of Machine Learning Research},
	volume={21},
	number={90},
	pages={1--61},
	year={2020}
}

@inproceedings{vaswani2017attention,
	title={Attention is all you need},
	author={Vaswani, Ashish and Shazeer, Noam and Parmar, Niki and Uszkoreit, Jakob and Jones, Llion and Gomez, Aidan N and Kaiser, {\L}ukasz and Polosukhin, Illia},
	booktitle={Advances in neural information processing systems},
	volume={30},
	year={2017}
}

@inproceedings{tschiatschek2018differentiable,
	title={Differentiable submodular maximization},
	author={Tschiatschek, Sebastian and Sahin, Aytunc and Krause, Andreas},
	booktitle={Proceedings of the International Joint Conference on Artificial Intelligence},
	pages={2731--2738},
	year={2018}
}

@inproceedings{lee2019set,
	title={Set transformer: A framework for attention-based permutation-invariant neural networks},
	author={Lee, Juho and Lee, Yoonho and Kim, Jungtaek and Kosiorek, Adam and Choi, Seungjin and Teh, Yee Whye},
	booktitle={International conference on machine learning},
	pages={3744--3753},
	year={2019},
}

@inproceedings{xie2024horse,
	title={HORSE: hierarchical representation for large-scale neural subset selection},
	author={Xie, Binghui and Wang, Yixuan and Chen, Yongqiang and Zhou, Kaiwen and Li, Yu and Meng, Wei and Cheng, James},
	booktitle={Advances in neural information processing systems},
	pages={4852--4877},
	year={2024}
}

@article{pedregosa2011scikit,
	title={Scikit-learn: Machine learning in Python},
	author={Pedregosa, Fabian and Varoquaux, Ga{\"e}l and Gramfort, Alexandre and Michel, Vincent and Thirion, Bertrand and Grisel, Olivier and Blondel, Mathieu and Prettenhofer, Peter and Weiss, Ron and Dubourg, Vincent and others},
	journal={the Journal of machine Learning research},
	volume={12},
	pages={2825--2830},
	year={2011},
}

@article{gillenwater2014expectation,
	title={Expectation-maximization for learning determinantal point processes},
	author={Gillenwater, Jennifer A and Kulesza, Alex and Fox, Emily and Taskar, Ben},
	journal={Advances in Neural Information Processing Systems},
	volume={27},
	year={2014}
}

@inproceedings{lee2018pre,
	title     = {{BERT}: Pre-training of Deep Bidirectional Transformers for Language Understanding},
	author    = {Devlin, Jacob and Chang, Ming-Wei and Lee, Kenton and Toutanova, Kristina},
	booktitle = {Proceedings of the 2019 Conference of the North American Chapter of the Association for Computational Linguistics: Human Language Technologies, Volume 1 (Long and Short Papers)},
	year      = {2019},
	pages     = {4171--4186},
}

@inproceedings{zhang2020set,
	title={Set Prediction without Imposing Structure as Conditional Density Estimation},
	author={Zhang, David W and Burghouts, Gertjan J and Snoek, Cees GM},
	booktitle={International Conference on Learning Representations},
	year={2020}
}

@article{gimeno2019light,
	title={The light and dark sides of virtual screening: what is there to know?},
	author={Gimeno, Aleix and Ojeda-Montes, Mar{\'\i}a Jos{\'e} and Tom{\'a}s-Hern{\'a}ndez, Sarah and Cereto-Massagu{\'e}, Adri{\`a} and Beltr{\'a}n-Deb{\'o}n, Ra{\'u}l and Mulero, Miquel and Pujadas, Gerard and Garcia-Vallv{\'e}, Santiago},
	journal={International journal of molecular sciences},
	volume={20},
	number={6},
	pages={1375},
	year={2019},
}

@inproceedings{rezatofighi2017deepsetnet,
	title={Deepsetnet: Predicting sets with deep neural networks},
	author={Rezatofighi, S Hamid and Bg, Vijay Kumar and Milan, Anton and Abbasnejad, Ehsan and Dick, Anthony and Reid, Ian},
	booktitle={International Conference on Computer Vision},
	pages={5257--5266},
	year={2017},
}

@inproceedings{coppolillo2024relevance,
	title={Relevance meets diversity: A user-centric framework for knowledge exploration through recommendations},
	author={Coppolillo, Erica and Manco, Giuseppe and Gionis, Aristides},
	booktitle={Proceedings of the ACM SIGKDD conference on knowledge discovery and data mining},
	pages={490--501},
	year={2024}
}

@article{balcan2018submodular,
	title={Submodular functions: Learnability, structure, and optimization},
	author={Balcan, Maria-Florina and Harvey, Nicholas JA},
	journal={SIAM Journal on Computing},
	volume={47},
	number={3},
	pages={703--754},
	year={2018},
}

@inproceedings{ozcan2025learning,
	title={Learning Set Functions with Implicit Differentiation},
	author={{\"O}zcan, G{\"o}zde and Shi, Chengzhi and Ioannidis, Stratis},
	booktitle={Proceedings of the AAAI Conference on Artificial Intelligence},
	pages={19777--19785},
	year={2025}
}

@article{blei2017variational,
	title={Variational inference: A review for statisticians},
	author={Blei, David M and Kucukelbir, Alp and McAuliffe, Jon D},
	journal={Journal of the American statistical Association},
	volume={112},
	number={518},
	pages={859--877},
	year={2017},
}

@inproceedings{karalias2022neural,
	title={Neural set function extensions: Learning with discrete functions in high dimensions},
	author={Karalias, Nikolaos and Robinson, Joshua and Loukas, Andreas and Jegelka, Stefanie},
	booktitle={Advances in Neural Information Processing Systems},
	volume={35},
	pages={15338--15352},
	year={2022}
}

@inproceedings{hassani2017gradient,
	title={Gradient methods for submodular maximization},
	author={Hassani, Hamed and Soltanolkotabi, Mahdi and Karbasi, Amin},
	booktitle={Advances in Neural Information Processing Systems},
	volume={30},
	year={2017}
}

@article{sun2019multi,
	title={Multi-digit MNIST for few-shot learning},
	author={Sun, Shao-Hua},
	journal={GitHub repository},
	year={2019}
}

@inproceedings{liu2015deep,
	title={Deep learning face attributes in the wild},
	author={Liu, Ziwei and Luo, Ping and Wang, Xiaogang and Tang, Xiaoou},
	booktitle={Proceedings of the IEEE international conference on computer vision},
	pages={3730--3738},
	year={2015}
}

@article{xiao2017fashion,
	title={Fashion-mnist: a novel image dataset for benchmarking machine learning algorithms},
	author={Xiao, Han and Rasul, Kashif and Vollgraf, Roland},
	journal={arXiv preprint arXiv:1708.07747},
	year={2017}
}

@techreport{krizhevsky2009learning,
	title       = {Learning Multiple Layers of Features from Tiny Images},
	author      = {Krizhevsky, Alex},
	institution = {University of Toronto},
	year        = {2009},
	url         = {https://www.cs.toronto.edu/~kriz/learning-features-2009-TR.pdf}
}

@article{wallach2015atomnet,
	title={AtomNet: a deep convolutional neural network for bioactivity prediction in structure-based drug discovery},
	author={Wallach, Izhar and Dzamba, Michael and Heifets, Abraham},
	journal={arXiv preprint arXiv:1510.02855},
	year={2015}
}

@inproceedings{li2021structure,
	title={Structure-aware interactive graph neural networks for the prediction of protein-ligand binding affinity},
	author={Li, Shuangli and Zhou, Jingbo and Xu, Tong and Huang, Liang and Wang, Fan and Xiong, Haoyi and Huang, Weili and Dou, Dejing and Xiong, Hui},
	booktitle={Proceedings of the 27th ACM SIGKDD conference on knowledge discovery \& data mining},
	pages={975--985},
	year={2021}
}

@inproceedings{ji2023drugood,
	title={Drugood: Out-of-distribution dataset curator and benchmark for ai-aided drug discovery--a focus on affinity prediction problems with noise annotations},
	author={Ji, Yuanfeng and Zhang, Lu and Wu, Jiaxiang and Wu, Bingzhe and Li, Lanqing and Huang, Long-Kai and Xu, Tingyang and Rong, Yu and Ren, Jie and Xue, Ding and others},
	booktitle={Proceedings of the AAAI Conference on Artificial Intelligence},
	pages={8023--8031},
	year={2023}
}

@article{liu2015pdb,
	title={PDB-wide collection of binding data: current status of the PDBbind database},
	author={Liu, Zhihai and Li, Yan and Han, Li and Li, Jie and Liu, Jie and Zhao, Zhixiong and Nie, Wei and Liu, Yuchen and Wang, Renxiao},
	journal={Bioinformatics},
	volume={31},
	number={3},
	pages={405--412},
	year={2015},
}

@article{liu2007bindingdb,
	title={BindingDB: a web-accessible database of experimentally determined protein--ligand binding affinities},
	author={Liu, Tiqing and Lin, Yuhmei and Wen, Xin and Jorissen, Robert N and Gilson, Michael K},
	journal={Nucleic acids research},
	volume={35},
	number={suppl\_1},
	pages={D198--D201},
	year={2007},
}

@article{lecun2006tutorial,
	title={A tutorial on energy-based learning},
	author={LeCun, Yann and Chopra, Sumit and Hadsell, Raia and Ranzato, M and Huang, Fujie and others},
	journal={Predicting structured data},
	volume={1},
	number={0},
	year={2006}
}

@inproceedings{wendler2021learning,
	title={Learning set functions that are sparse in non-orthogonal Fourier bases},
	author={Wendler, Chris and Amrollahi, Andisheh and Seifert, Bastian and Krause, Andreas and P{\"u}schel, Markus},
	booktitle={Proceedings of the AAAI Conference on Artificial Intelligence},
	pages={10283--10292},
	year={2021}
}

@inproceedings{de2022neural,
	title={Neural estimation of submodular functions with applications to differentiable subset selection},
	author={De, Abir and Chakrabarti, Soumen},
	booktitle={Advances in Neural Information Processing Systems},
	pages={19537--19552},
	year={2022}
}

@article{bilmes2017deep,
	title={Deep submodular functions},
	author={Bilmes, Jeffrey and Bai, Wenruo},
	journal={arXiv preprint arXiv:1701.08939},
	year={2017}
}

@inproceedings{dolhansky2016deep,
	title={Deep submodular functions: Definitions and learning},
	author={Dolhansky, Brian W and Bilmes, Jeff A},
	booktitle={Advances in Neural Information Processing Systems},
	volume={29},
	year={2016}
}

@inproceedings{wendler2019powerset,
	title={Powerset convolutional neural networks},
	author={Wendler, Chris and P{\"u}schel, Markus and Alistarh, Dan},
	booktitle={Advances in Neural Information Processing Systems},
	year={2019}
}

@inproceedings{du2022learning,
	title={Learning iterative reasoning through energy minimization},
	author={Du, Yilun and Li, Shuang and Tenenbaum, Joshua and Mordatch, Igor},
	booktitle={International Conference on Machine Learning},
	pages={5570--5582},
	year={2022},
}

@inproceedings{du2024learning,
	title={Learning Iterative Reasoning through Energy Diffusion},
	author={Du, Yilun and Mao, Jiayuan and Tenenbaum, Joshua B},
	booktitle={International Conference on Machine Learning},
	pages={11764--11776},
	year={2024},
}

@article{wang2025equilibrium,
	title={Equilibrium Matching: Generative Modeling with Implicit Energy-Based Models},
	author={Wang, Runqian and Du, Yilun},
	journal={arXiv preprint arXiv:2510.02300},
	year={2025}
}

@inproceedings{carreira2005contrastive,
	title={On contrastive divergence learning},
	author={Carreira-Perpinan, Miguel A and Hinton, Geoffrey},
	booktitle={International workshop on artificial intelligence and statistics},
	pages={33--40},
	year={2005},
}

@article{hyvarinen2005estimation,
	title   = {Estimation of Non-Normalized Statistical Models by Score Matching},
	author  = {Hyv{\"a}rinen, Aapo},
	journal = {Journal of Machine Learning Research},
	volume  = {6},
	number  = {24},
	pages   = {695--709},
	year    = {2005},
}

@article{vincent2011connection,
	title={A connection between score matching and denoising autoencoders},
	author={Vincent, Pascal},
	journal={Neural computation},
	volume={23},
	number={7},
	pages={1661--1674},
	year={2011},
}

@inproceedings{song2020sliced,
	title={Sliced score matching: A scalable approach to density and score estimation},
	author={Song, Yang and Garg, Sahaj and Shi, Jiaxin and Ermon, Stefano},
	booktitle={Uncertainty in artificial intelligence},
	pages={574--584},
	year={2020},
}

@inproceedings{tieleman2008training,
	title={Training restricted Boltzmann machines using approximations to the likelihood gradient},
	author={Tieleman, Tijmen},
	booktitle={Proceedings of the 25th international conference on Machine learning},
	pages={1064--1071},
	year={2008}
}

@misc{kingma2017adammethodstochasticoptimization,
	title={Adam: A Method for Stochastic Optimization}, 
	author={Diederik P. Kingma and Jimmy Ba},
	year={2017},
	eprint={1412.6980},
	archivePrefix={arXiv},
}

@article{loshchilov2016sgdr,
	title={Sgdr: Stochastic gradient descent with warm restarts},
	author={Loshchilov, Ilya and Hutter, Frank},
	journal={arXiv preprint arXiv:1608.03983},
	year={2016}
}

@article{gomes2017atomic,
	title={Atomic convolutional networks for predicting protein-ligand binding affinity},
	author={Gomes, Joseph and Ramsundar, Bharath and Feinberg, Evan N and Pande, Vijay S},
	journal={arXiv preprint arXiv:1703.10603},
	year={2017}
}

@article{ozturk2018deepdta,
	title={DeepDTA: deep drug--target binding affinity prediction},
	author={{\"O}zt{\"u}rk, Hakime and {\"O}zg{\"u}r, Arzucan and Ozkirimli, Elif},
	journal={Bioinformatics},
	volume={34},
	number={17},
	pages={i821--i829},
	year={2018},
}

@inproceedings{djolonga2017differentiable,
	author = {Djolonga, Josip and Krause, Andreas},
	booktitle = {Advances in Neural Information Processing Systems},
	title = {Differentiable Learning of Submodular Models},
	year = {2017}
}

@misc{bhatt2024deepsubmodularperipteralnetworks,
	title={Deep Submodular Peripteral Networks}, 
	author={Gantavya Bhatt and Arnav Das and Jeff Bilmes},
	year={2024},
	eprint={2403.08199},
	archivePrefix={arXiv},
}

@misc{kothawade2020deepsubmodularnetworksextractive,
	title={Deep Submodular Networks for Extractive Data Summarization}, 
	author={Suraj Kothawade and Jiten Girdhar and Chandrashekhar Lavania and Rishabh Iyer},
	year={2020},
	eprint={2010.08593},
	archivePrefix={arXiv},
}

@misc{maron2020learningsetssymmetricelements,
	title={On Learning Sets of Symmetric Elements}, 
	author={Haggai Maron and Or Litany and Gal Chechik and Ethan Fetaya},
	year={2020},
	eprint={2002.08599},
	archivePrefix={arXiv},
}

@misc{pedramfar2024unifiedapproachmaximizingcontinuous,
	title={A Unified Approach for Maximizing Continuous DR-submodular Functions}, 
	author={Mohammad Pedramfar and Christopher John Quinn and Vaneet Aggarwal},
	year={2024},
	eprint={2305.16671},
	archivePrefix={arXiv},
}
\bibliographystyle{iclr2027_conference}

\clearpage

\appendix

\begin{center}
	\LARGE
	\textbf{Appendix for ``Advancing Optimal Subset Oracle via Learning Relaxation of Neural Set Functions''}
\end{center}

\etocdepthtag.toc{mtappendix}
\etocsettagdepth{mtchapter}{none}
\etocsettagdepth{mtappendix}{subsection}

{\small \tableofcontents}

\section{Proofs and Derivations}

\subsection{Proof of Theorem \ref{theorem1}} \label{appendix:Proof of Theorem1}
{
	\renewcommand{\theoremfootnote}[1]{}
	\MainTheorem*
}
\begin{proof}
	Recall that we consider the continuous relaxation
	\begin{equation}
		\max_{\bm{\psi}\in\mathcal{X}}g_\theta(\bm{\psi}),
		\qquad
		\mathcal{X}=[0,1]^{|V|}.
		\nonumber
	\end{equation}
	Here, $g_\theta(\cdot)$ is a differentiable surrogate of the
	multilinear extension induced by the discrete set function
	$F_\theta(\cdot)$. Under the assumptions of the theorem,
	$g_\theta:[0,1]^{|V|}\to\mathbb{R}_{+}$ is $L$-smooth and
	$\gamma$-weakly DR-submodular for some $\gamma\in(0,1]$.
	Specifically, for any $\bm{x},\bm{y}\in[0,1]^{|V|}$ satisfying
	$\bm{x}\le\bm{y}$,
	\begin{equation}
		\nabla g_\theta(\bm{x})
		\ge
		\gamma\nabla g_\theta(\bm{y}),
		\label{eq:weak_dr_definition}
	\end{equation}
	where all vector inequalities are understood coordinate-wise. Let
	\[
	\bm{\psi}_{\mathrm{OPT}}
	\in
	\arg\max_{\bm{\psi}\in\mathcal{X}}
	g_\theta(\bm{\psi})
	\]
	denote a global maximizer.
	
	For brevity, let $d\triangleq |V|$, $\eta\triangleq1/T$, and define
	\[
	G_t
	\triangleq
	g_\theta\left(\bm{\psi}^{(t)}\right),
	\qquad
	G_{\mathrm{OPT}}
	\triangleq
	g_\theta\left(\bm{\psi}_{\mathrm{OPT}}\right).
	\]
	
	For part 1, initialize $\bm{\psi}^{(0)}=\bm{0}$ and consider
	the continuous-greedy iteration
	\begin{equation}\label{eq:continuous_greedy_update}
		\begin{aligned}
			\bm{v}^{(t)}
			&\in
			\arg\max_{\bm{0}\le\bm{v}\le\bm{1}-\bm{\psi}^{(t)}}
			\left\langle
			\nabla g_\theta(\bm{\psi}^{(t)}),\bm{v}
			\right\rangle,\\
			\bm{\psi}^{(t+1)}
			&=\bm{\psi}^{(t)}+\frac{1}{T}\bm{v}^{(t)},
			\qquad t=0,\ldots,T-1,
		\end{aligned}
	\end{equation}
	where $T\ge2$.
	We first verify that the iterates generated by
	Eq.~(\ref{eq:continuous_greedy_update}) remain feasible. The claim holds at
	$t=0$ because $\bm{\psi}^{(0)}=\bm{0}$. Moreover, since
	\[
	\bm{0}
	\le
	\bm{v}^{(t)}
	\le
	\bm{1}-\bm{\psi}^{(t)}
	\]
	and $\eta\in(0,1]$, we have
	\[
	\bm{0}
	\le
	\bm{\psi}^{(t+1)}
	=
	\bm{\psi}^{(t)}+\eta\bm{v}^{(t)}
	\le
	(1-\eta)\bm{\psi}^{(t)}+\eta\bm{1}
	\le
	\bm{1}.
	\]
	Therefore, $\bm{\psi}^{(t)}\in[0,1]^d$ for every $t$.
	
	We next establish two consequences of $\gamma$-weak
	DR-submodularity that will be used below. For any
	$\bm{x},\bm{y}\in[0,1]^d$ satisfying $\bm{x}\le\bm{y}$,
	Eq.~(\ref{eq:weak_dr_definition}) implies
	\begin{equation}
		\left\langle
		\nabla g_\theta(\bm{x}),
		\bm{y}-\bm{x}
		\right\rangle
		\ge
		\gamma
		\left(
		g_\theta(\bm{y})-g_\theta(\bm{x})
		\right).
		\label{eq:weak_dr_ascent}
	\end{equation}
	Indeed, letting
	$\bm{z}\triangleq\bm{y}-\bm{x}\ge\bm{0}$ gives
	\begin{align}
		g_\theta(\bm{y})-g_\theta(\bm{x})
		&=
		\int_0^1
		\left\langle
		\nabla g_\theta(\bm{x}+s\bm{z}),
		\bm{z}
		\right\rangle ds \nonumber\\
		&\le
		\frac{1}{\gamma}
		\left\langle
		\nabla g_\theta(\bm{x}),
		\bm{z}
		\right\rangle,
	\end{align}
	where the inequality follows from
	$\bm{x}\le\bm{x}+s\bm{z}$,
	$\bm{z}\ge\bm{0}$, and
	Eq.~(\ref{eq:weak_dr_definition}).
	
	In addition, nonnegativity and $\gamma$-weak DR-submodularity
	imply that, for any $\bm{x},\bm{y}\in[0,1]^d$,
	\begin{equation}
		g_\theta(\bm{x}\vee\bm{y})
		\ge
		\left(1-\gamma\|\bm{x}\|_\infty\right)
		g_\theta(\bm{y}).
		\label{eq:join_lower_bound}
	\end{equation}
	To establish this inequality, let
	$r\triangleq\|\bm{x}\|_\infty$. The result is immediate when
	$r=0$. Otherwise, define
	\[
	\bm{z}
	\triangleq
	\bm{x}\vee\bm{y}-\bm{y}
	\ge
	\bm{0}.
	\]
	For every $s\in[0,1/r]$, both
	$\bm{y}+rs\bm{z}$ and $\bm{y}+s\bm{z}$ belong to $[0,1]^d$,
	and
	$\bm{y}+rs\bm{z}\le\bm{y}+s\bm{z}$. It follows that
	\begin{align}
		g_\theta(\bm{x}\vee\bm{y})-g_\theta(\bm{y})
		&=
		r\int_0^{1/r}
		\left\langle
		\nabla g_\theta(\bm{y}+rs\bm{z}),
		\bm{z}
		\right\rangle ds \nonumber\\
		&\ge
		r\gamma\int_0^{1/r}
		\left\langle
		\nabla g_\theta(\bm{y}+s\bm{z}),
		\bm{z}
		\right\rangle ds \nonumber\\
		&=
		r\gamma
		\left[
		g_\theta\left(\bm{y}+\frac{\bm{z}}{r}\right)
		-
		g_\theta(\bm{y})
		\right] \nonumber\\
		&\ge
		-r\gamma g_\theta(\bm{y}),
	\end{align}
	where the last inequality follows from
	$g_\theta(\cdot)\ge0$. Rearranging proves
	Eq.~(\ref{eq:join_lower_bound}).
	
	We now apply these inequalities to bound the progress at each
	iteration. Consider the comparison direction
	\begin{equation}
		\bm{d}^{(t)}
		\triangleq
		\bm{\psi}_{\mathrm{OPT}}
		\vee
		\bm{\psi}^{(t)}
		-
		\bm{\psi}^{(t)}
		=
		\left(
		\bm{\psi}_{\mathrm{OPT}}
		-
		\bm{\psi}^{(t)}
		\right)_{+}.
		\label{eq:comparison_direction}
	\end{equation}
	Since
	\[
	\bm{0}
	\le
	\bm{d}^{(t)}
	\le
	\bm{1}-\bm{\psi}^{(t)},
	\]
	the vector $\bm{d}^{(t)}$ is feasible for the linear optimization
	problem in Eq.~(\ref{eq:continuous_greedy_update}). The optimality of
	$\bm{v}^{(t)}$, together with
	Eqs.~(\ref{eq:weak_dr_ascent}) and
	(\ref{eq:join_lower_bound}), therefore yields
	\begin{align}
		\left\langle
		\nabla g_\theta(\bm{\psi}^{(t)}),
		\bm{v}^{(t)}
		\right\rangle
		&\ge
		\left\langle
		\nabla g_\theta(\bm{\psi}^{(t)}),
		\bm{d}^{(t)}
		\right\rangle \nonumber\\
		&\ge
		\gamma
		\left[
		g_\theta\left(
		\bm{\psi}^{(t)}
		\vee
		\bm{\psi}_{\mathrm{OPT}}
		\right)
		-
		G_t
		\right] \nonumber\\
		&\ge
		\gamma
		\left[
		\left(
		1-\gamma
		\|\bm{\psi}^{(t)}\|_\infty
		\right)
		G_{\mathrm{OPT}}
		-
		G_t
		\right].
		\label{eq:greedy_direction_bound}
	\end{align}
	
	For every coordinate $i$, the direction constraint and
	$\gamma\le1$ imply
	\[
	v_i^{(t)}
	\le
	1-\psi_i^{(t)}
	\le
	1-\gamma\psi_i^{(t)}.
	\]
	Consequently,
	\begin{align}
		1-\gamma\psi_i^{(t+1)}
		&=
		1-\gamma\psi_i^{(t)}
		-
		\eta\gamma v_i^{(t)} \nonumber\\
		&\ge
		(1-\eta\gamma)
		\left(
		1-\gamma\psi_i^{(t)}
		\right).
	\end{align}
	Since $\bm{\psi}^{(0)}=\bm{0}$, induction gives
	\[
	1-\gamma\psi_i^{(t)}
	\ge
	(1-\eta\gamma)^t
	\]
	for every coordinate $i$. Hence,
	\begin{equation}
		1-\gamma\|\bm{\psi}^{(t)}\|_\infty
		\ge
		(1-\eta\gamma)^t.
		\label{eq:slack_bound}
	\end{equation}
	Combining Eqs.~(\ref{eq:greedy_direction_bound}) and
	(\ref{eq:slack_bound}), we obtain
	\begin{equation}
		\left\langle
		\nabla g_\theta(\bm{\psi}^{(t)}),
		\bm{v}^{(t)}
		\right\rangle
		\ge
		\gamma
		\left[
		(1-\eta\gamma)^tG_{\mathrm{OPT}}
		-
		G_t
		\right].
		\label{eq:direction_final}
	\end{equation}
	
	By $L$-smoothness and Eq.~(\ref{eq:direction_final}),
	\begin{align}
		G_{t+1}
		&\ge
		G_t
		+
		\eta
		\left\langle
		\nabla g_\theta(\bm{\psi}^{(t)}),
		\bm{v}^{(t)}
		\right\rangle
		-
		\frac{L\eta^2}{2}
		\|\bm{v}^{(t)}\|_2^2 \nonumber\\
		&\ge
		(1-\eta\gamma)G_t
		+
		\eta\gamma(1-\eta\gamma)^tG_{\mathrm{OPT}}
		-
		\frac{L\eta^2d}{2},
		\label{eq:one_step_recurrence}
	\end{align}
	where the last inequality also uses
	$\|\bm{v}^{(t)}\|_2^2\le d$.
	
	Letting $a\triangleq1-\eta\gamma$ and recursively applying
	Eq.~(\ref{eq:one_step_recurrence}) yields
	\begin{align}
		G_T
		&\ge
		a^TG_0
		+
		T\eta\gamma a^{T-1}G_{\mathrm{OPT}}
		-
		\frac{L\eta^2d}{2}
		\sum_{t=0}^{T-1}a^{T-1-t} \nonumber\\
		&\ge
		T\eta\gamma a^{T-1}G_{\mathrm{OPT}}
		-
		\frac{TL\eta^2d}{2},
		\label{eq:unrolled_recurrence}
	\end{align}
	where we used $G_0\ge0$, $a\in[0,1]$, and
	$\sum_{t=0}^{T-1}a^{T-1-t}\le T$.
	Substituting $\eta=1/T$ gives
	\begin{equation}
		G_T
		\ge
		\gamma
		\left(1-\frac{\gamma}{T}\right)^{T-1}
		G_{\mathrm{OPT}}
		-
		\frac{Ld}{2T}.
		\label{eq:finite_iteration_result}
	\end{equation}
	
	For $\gamma\in(0,1]$ and $T\ge2$, applying
	$\log(1-u)\ge-u/(1-u)$ for $u\in[0,1)$ gives
	\[
	(T-1)
	\log\left(1-\frac{\gamma}{T}\right)
	\ge
	-\frac{\gamma(T-1)}{T-\gamma}
	\ge
	-\gamma,
	\]
	and hence
	\[
	\left(1-\frac{\gamma}{T}\right)^{T-1}
	\ge
	e^{-\gamma}.
	\]
	Substituting this inequality into
	Eq.~(\ref{eq:finite_iteration_result}) and recalling that
	$d=|V|$, we conclude that
	\begin{equation}
		g_\theta\left(\bm{\psi}^{(T)}\right)
		\ge
		\gamma e^{-\gamma}
		g_\theta\left(\bm{\psi}_{\mathrm{OPT}}\right)
		-
		\frac{L|V|}{2T}.
		\label{eq:final_approximation}
	\end{equation}
	Let $\bm{\psi}_T$ denote the terminal output obtained with
	step size $1/T$. By compactness of $\mathcal{X}$, the sequence
	of terminal outputs has accumulation points. Let
	$\bm{\psi}^{*}$ be any such point and choose
	$T_j\to\infty$ such that
	$\bm{\psi}_{T_j}\to\bm{\psi}^{*}$.
	By continuity of $g_\theta$ and
	Eq.~(\ref{eq:final_approximation}),
	\begin{equation}
		g_\theta(\bm{\psi}^{*})
		=
		\lim_{j\to\infty}g_\theta(\bm{\psi}_{T_j})
		\ge
		\gamma e^{-\gamma}
		g_\theta(\bm{\psi}_{\mathrm{OPT}}).
	\end{equation}
	This proves part 1 for the limiting continuous-greedy output.
	
	We now prove part 2 for projected gradient ascent.
	Specifically, consider
	\begin{equation}
		\bm{\psi}^{(k+1)}
		=
		\Pi_{\mathcal{X}}
		\left(
		\bm{\psi}^{(k)}
		+
		\alpha\nabla g_\theta(\bm{\psi}^{(k)})
		\right),
		\qquad
		0<\alpha<\frac{1}{L},
		\label{eq:projected_gradient_update}
	\end{equation}
	where $\Pi_{\mathcal{X}}(\cdot)$ denotes the Euclidean projection
	onto $\mathcal{X}$. The first-order optimality condition of the
	projection gives
	\begin{equation}
		\left\langle
		\bm{\psi}^{(k)}
		+
		\alpha\nabla g_\theta(\bm{\psi}^{(k)})
		-
		\bm{\psi}^{(k+1)},
		\bm{\psi}
		-
		\bm{\psi}^{(k+1)}
		\right\rangle
		\le0,
		\qquad
		\forall\,\bm{\psi}\in\mathcal{X}.
		\label{eq:projection_optimality}
	\end{equation}
	Setting $\bm{\psi}=\bm{\psi}^{(k)}$ in
	Eq.~(\ref{eq:projection_optimality}) yields
	\begin{equation}
		\left\langle
		\nabla g_\theta(\bm{\psi}^{(k)}),
		\bm{\psi}^{(k+1)}-\bm{\psi}^{(k)}
		\right\rangle
		\ge
		\frac{1}{\alpha}
		\left\|
		\bm{\psi}^{(k+1)}-\bm{\psi}^{(k)}
		\right\|_2^2.
		\label{eq:projected_direction_bound}
	\end{equation}
	
	Combining Eq.~(\ref{eq:projected_direction_bound}) with
	$L$-smoothness gives
	\begin{align}
		g_\theta(\bm{\psi}^{(k+1)})
		&\ge
		g_\theta(\bm{\psi}^{(k)})
		+
		\left\langle
		\nabla g_\theta(\bm{\psi}^{(k)}),
		\bm{\psi}^{(k+1)}-\bm{\psi}^{(k)}
		\right\rangle \nonumber\\
		&\quad
		-
		\frac{L}{2}
		\left\|
		\bm{\psi}^{(k+1)}-\bm{\psi}^{(k)}
		\right\|_2^2 \nonumber\\
		&\ge
		g_\theta(\bm{\psi}^{(k)})
		+
		\left(
		\frac{1}{\alpha}-\frac{L}{2}
		\right)
		\left\|
		\bm{\psi}^{(k+1)}-\bm{\psi}^{(k)}
		\right\|_2^2 \nonumber\\
		&\ge
		g_\theta(\bm{\psi}^{(k)})
		+
		\frac{1}{2\alpha}
		\left\|
		\bm{\psi}^{(k+1)}-\bm{\psi}^{(k)}
		\right\|_2^2,
		\label{eq:smooth_ascent}
	\end{align}
	where the last inequality follows from $\alpha<1/L$.
	Thus, the objective values are monotonically nondecreasing.
	Since $g_\theta$ is continuous and $\mathcal{X}$ is compact,
	$g_\theta$ is bounded from above on $\mathcal{X}$. Summing
	Eq.~(\ref{eq:smooth_ascent}) over $k$ therefore gives
	\begin{equation}
		\sum_{k=0}^{\infty}
		\left\|
		\bm{\psi}^{(k+1)}-\bm{\psi}^{(k)}
		\right\|_2^2
		<\infty,
	\end{equation}
	and consequently
	\begin{equation}
		\lim_{k\to\infty}
		\left\|
		\bm{\psi}^{(k+1)}-\bm{\psi}^{(k)}
		\right\|_2
		=0.
		\label{eq:successive_difference}
	\end{equation}
	
	To express this result in terms of first-order stationarity, define
	the projected gradient mapping
	\begin{equation}
		\mathcal{G}_{\alpha}(\bm{\psi})
		\triangleq
		\frac{1}{\alpha}
		\left[
		\Pi_{\mathcal{X}}
		\left(
		\bm{\psi}
		+
		\alpha\nabla g_\theta(\bm{\psi})
		\right)
		-
		\bm{\psi}
		\right].
		\label{eq:projected_gradient_mapping}
	\end{equation}
	By Eqs.~(\ref{eq:projected_gradient_update}) and
	(\ref{eq:successive_difference}),
	\begin{equation}
		\lim_{k\to\infty}
		\left\|
		\mathcal{G}_{\alpha}(\bm{\psi}^{(k)})
		\right\|_2
		=0.
		\label{eq:stationarity_convergence}
	\end{equation}
	Because $\mathcal{X}$ is compact, the sequence
	$\{\bm{\psi}^{(k)}\}$ admits at least one accumulation point.
	Let $\bar{\bm{\psi}}$ be any such point. The continuity of the
	projected gradient mapping and
	Eq.~(\ref{eq:stationarity_convergence}) imply
	$\mathcal{G}_{\alpha}(\bar{\bm{\psi}})=\bm{0}$.
	Equivalently,
	\begin{equation}
		\left\langle
		\nabla g_\theta(\bar{\bm{\psi}}),
		\bm{\psi}-\bar{\bm{\psi}}
		\right\rangle
		\le0,
		\qquad
		\forall\,\bm{\psi}\in\mathcal{X}.
		\label{eq:stationary_condition}
	\end{equation}
	Therefore, projected gradient ascent converges to first-order
	stationarity in the sense that its projected gradient mapping
	vanishes and every accumulation point is stationary. Together
	with Eq.~(\ref{eq:final_approximation}), this establishes both
	the approximation and convergence guarantees stated in the
	theorem.
	
	The two guarantees can also be combined through initialization.
	Let $\widehat{\bm{\psi}}$ be the continuous-greedy solution
	from part 1. With $\theta$ fixed, initialize projected gradient
	ascent in Eq.~(\ref{eq:projected_gradient_update}) at
	$\bm{\psi}^{(0)}=\widehat{\bm{\psi}}$, using
	$0<\alpha<1/L$.
	By Eq.~(\ref{eq:smooth_ascent}), the objective values are
	nondecreasing, so for every $k\geq 0$,
	\begin{equation}
		g_\theta(\bm{\psi}^{(k)})
		\geq g_\theta(\widehat{\bm{\psi}})
		\geq \gamma e^{-\gamma}
		g_\theta(\bm{\psi}_{\mathrm{OPT}}).
	\end{equation}
	The preceding stationarity argument applies to this initialization.
	Thus, every accumulation point $\bar{\bm{\psi}}$ is first-order
	stationary and, by continuity of $g_\theta$, satisfies
	\begin{equation}
		g_\theta(\bar{\bm{\psi}})
		\geq \gamma e^{-\gamma}
		g_\theta(\bm{\psi}_{\mathrm{OPT}}),
		\qquad
		\mathcal{G}_{\alpha}(\bar{\bm{\psi}})=\bm{0}.
	\end{equation}
	Hence, projected gradient refinement preserves the
	continuous-greedy approximation guarantee while attaining
	first-order stationarity.
\end{proof}

\subsection{Proof of Corollary \ref{corollary1}} \label{appendix:Proof of Corollary1}
\ELBOCorollary*
\begin{proof}
	At unit temperature, $k_BT=1$, the variational free energy
	in Eq.~(\ref{reformulated Helmholtz free energy}) becomes
	\begin{equation}\label{eq:free_energy_unitT}
		\mathcal{F}_\theta(\bm{\psi})
		=
		\mathbb{E}_{S\sim q(\bm{\psi})}[E_\theta(S;V)]
		-
		\mathbb{H}(q(\bm{\psi})).
	\end{equation}
	Since $E_\theta(S;V)=-F_\theta(S;V)$, the definition of the
	multilinear extension in Eq.~(\ref{f_mt}) gives
	\begin{align}
		\mathcal{F}_\theta(\bm{\psi})
		=
		-f_{\mathrm{mt}}^{F_\theta}(\bm{\psi})
		-\mathbb{H}(q(\bm{\psi})) \nonumber
		=
		-\mathrm{ELBO}(\bm{\psi}),
	\end{align}
	where the last equality follows from Eq.~(\ref{ELBO}).
	Therefore, maximizing the ELBO is equivalent to minimizing
	the variational free energy.
\end{proof}

\subsection{Derivation of the Discrete Approximation Guarantee}
\label{appendix:Derivation of the Discrete Approximation Guarantee}

Let $F_\theta(S):=g_\theta(\mathbf{1}_S)$ and write
$f:=f_{\mathrm{mt}}^{F_\theta}$ for brevity.
Under the assumptions of Theorem~\ref{theorem1}, consider
the continuous-greedy trajectory $\bm{\psi}(t)$,
$t\in[0,1]$, starting from $\bm{0}$ and ending at
$\bm{\psi}^*$.
Let $\bm{v}_g(t)$ be the direction selected by $g_\theta$.
Directional consistency assumes that this direction also
maximizes the linearized objective $f$ along the trajectory:
\begin{equation}
	\bm{v}_g(t)\in
	\arg\max_{\bm{0}\le\bm{v}\le\bm{1}-\bm{\psi}(t)}
	\left\langle\nabla f(\bm{\psi}(t)),\bm{v}\right\rangle.
	\label{eq:directional_consistency}
\end{equation}
For the box constraints, matching gradient signs is
sufficient for this condition.
Appendix~\ref{appendix:gradient_alignment} provides
complementary empirical evidence of local gradient alignment,
with median cosine similarities of $0.845$--$0.982$ at sampled
probability vectors on synthetic sets.
These results support the use of $g_\theta$ to guide local
optimization, while the guarantee below relies on the exact
condition in Eq.~(\ref{eq:directional_consistency}).

We first show that $f$ satisfies the structural conditions
required by Theorem~\ref{theorem1}.
Since $g_\theta\ge0$, both $F_\theta$ and $f$ are nonnegative.
For $A\subseteq V\setminus\{i\}$, define the discrete
marginal gain
$\Delta_iF_\theta(A):=F_\theta(A\cup\{i\})-F_\theta(A)$.
For any $A\subseteq B\subseteq V\setminus\{i\}$,
the weak DR condition in Eq.~(\ref{eq:weak_dr_definition})
implies
\begin{align*}
	\Delta_iF_\theta(A)
	=\int_0^1
	\partial_i g_\theta(\mathbf{1}_A+s\bm{e}_i)\,ds
	\ge\gamma\int_0^1
	\partial_i g_\theta(\mathbf{1}_B+s\bm{e}_i)\,ds
	=\gamma\Delta_iF_\theta(B),
\end{align*}
where $\bm{e}_i$ is the $i$-th standard basis vector.
For any $\bm{x}\le\bm{y}$, couple the independent random
subsets $R_{\bm{x}},R_{\bm{y}}\subseteq V\setminus\{i\}$
with inclusion probabilities $x_j$ and $y_j$, respectively,
using the same uniform random variables.
Then $R_{\bm{x}}\subseteq R_{\bm{y}}$ almost surely, and
the derivative formula for the multilinear extension gives
\begin{align*}
	\partial_i f(\bm{x})
	=\mathbb{E}\left[\Delta_iF_\theta(R_{\bm{x}})\right]
	\ge\gamma\mathbb{E}
	\left[\Delta_iF_\theta(R_{\bm{y}})\right]
	=\gamma\partial_i f(\bm{y}).
\end{align*}
Thus, $f$ is $\gamma$-weakly DR-submodular.
Moreover, as a multilinear polynomial on the compact box
$\mathcal{X}$, it is $L_f$-smooth for some finite $L_f$.

By Eq.~(\ref{eq:directional_consistency}),
$\bm{\psi}(t)$ is also a continuous-greedy trajectory for
$f$. Applying Theorem~\ref{theorem1}(1) to $f$ therefore
yields
\begin{equation}
	f(\bm{\psi}^*)
	\ge\gamma e^{-\gamma}
	\max_{\bm{\psi}\in\mathcal{X}}f(\bm{\psi}).
	\label{eq:multilinear_directional_bound}
\end{equation}
The multilinear extension is an expectation of discrete
function values, so it cannot exceed
$\max_{S\subseteq V}F_\theta(S)$.
Equality is attained at the indicator vector of an optimal
subset, giving
\[
\max_{\bm{\psi}\in\mathcal{X}}f(\bm{\psi})
=\max_{S\subseteq V}F_\theta(S)
=F_\theta(S_{\mathrm{OPT}}).
\]

Finally, independent Bernoulli rounding
$S^*\sim q(\bm{\psi}^*)$ preserves the multilinear value
in expectation:
\begin{align*}
	\mathbb{E}_{S^*\sim q(\bm{\psi}^*)}[F_\theta(S^*)]
	=f(\bm{\psi}^*)
	\ge\gamma e^{-\gamma}F_\theta(S_{\mathrm{OPT}}).
\end{align*}
This establishes Eq.~(\ref{eq:discrete_bound_with_solution_error}).

\subsection{Entropy-Regularized Refinement}
\label{appendix:Entropy-Regularized Refinement}

Under the assumptions of Theorem~\ref{theorem1}, fix
$\theta$ and $T>0$, and define
\begin{equation}
	h(\bm{\psi})
	:=g_\theta(\bm{\psi})+k_BT\mathbb{H}(q(\bm{\psi})),
	\qquad
	\mathcal{X}_\delta=[\delta,1-\delta]^{|V|},
	\quad 0<\delta<\tfrac12.
	\label{eq:entropy_refinement_objective}
\end{equation}
The initialization and each displayed update in
Eqs.~(\ref{gradient descent}) and~(\ref{add noise}) are
projected onto $\mathcal{X}_\delta$; below,
$\bm{\psi}^{(k)}$ denotes the feasible iterate.
We analyze deterministic refinement, with $k_BT=1$
recovering the entropy coefficient used in the main text.

\paragraph{Convergence.}
The Bernoulli entropy has Hessian
$-\operatorname{diag}(1/[\psi_i(1-\psi_i)])$.
Hence $h$ is $L_\delta$-smooth on $\mathcal{X}_\delta$, where
\begin{equation}
	L_\delta=L+\frac{k_BT}{\delta(1-\delta)}.
	\label{eq:entropy_refinement_smoothness}
\end{equation}
For a fixed $0<\alpha<1/L_\delta$, consider
\begin{equation}
	\bm{\psi}^{(k+1)}
	=\Pi_{\mathcal{X}_\delta}
	\left(\bm{\psi}^{(k)}+\alpha\nabla h(\bm{\psi}^{(k)})\right).
	\label{eq:entropy_projected_update}
\end{equation}
Applying the ascent argument in
Appendix~\ref{appendix:Proof of Theorem1} to $h$ gives
\begin{equation}
	h(\bm{\psi}^{(k+1)})
	\geq h(\bm{\psi}^{(k)})
	+\frac{1}{2\alpha}
	\|\bm{\psi}^{(k+1)}-\bm{\psi}^{(k)}\|_2^2.
	\label{eq:entropy_smooth_ascent}
\end{equation}
Since $h$ is bounded above, summing this inequality shows that
the projected gradient mapping vanishes. By compactness and
continuity, every accumulation point $\bar{\bm{\psi}}$ satisfies
\begin{equation}
	\langle\nabla h(\bar{\bm{\psi}}),
	\bm{\psi}-\bar{\bm{\psi}}\rangle\leq0,
	\qquad\forall\,\bm{\psi}\in\mathcal{X}_\delta.
	\label{eq:entropy_stationarity}
\end{equation}

\paragraph{Utility after refinement.}
Let $\bm{\psi}^{*}$ be the continuous-greedy solution in
Theorem~\ref{theorem1}(1), and initialize at
$\bm{\psi}^{(0)}=\Pi_{\mathcal{X}_\delta}(\bm{\psi}^{*})$.
Write
$\varepsilon_\delta
=\max\{0,g_\theta(\bm{\psi}^{*})-g_\theta(\bm{\psi}^{(0)})\}$
for the initialization loss. Since
$\|\bm{\psi}^{(0)}-\bm{\psi}^{*}\|_2\leq\delta\sqrt{|V|}$,
smoothness yields
\begin{equation}
	\varepsilon_\delta
	\leq\delta\sqrt{|V|}\,
	\|\nabla g_\theta(\bm{\psi}^{*})\|_2
	+\frac{L|V|\delta^2}{2}.
	\label{eq:entropy_initialization_loss}
\end{equation}
Monotonicity and continuity give
$h(\bar{\bm{\psi}})\geq h(\bm{\psi}^{(0)})$.
Using $\mathbb{H}(q(\bm{\psi}))\leq|V|\log2$
and Theorem~\ref{theorem1}(1), we obtain
\begin{align}
	g_\theta(\bar{\bm{\psi}})
	&\geq g_\theta(\bm{\psi}^{(0)})
	+k_BT\left[
	\mathbb{H}(q(\bm{\psi}^{(0)}))
	-\mathbb{H}(q(\bar{\bm{\psi}}))\right]
	\nonumber\\
	&\geq\gamma e^{-\gamma}
	g_\theta(\bm{\psi}_{\mathrm{OPT}})-\varepsilon_\delta
	-k_BT\left[|V|\log2-
	\mathbb{H}(q(\bm{\psi}^{(0)}))\right].
	\label{eq:entropy_refined_utility}
\end{align}
Thus the same limit point is stationary for $h$ and retains
a utility guarantee with an entropy-dependent correction.
A finite-step continuous-greedy initializer additionally
incurs the discretization loss in Eq.~(\ref{eq:final_approximation}).

\paragraph{Discrete output.}
For $S\sim q(\bar{\bm{\psi}})$ and
$F_\theta(S)=g_\theta(\mathbf{1}_S)$, smoothness gives
\begin{equation}
	\mathbb{E}[F_\theta(S)]
	\geq g_\theta(\bar{\bm{\psi}})
	-\frac{L}{2}\sum_{i\in V}\bar{\psi}_i(1-\bar{\psi}_i).
	\label{eq:entropy_rounding_transfer}
\end{equation}
Indeed, the linear term in the smoothness inequality has zero
expectation because $\mathbb{E}[\mathbf{1}_S]=\bar{\bm{\psi}}$,
and the expected squared deviation is the sum of the Bernoulli
variances. Combining this with
Eq.~(\ref{eq:entropy_refined_utility}) and
$g_\theta(\bm{\psi}_{\mathrm{OPT}})\geq F_\theta(S_{\mathrm{OPT}})$
yields
\begin{align}
	\mathbb{E}[F_\theta(S)]
	\geq{}&\gamma e^{-\gamma}F_\theta(S_{\mathrm{OPT}})
	-\varepsilon_\delta
	-k_BT\left[|V|\log2-
	\mathbb{H}(q(\bm{\psi}^{(0)}))\right]
	\nonumber\\
	&-\frac{L}{2}\sum_{i\in V}\bar{\psi}_i(1-\bar{\psi}_i).
	\label{eq:entropy_refined_discrete_bound}
\end{align}

\section{Details of Differentiable Mean Field Variational Inference in OS Oracle}

\subsection{Derivations of the Fixed Point Iteration} \label{appendix:Derivations of the Fixed Point Iteration}
Following \citet{ou2022learning}, we present a detailed derivation of the fixed-point iteration (FPI) associated with Eq.~(\ref{FPI}):
\begin{align} \label{app-fpi-formula}
	\psi_i^{(k+1)} \leftarrow \bigl(1 + \exp\bigl(- \nabla_{\psi_i^{(k)}} f_{\mathrm{mt}}^{F_\theta}(\boldsymbol{\psi}^{(k)})\bigr)\bigr)^{-1}.
\end{align}
We begin by recalling that the objective is to maximize the evidence lower bound (ELBO),
\begin{align}
	\max_{\boldsymbol{\psi} \in [0,1]^{|V|}}
	\quad
	\underbrace{
		\sum_{S \subseteq V} F_\theta(S)
		\prod_{i \in S} \psi_i
		\prod_{i \notin S} (1 - \psi_i)
	}_{f_{\mathrm{mt}}^{F_\theta}(\boldsymbol{\psi})}
	-
	\underbrace{
		\sum_{i=1}^{|V|}
		\bigl[
		\psi_i \log \psi_i
		+
		(1-\psi_i)\log(1-\psi_i)
		\bigr]
	}_{-\mathbb{H}(q(S;\boldsymbol{\psi}))}.
\end{align}
The update in Eq.~(\ref{app-fpi-formula}) is obtained by enforcing stationarity of the ELBO with respect to coordinate $\psi_i$. Specifically, setting the partial derivative to zero yields
\begin{align}
	\nabla{\psi_i} f_{\mathrm{mt}}^{F_\theta}(\boldsymbol{\psi})
	+ \nabla_{\psi_i} \mathbb{H}(q(S;\boldsymbol{\psi}))
	= \nabla_{\psi_i} f_{\mathrm{mt}}^{F_\theta}(\boldsymbol{\psi})
	+ \log\frac{1-\psi_i}{\psi_i}
	= 0. \nonumber
\end{align}
Solving for $\psi_i$ leads to
\begin{align}
	\psi_i
	= \bigl(1 + \exp\bigl(-\nabla_{\psi_i} f_{\mathrm{mt}}^{F_\theta}(\boldsymbol{\psi})\bigr)\bigr)^{-1}, \nonumber
\end{align}
which coincides with the fixed-point update employed in mean-field variational inference.

\subsection{Details of Monte Carlo Gradient Estimation} \label{appendix:Details of Monte Carlo Gradient Estimation}
According to Eq.~(\ref{app-fpi-formula}), the variational parameters can be updated via fixed-point iteration.
Nevertheless, a major computational challenge arises from evaluating the gradient of the multilinear extension $f_{\mathrm{mt}}^{F_\theta} (\boldsymbol{\psi})$ defined in Eq.~(\ref{f_mt}), as it involves summation over all $2^{|V|}$ possible subsets.
To address this issue, \citet{ou2022learning} showed that the gradient $\nabla_{\boldsymbol{\psi}} f_{\mathrm{mt}}^{F_\theta}$
admits an expectation form, which enables efficient estimation through Monte Carlo sampling. Specifically, the partial derivative with respect to $\psi_i$ can be derived as
\begin{align} \label{appendix_gradient_mt}
	\nabla_{\psi_i} f_{\mathrm{mt}}^{F_\theta}
	&= \nabla_{\psi_i} \sum\limits_{S \subseteq V} F_\theta (S) \prod_{i \in S}\psi_i \prod_{i \not\in S}(1-\psi_i) \nonumber \\
	&= \mathbb{E}_{q(S;(\boldsymbol{\psi} | \psi_i \leftarrow 1))}[F_\theta (S)] - \mathbb{E}_{q(S;(\boldsymbol{\psi} | \psi_i \leftarrow 0))}[F_\theta (S)] \nonumber \\
	&= \sum\limits_{S \subseteq V, i \in S} F_\theta (S) \prod\limits_{j\in S\backslash\{i\}} \psi_j \prod\limits_{j^\prime \not\in S} (1 - \psi_{j^\prime}) - \sum\limits_{S \subseteq V\backslash \{i\}} F_\theta (S) \prod\limits_{j\in S} \psi_j \prod\limits_{j^\prime \in V \setminus (S \cup \{i\})} (1 - \psi_{j^\prime}) \nonumber \\
	&= \sum\limits_{S \subseteq V\backslash \{i\}} [F_\theta (S + i) - F_\theta (S)] \prod\limits_{j\in S} \psi_j \prod\limits_{j^\prime \in V\backslash S \backslash\{i\} } (1 - \psi_{j^\prime}) \nonumber \\
	&= \mathbb{E}_{q(S;(\boldsymbol{\psi} | \psi_i \leftarrow 0))} \left[ F_\theta (S + i) - F_\theta (S) \right]. 
\end{align}
Here, $S+i$ denotes the union $S \cup \{i\}$.
Based on Eq.~(\ref{appendix_gradient_mt}), the gradient $\nabla_{\psi_i} f_{\mathrm{mt}}^{F_\theta}$ can be approximated using Monte Carlo estimation:
i) sample $m$ subsets $S_n, n=1,\dots,m$ from the variational distribution $q(S;(\boldsymbol{\psi} | \psi_i \leftarrow 0))$; ii) approximate the expectation by the average $\frac{1}{m} \sum_{n=1}^m [F_\theta (S_n + i) - F_\theta (S_n)]$.
This procedure provides a tractable and unbiased estimator of the gradient required for the fixed-point updates.

\subsection{Accelerating Inference with Equivariant Neural Networks}\label{appendix:Accelerating Inference with Equivariant Neural Networks}
Although Monte Carlo–based gradient estimation makes the optimization of the variational parameters $\boldsymbol{\psi}$ feasible, it is accompanied by a notorious drawback: the computational cost is often prohibitively high. In particular, the DiffMF procedure in Eq.~(\ref{MFVI}) typically involves an inner sampling loop for each data instance, which substantially increases the overall training overhead.
To improve scalability on large-scale datasets, prior works \citep{ou2022learning, xie2024enhancing} propose to amortize the approximate inference process by introducing an auxiliary recognition network. This network directly predicts the variational parameters $\boldsymbol{\psi}$ of the distribution $q_\phi(S;\boldsymbol{\psi})$, where $\phi$ denotes the learnable parameters of the recognition model.
Crucially, the recognition network must be permutation equivariant:
permuting the input elements should permute the output probabilities
in the same way, preserving the correspondence between each
element and its variational parameter.
In our framework, optimizing each data point similarly requires multiple gradient ascent steps as in Eq.~(\ref{gradient descent}). Therefore, we likewise employ an auxiliary recognition network to accelerate the inference procedure. 
Concretely, the recognition network provides an initialization
$\boldsymbol{\psi}^{(0)}
= \operatorname{RecNet}(V;\phi)\in[0,1]^{|V|}$.
Each output coordinate specifies the inclusion probability of
the corresponding element in $V$, jointly defining the initial
variational distribution
$q_\phi(S;\boldsymbol{\psi}^{(0)})$.

\section{Incorporate Sufficient Invariant Statistics of Background Information}\label{appendix:Incorporate Sufficient Invariant Statistics of Background Information}
Within our framework, we adopt the architecture proposed by \citet{xie2024enhancing} and implement its relaxed variant, INSET-$\operatorname{R}$. The central idea of INSET is to incorporate contextual information from the ground set $V$ into the set function, yielding a formulation of the form $F(V, S)$, in contrast to EquiVSet, which models the set function solely as $F(S)$. This mechanism is referred to by \citet{xie2024enhancing} as the fusion of \textit{invariant sufficient representations} of the subset $S$ and the ground set $V$.
Specifically, \citet{bloem2020probabilistic} show that the empirical measure $M_S(S) = \sum_{s_i \in S} \delta(s_i)$
can be served as a valid invariant sufficient representation, where $\delta(s_i)$ denotes a unit-mass atom located at $s_i$, for example instantiated via one-hot encodings. Furthermore, building on the result of
\citet{zaheer2017deep}, this empirical measure can be efficiently approximated using the aggregation form $\rho \left( \sum_{s \in S} \phi(s) \right)$. This formulation provides a practical and effective means of realizing invariant sufficient representations in neural architectures.
\begin{proposition}
	If $f$ is a valid permutation invariant function on $S$, it can be approximated arbitrarily close in the form of $f(S) = \rho \left( \sum_{s \in S} \phi(s) \right)$, for suitable transformations $\phi$ and $\rho$.
\end{proposition}
In practice, an encoder $\phi(\cdot)$ is employed to produce element-wise embeddings, while the mapping $\rho(\cdot)$ may be instantiated by a variety of feedforward architectures, including fully connected layers equipped with nonlinear activations. In a related line of work, \citet{maron2020learningsetssymmetricelements} establish that universal approximation of invariant sufficient representations can be achieved via the aggregation form $\sum_{S \in V} \sum_{s \in S} \phi(s)$, which simplifies to the equivalent expression $\sum_{x_j \in V} \phi(x_j)$.
Therefore, the neural network architecture for set functions injected with superset background information can be expressed as
\begin{equation}
	\theta(S, V) = \sigma \left( \theta_1 \left( \sum_{i=1}^{n_i} \phi(x_i) \right) + \theta_2 \left( \sum_{j=1}^{n} \phi(x_j) \right) \right).
\end{equation}
Here, the feed-forward modules $\theta_1$ and $\theta_2$ are accompanied by a non-linear activation layer denoted
by $\sigma$.
As a comparison, we provide the set function architectures for four different optimal subset oracles in Table~\ref{tab: architectures of set function}.

\begin{table}[h]
	\centering
	\caption{Comparison of set function architectures across four different optimal subset oracles.}
	\label{tab: architectures of set function}
	\small 
	\setlength{\tabcolsep}{8pt} 
	\begin{tabular}{@{}cccc@{}}
		\toprule
		\textbf{EquiVSet} & \textbf{EquiVSet-$\operatorname{R}$} & \textbf{INSET} & \textbf{INSET-$\operatorname{R}$} \\
		\cmidrule(r){1-1} \cmidrule(lr){2-2} \cmidrule(lr){3-3} \cmidrule(l){4-4}
		$\mathrm{FC}_{S}(256 ,500, \mathrm{ReLU})$ 
		& $\mathrm{MHA}_{S}(256,4)$ 
		& $\mathrm{FC}_{V}(256 ,500, \mathrm{ReLU})$ 
		& $\mathrm{MHA}_{V}(256,4)$ \\
		
		$\mathrm{FC}_{S}(500,1, -)$ 
		& $\mathrm{FC}_{S}(256 ,512, \mathrm{ReLU})$ 
		& $\mathrm{FC}_{V}(500,1, -)$ 
		& $\mathrm{MHA}_{S}(256,4)$ \\
		
		- 
		& $\mathrm{FC}_{S}(512 ,1, -)$ 
		& $\mathrm{FC}_{S}(256,500, \mathrm{ReLU})$ 
		& $\mathrm{FC}_{V}(256,512, \mathrm{ReLU})$ \\
		
		- 
		& - 
		& $\mathrm{FC}_{S}(500,1, -)$ 
		& $\mathrm{FC}_{V}(512,1, -)$ \\
		
		- 
		& - 
		& -
		& $\mathrm{FC}_{S}(256 ,512, \mathrm{ReLU})$ \\
		
		- 
		& - 
		& -
		& $\mathrm{FC}_{S}(512 ,1, -)$ \\
		\bottomrule
	\end{tabular}
\end{table}

\begin{figure}[h]
	\centering
	\includegraphics[width=0.8\linewidth]{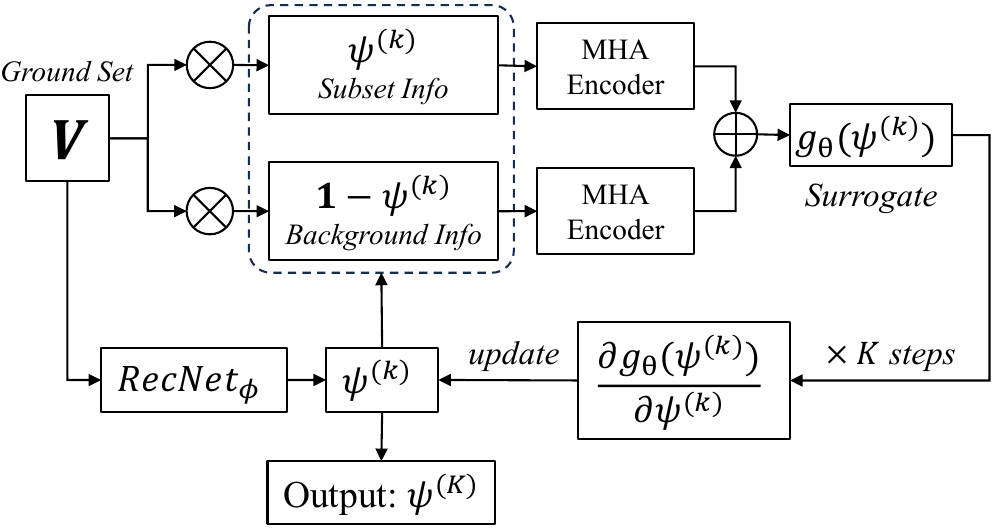}
	\caption{Overall framework of the proposed network architecture.
	}
	\label{framework_fig}
\end{figure}

\section{Detailed Procedure of ReSet}
\subsection{Overall Framework}
\label{appendix:Overall Framework}

Figure~\ref{framework_fig} illustrates the overall architecture of $\operatorname{ReSet}$. Given a ground set $V$, an auxiliary recognition network $\operatorname{RecNet}_{\phi}$ first produces an initial variational parameter $\bm{\psi}^{(0)}$, which amortizes the inference process. At each optimization step, the current soft assignment $\bm{\psi}^{(k)}$ and its complement $1-\bm{\psi}^{(k)}$ are used to construct subset-aware and background-aware representations of the input set. These two branches are encoded by multi-head attention modules to capture element-wise dependencies, and are then combined to produce the surrogate $g_{\theta}(\bm{\psi}^{(k)})$. 
The surrogate gradient is combined with the analytic entropy gradient to iteratively update the variational parameter for $K$ steps, yielding the final prediction $\bm{\psi}^{(K)}$.
This design enables efficient differentiable optimization without Monte Carlo gradient estimation.

\subsection{Detailed Pseudo Code}\label{appendix:Detailed Pseudo Code of ReSet}
We provide the pseudo-code for the training and inference procedures of $\operatorname{ReSet}$ in Algorithms~\ref{algorithm of training} and~\ref{algorithm of Inference}, respectively. The training process consists of three key steps: i) generating the initial variational parameters $\boldsymbol{\psi}^{(0)}$ via $\operatorname{RecNet}_\phi$;
ii) iteratively updating $\boldsymbol{\psi}^{(k)}$ for $K$ steps using the entropy-regularized surrogate update in Eq.~(\ref{add noise});
and iii) computing the cross-entropy loss based on the optimized $\boldsymbol{\psi}^{(K)}$ in Eq.~(\ref{cross-entropy loss}) and updating the network parameters through backpropagation.
During inference, only steps i) and ii) are executed. Notably, we incorporate Langevin dynamics exclusively during training to facilitate exploration, while opting for a deterministic approach during inference to ensure stability and finer convergence. 

\begin{algorithm}[!h]
	\caption{ReSet (Training)}
	\label{algorithm of training}
	
	\begin{algorithmic}[1] 
		\REQUIRE $\{V_i, S^*_i\}_{i=1}^N$: training dataset; $\eta$: learning rate; $K$: number of iterations; $\alpha$: step size for gradient ascent
		\ENSURE Optimal parameters $\theta^*, \phi^*$
		
		\STATE $\theta, \phi\leftarrow$ Initialize parameters
		\WHILE{$\theta, \phi$ not converged}
		\STATE Sample training data point\\ $(V, S^*) \sim \{V_i, S^*_i\}_{i=1}^N$
		\STATE Obtain the initial variational parameter $\bm{\psi}^{(0)}$ via the auxiliary recognition network\\
		$\bm{\psi}^{(0)} \leftarrow \text{RecNet}_\phi(V)$
		
		\FOR{$k=0$ to $K-1$}
		\STATE Calculate the surrogate $g_\theta(\bm{\psi}^{(k)})$ using Eq.~(\ref{Network Architecture})
		\STATE Update the parameter $\bm{\psi}$ by gradient ascent in Eq.~(\ref{add noise})\\
		$\bm{\psi}^{(k+1)} \leftarrow \bm{\psi}^{(k)} + \alpha(
		\nabla g_\theta(\bm{\psi}^{(k)})
		+\log\frac{\bm{1}-\bm{\psi}^{(k)}}{\bm{\psi}^{(k)}}
		) + \bm{\epsilon}_k$
		\ENDFOR
		\STATE Set $\bm{\psi}^{*} \leftarrow \bm{\psi}^{(K)}$
		\STATE Update $\theta$ and $\phi$ by minimizing the cross-entropy loss in Eq.~(\ref{cross-entropy loss})\\
		$(\theta, \phi) \leftarrow (\theta, \phi) - \eta \nabla_{(\theta, \phi)} 
		\Big(
		-\sum_{i \in S^{*}} \log \psi_{i}^{*} - \sum_{i \in V\backslash S^{*}} \log (1 - \psi_{i}^{*})
		\Big)
		$
		\ENDWHILE
		\STATE Set $\theta^*, \phi^* \leftarrow \theta, \phi$  
		\RETURN $\theta^*, \phi^*$
	\end{algorithmic}
\end{algorithm}

\begin{algorithm}[h]
	\caption{ReSet (Inference)}
	\label{algorithm of Inference}
	\begin{algorithmic}[1]
		\REQUIRE $V$: test ground set; $\theta^*, \phi^*$: trained parameters; $K$: number of iterations; $\alpha$: step size for gradient ascent 
		\ENSURE $S^*$: Predicted optimal subset
		
		\STATE Obtain the initial variational parameter via the recognition network
		\STATE $\bm{\psi}^{(0)} \leftarrow \text{RecNet}_{\phi^*}(V)$
		
		\FOR{$k=0$ to $K-1$}
		\STATE Calculate the surrogate $g_{\theta^*}(\bm{\psi}^{(k)})$ using Eq.~(\ref{Network Architecture})
		\STATE Update $\bm{\psi}$ by gradient ascent:
		\STATE $\bm{\psi}^{(k+1)} \leftarrow \bm{\psi}^{(k)} + \alpha(
		\nabla g_{\theta^*}(\bm{\psi}^{(k)})
		+\log\frac{\bm{1}-\bm{\psi}^{(k)}}{\bm{\psi}^{(k)}}
		)$ 
		\ENDFOR
		
		\STATE Set $\bm{\psi}^* \leftarrow \bm{\psi}^{(K)}$
		\STATE Discretize the continuous solution
		\STATE $S^* \leftarrow \mathrm{Rounding}(\boldsymbol{\psi}^*)$
		\RETURN $S^*$
	\end{algorithmic}
\end{algorithm}

\section{Experimental Details}\label{appendix:Experimental Details}
\subsection{The Architecture of ReSet}\label{appendix:The Architecture of ReSet}
Similar to the framework of \citet{ou2022learning, xie2024enhancing}, $\operatorname{ReSet}$ consists of two neural components:
(i) a set function that is permutation invariant, and
(ii) a recognition network that is permutation equivariant.
Both components are implemented using the DeepSets architecture, with detailed architectural specifications provided in Table~\ref{tab: architectures of ReSet}.
\begin{table}[!t]
	\centering
	\caption{Detailed architectures of $\operatorname{ReSet}$.}
	\begin{tabular}{cc}
		\toprule
		\textbf{Set (Surrogate) Function}
		& \textbf{Recognition Network}\\
		\cmidrule(lr){1-1} \cmidrule(lr){2-2}
		$\mathrm{InitLayer}(V,256)$
		& $\mathrm{InitLayer}(V,256)$\\
		$\mathrm{MHA}(256,4)$
		& $\mathrm{FC}(256,512,\mathrm{ReLU})$\\
		$\mathrm{LayerNorm}(256)$
		& $\mathrm{FC}(512,1,\mathrm{Sigmoid})$\\
		$\mathrm{SumPooling}$ & -\\
		$\mathrm{FC}(256,512,\mathrm{ReLU})$ & -\\
		$\mathrm{FC}(512,1,-)$ & -\\
		\bottomrule
	\end{tabular}
	\label{tab: architectures of ReSet}
\end{table}
Specifically, $\mathrm{InitLayer}(d_{in},d_{out})$ encodes the set objects into vector representations, $\mathrm{FC}(d_{in},d_{out}, f)$ is a fully-connected layer with activation function $f$, $\mathrm{MHA}(d, h)$ denotes a multi-head attention layer with $h$ heads.

\textbf{Synthetic datasets.}
We consider two synthetic benchmarks, namely Two-Moons and Gaussian Mixture.
Each set element is represented by a two-dimensional vector encoding its Cartesian coordinates.
For both datasets, the $\mathrm{InitLayer}$ is instantiated as a single-layer feed-forward network
with architecture $\mathrm{FC}(2, 256, -)$.

\textbf{Amazon Baby Registry.}
The Amazon Baby Registry dataset consists of sets of products, where each product is associated
with a short textual description.
We encode these descriptions using a pre-trained BERT model \citep{lee2018pre},
resulting in a $768$-dimensional embedding for each element.
The corresponding $\mathrm{InitLayer}$ is implemented as $\mathrm{FC}(768, 256, -)$.

\textbf{Double MNIST.}
The Double MNIST dataset contains digit images spanning all combinations from $00$ to $99$.
Each image has a resolution of $(64,64)$ and is reshaped into a vector of dimension $4096$.
Accordingly, we employ a fully connected initialization layer with architecture
$\mathrm{FC}(4096, 256, -)$.

\textbf{F-MNIST.}
The Fashion-MNIST dataset consists of grayscale images of fashion items from $10$ categories.
Each image has a spatial resolution of $(28,28)$ and is flattened into a $784$-dimensional vector.
Similar to Double MNIST, the $\mathrm{InitLayer}$ is implemented as a feed-forward neural network with architecture
$\mathrm{FC}(784, 256, -)$.

\textbf{CelebA.}
The CelebA dataset comprises $202{,}599$ face images, each of size $(3,64,64)$.
For this dataset, the $\mathrm{InitLayer}$ is realized using a convolutional neural network.
Specifically, its architecture is given by
\begin{align}
	\operatorname{ModuleList}([\mathrm{Conv}(32,3,2,\mathrm{ReLU}),\ \mathrm{Conv}(64,4,2,\mathrm{ReLU}),\ \nonumber\\
	\mathrm{Conv}(128,5,2,\mathrm{ReLU}),\ \operatorname{MaxPooling},\ \mathrm{FC}(128, 256, -)]),\nonumber
\end{align}
where $\mathrm{Conv}(d,k,s,f)$ denotes a convolutional layer parameterized by $d$ output channels,
kernel size $k$, stride $s$, and activation function $f$.

\textbf{CIFAR-10.}
The CIFAR-10 dataset consists of $60,000$ color images of size $32 \times 32$ pixels. For this dataset, the feature extraction network is implemented as a deep convolutional neural network consisting of three cascaded functional blocks. Specifically, the architecture is formulated as:
\begin{equation}
	\operatorname{ModuleList}\left(\left[ \operatorname{Block}(32), \operatorname{Block}(64), \operatorname{Block}(128), \operatorname{AvgPool}, \mathrm{FC}(128, 256, -) \right]\right), \nonumber
\end{equation}
where each $\operatorname{Block}(d)$ is a sub-network designed to capture hierarchical spatial features, defined as the sequence:
\begin{equation}
	[\mathrm{Conv}(d, 3, 2, \mathrm{ReLU}), \mathrm{BN}, \mathrm{Conv}(d, 3, 1, \mathrm{ReLU}), \mathrm{BN}]. \nonumber
\end{equation}
To maintain spatial resolution during convolution, zero-padding is applied such that $p = \lfloor k/2 \rfloor$. Each $\mathrm{BN}$ denotes a Batch Normalization layer applied after the convolution. 
It should be noted that the code related to this part was not publicly released in \citet{ou2022learning}, so we constructed the $\mathrm{InitLayer}$ architecture ourselves.

\textbf{PDBBind.}
The PDBBind database provides experimentally determined binding affinities for biomolecular complexes \citep{liu2015pdb},
together with high-resolution three-dimensional Cartesian coordinates of ligands and their corresponding target proteins,
typically obtained via techniques such as X-ray crystallography.
To extract informative representations, we leverage the atomic convolutional neural network (ACNN) \citep{gomes2017atomic},
which constructs nearest-neighbor graphs based on atomic spatial coordinates.
In our implementation, the representation is obtained from the penultimate layer of ACNN,
followed by additional feed-forward layers.
Formally, the $\mathrm{InitLayer}$ is defined as
\begin{align}
	\operatorname{ModuleList}([\mathrm{ACNN}[:-1],\ \mathrm{FC}(1922, 2048, \mathrm{ReLU}),\ \mathrm{FC}(2048, 256, -)]), \nonumber
\end{align}
where $\mathrm{ACNN}[:-1]$ denotes the ACNN architecture without its final prediction layer,
yielding an output dimension of $1{,}922$.

\textbf{BindingDB.}
The BindingDB dataset consists of $52{,}273$ drug--target interaction pairs.
We adopt the DeepDTA framework \citep{ozturk2018deepdta} to encode each drug--target pair into a fixed-dimensional vector.
DeepDTA represents drug compounds and target proteins as one-hot encoded sequences,
which are subsequently processed by convolutional neural networks.
The full specification of the $\mathrm{InitLayer}$ used for this dataset
is reported in Table~\ref{tab:bindingdb_initlayer_architecture}.

\begin{table}[ht]
	\centering
	\caption{Detailed architectures of InitLayer in the BindingDB dataset.}
	\begin{tabular}{@{}cc@{}}\toprule
		\textbf{Drug} & \textbf{Target} \\
		\cmidrule(lr){1-1} \cmidrule(lr){2-2}
		$\mathrm{Conv}(32,4,1,\mathrm{ReLU})$ & $\mathrm{Conv}(32,4,1,\mathrm{ReLU})$ \\
		$\mathrm{Conv}(64,6,1,\mathrm{ReLU})$ & $\mathrm{Conv}(64,8,1,\mathrm{ReLU})$ \\
		$\mathrm{Conv}(96,8,1,\mathrm{ReLU})$ & $\mathrm{Conv}(96,12,1,\mathrm{ReLU})$ \\
		$\mathrm{MaxPooling}$ & $\mathrm{MaxPooling}$ \\
		$\mathrm{FC}(96, 256, \mathrm{ReLU})$ & $\mathrm{FC}(96, 256, \mathrm{ReLU})$ \\
		\multicolumn{2}{c}{$\mathrm{Concat}$} \\
		\multicolumn{2}{c}{$\mathrm{FC}(512, 256, -)$} \\
		\bottomrule
	\end{tabular}
	\label{tab:bindingdb_initlayer_architecture}
\end{table}

\subsection{Implementation Details}\label{appendix:Implementation Details}
We provide the training details and hyperparameter settings for $\operatorname{ReSet}$.
The proposed model is trained using the Adam optimizer \citep{kingma2017adammethodstochasticoptimization} with a fixed weight decay of $1e-5$ and a batch size of $128$.
During training, we employ a cosine annealing learning rate scheduler \citep{loshchilov2016sgdr} together with an early stopping strategy. 
The initial learning rate is selected via grid search over the range $[1e-5, 1e-3]$ using half-decade intervals, with the minimum learning rate fixed at one-tenth of the chosen initial value. Training is conducted for up to 500 epochs, and early stopping is triggered if no performance improvement is observed for 20 consecutive epochs. For datasets that are more difficult to converge, we allow a longer patience.
After training, the saved models are evaluated on the test set. All experiments are repeated 5 times with different random seeds, and the mean performance along with their standard deviations are reported.

Within the $\operatorname{ReSet}$ framework, the gradient ascent step size $\alpha$ and the noise magnitude $\epsilon$ in the Langevin dynamics constitute two critical hyperparameters governing model performance. To enhance the adaptivity of the learning process, we allow both parameters to be updated during training. Specifically, the initial value of $\alpha$ is selected empirically. In contrast, although $\epsilon$ is treated as a learnable parameter, it is constrained to remain at a small scale of $1e-5$. This strategy aims to ensure numerical stability during training while enabling precise control over the determinacy of the model.

\subsection{Baselines}\label{Baselines}
Throughout our experiments, we evaluate the proposed models against four conventional approaches: Random Guess, Probabilistic Greedy Model (PGM) \citep{tschiatschek2018differentiable}, DeepSet \citep{zaheer2017deep}, and Set Transformer \citep{lee2019set}. Additionally, we compare our method with two state-of-the-art optimal subset oracles: EquiVSet \citep{ou2022learning} and INSET \citep{xie2024enhancing}. Detailed descriptions of these benchmarks are provided below.
\begin{itemize}
	\item \textbf{Random.} We report the expected Jaccard coefficient obtained by random guessing. This baseline serves as a reference for assessing the intrinsic difficulty of the task.
	
	\item \textbf{PGM} \citep{tschiatschek2018differentiable}. The probabilistic greedy model tackles Problem~(\ref{OS oracle}) by introducing a differentiable relaxation of the greedy maximization procedure, which relies on enumerating all possible element permutations. For a detailed exposition, we refer readers to Appendix~A of \citet{ou2022learning}.
	
	\item \textbf{DeepSet} \citep{zaheer2017deep}. By learning a permutation-invariant mapping from $2^{|V|}$ to $[0,1]^{|V|}$, the DeepSet architecture provides a generic backbone for modeling set functions in optimal subset oracle frameworks \citep{ou2022learning, xie2024enhancing, xie2024horse}. We therefore include it as a baseline in our evaluation.
	To adapt DeepSets for point-wise tasks, we shift from an invariant to a permutation equivariant design. Instead of global reduction, we broadcast the mean-pooled global descriptor to each element, allowing for an output of size $|V|$ that preserves the identity of individual members.
	
	\item \textbf{Set Transformer} \citep{lee2019set}. Extending the DeepSet framework, Set Transformer incorporates self-attention mechanisms to model pairwise interactions among set elements. Similar to DeepSet, it can be adapted for subset selection tasks and is included as a competitive baseline.
	Similar to our adaptation of DeepSet, we bypass the original global pooling (PMA) and invariant decoder, utilizing only the encoder blocks. This preserves the permutation equivariance of the self-attention mechanism, allowing for element-wise prediction while maintaining an output dimensionality of $|V|$.
	
	\item \textbf{EquiVSet} \citep{ou2022learning}. EquiVSet aims to learn an optimal subset oracle without explicitly specifying the underlying set function. Instead, the set function is modeled as an energy-based formulation and approximated through a variational distribution. Optimization via mean-field variational inference requires iterative unrolling with Monte Carlo sampling to estimate gradients of the multilinear extension, which introduces substantial computational overhead.
	
	\item \textbf{INSET} \citep{xie2024enhancing}. Building upon EquiVSet, INSET augments the set function $F_\theta(\cdot)$ by injecting background information from the ground set $V$. By explicitly accounting for the symmetry of the ground set, this design introduces a stronger inductive bias and more effectively exploits contextual relationships within the global set.
\end{itemize}

\subsection{Detailed Experimental Settings for Product Recommendation}\label{appendix:Detailed Experimental Settings for Product Recommendation}
The Amazon Baby Registry dataset \citep{gillenwater2014expectation} is collected from Amazon and organized into multiple subsets based on product categories, such as toys and furniture.
Within each category, which serves as a product universe, Amazon records multiple product subsets curated by different customers.
These customer-selected subsets naturally correspond to optimal subset (OS) oracles.
To ensure that each ground set $V$ is associated with a single OS oracle $S^*$, we construct data samples in the form of $(V, S^*)$ as follows.
We first discard customer-selected subsets whose cardinality is equal to $1$ or exceeds $30$.
For each remaining OS oracle $S^*$, we randomly sample $30 - |S^*|$ additional products from the same category to form the complement set $V \setminus S^*$.
Statistics of the resulting datasets across different categories are reported in Table~\ref{tab:statistic_amazon_product}.

\begin{table}[t]
	\centering 
	\small
	\caption{The statistics of Amazon product dataset. }
	\label{tab:statistic_amazon_product}
	\begin{tabular}{@{}cccccccc@{}}\toprule
		Categories & \#products & $|\mathcal{D}|$ & $|V|$ & $\sum |S^*|$ & $\mathbb{E}[|S^*|]$ & $\min_{S^*} |S^*|$ & $\max_{S^*} |S^*|$\\
		\midrule
		Toys & 62 & 2,421  & 30 & 9,924 & 4.09 & 3 & 14 \\
		Furniture & 32 & 280  & 30 & 892 & 3.18 & 3 & 6 \\
		Gear & 100 & 4,277  & 30 & 16,288 & 3.80 & 3 & 10 \\
		Carseats & 34 & 483  & 30 & 1,576 & 3.26 & 3 & 6 \\
		Bath & 100 & 3,195  & 30 & 12,147 & 3.80 & 3 & 11 \\
		Health & 62 & 2,995  & 30 & 11,053 & 3.69 & 3 & 9 \\
		Diaper & 100 & 6,108  & 30 & 25,333 & 4.14 & 3 & 15 \\
		Bedding & 100 & 4,524  & 30 & 17,509 & 3.87 & 3 & 12 \\
		Safety & 36 & 267  & 30 & 846 & 3.16 & 3 & 5 \\
		Feeding & 100 & 8,202  & 30 & 37,901 & 4.62 & 3 & 23 \\
		Apparel & 100 & 4,675  & 30 & 21,176 & 4.52 & 3 & 21 \\
		Media & 58 & 1,485  & 30 & 6,723 & 4.52 & 3 & 19 \\
		\bottomrule
	\end{tabular}
\end{table}

\subsection{Detailed Experimental Settings for Set Anomaly Detection}\label{Experimental Settings for Set Anomaly Detection}
\textbf{CelebA:} The CelebA dataset consists of $202{,}599$ images annotated with $40$ binary attributes. As illustrated in Fig.~\ref{fig:celeba}, we randomly choose two attributes and form sets of fixed cardinality $8$. For each ground set $V$, an OS oracle subset $S^{*}$ is constructed by randomly sampling $n \in \{2,3\}$ images such that neither of the selected attributes is exhibited. Based on this procedure, we generate training, validation, and test splits containing $10{,}000$, $1{,}000$, and $1{,}000$ samples, respectively.

\begin{figure}[h]
	\centering
	\includegraphics[width=0.9\linewidth]{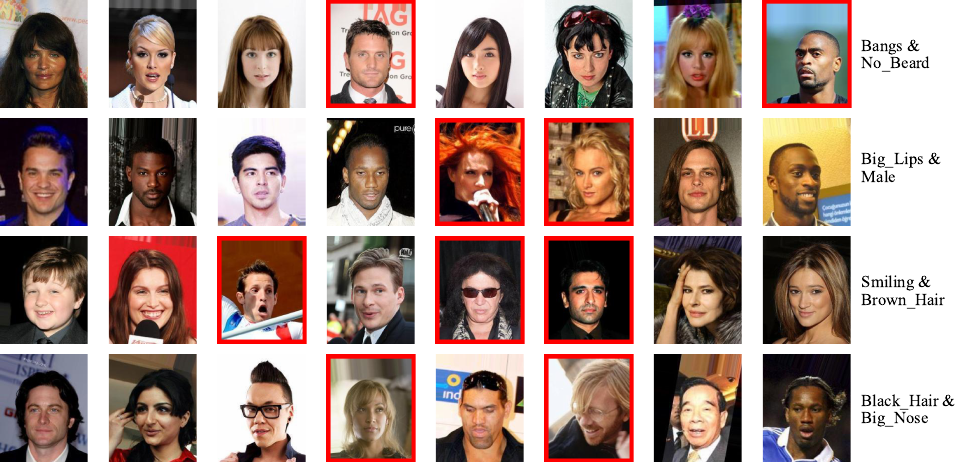}
	\caption{Example data points from the CelebA dataset adapted from \citet{ou2022learning}. Each row corresponds to one sample and contains $|S^*|$ anomalous images (highlighted in red) along with $8 - |S^|$ normal images. Within each sample, normal images exhibit both selected attributes (shown in the rightmost column), whereas anomalous images exhibit neither attribute.}
	\label{fig:celeba}
\end{figure}

\textbf{Comparing with the Setting of \citet{ou2022learning}.}
Although we aim to align our experimental setup as closely as possible with that of \cite{ou2022learning}, the experimental protocols for the anomaly detection task on Double MNIST \citep{sun2019multi}, F-MNIST \citep{xiao2017fashion}, and CIFAR-10 \citep{krizhevsky2009learning} are not publicly available. Moreover, the corresponding descriptions in the original work are relatively ambiguous. For clarity and reproducibility, we therefore re-establish and standardize the experimental settings for these three datasets.
Specifically, for Double MNIST, we generate a total of $1,000$ images distributed across all digit pairs from $00$ to $99$.\footnote{The source code for dataset generation is publicly available at \url{https://github.com/shaohua0116/MultiDigitMNIST}.} 
For F-MNIST and CIFAR-10, we use the official datasets obtained via online sources. 
During sample construction, we randomly select $n \in \{2,3, 4\}$ images from the same class to form the optimal subset oracle $S^*$ for all three datasets.
For the Double MNIST dataset, $V \backslash S^*$ is constructed by selecting $20 - |S^*|$ images, whereas for F-MNIST and CIFAR-10, it is formed by drawing $10 - |S^*|$ images from different classes. Notably, the images in $V \backslash S^*$ are drawn from distinct classes, with at most one image per class.
Finally, we generate samples on the fly, producing $10,000$, $1,000$, and $1,000$ instances for the training, validation, and test sets, respectively. Visual illustrations are provided in Figures~\ref{fig:double_mnist} and~\ref{fig:fmnist-cifar10}.

\begin{figure}[h]
	\centering
	\includegraphics[width=0.82\linewidth]{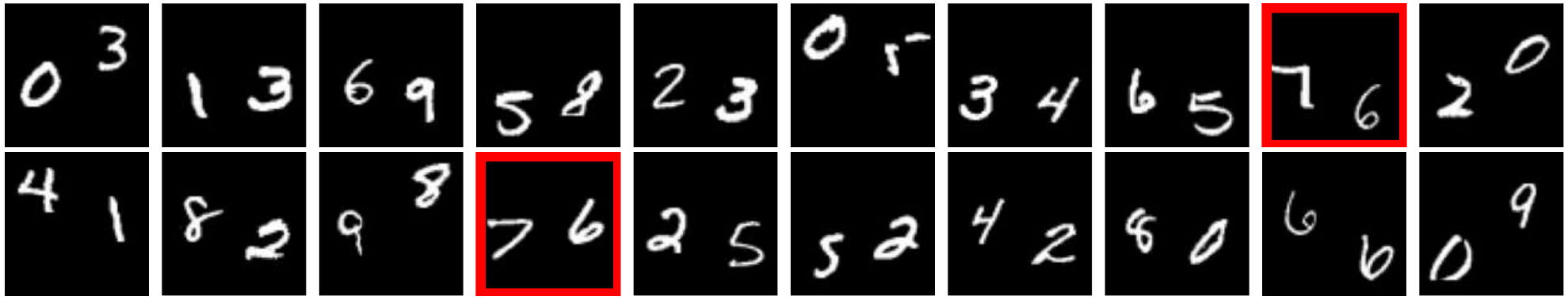}
	\caption{A sampled data for the Double MNIST dataset, which consists of $|S^*|$ images with the same digit (red box, $76$ in this case) and $20-|S^*|$ images with different digits.
	}
	\label{fig:double_mnist}
\end{figure}

\begin{figure}[h]
	\begin{minipage}[!b]{0.5\linewidth}
		\centering
		\includegraphics[width=2.7in]{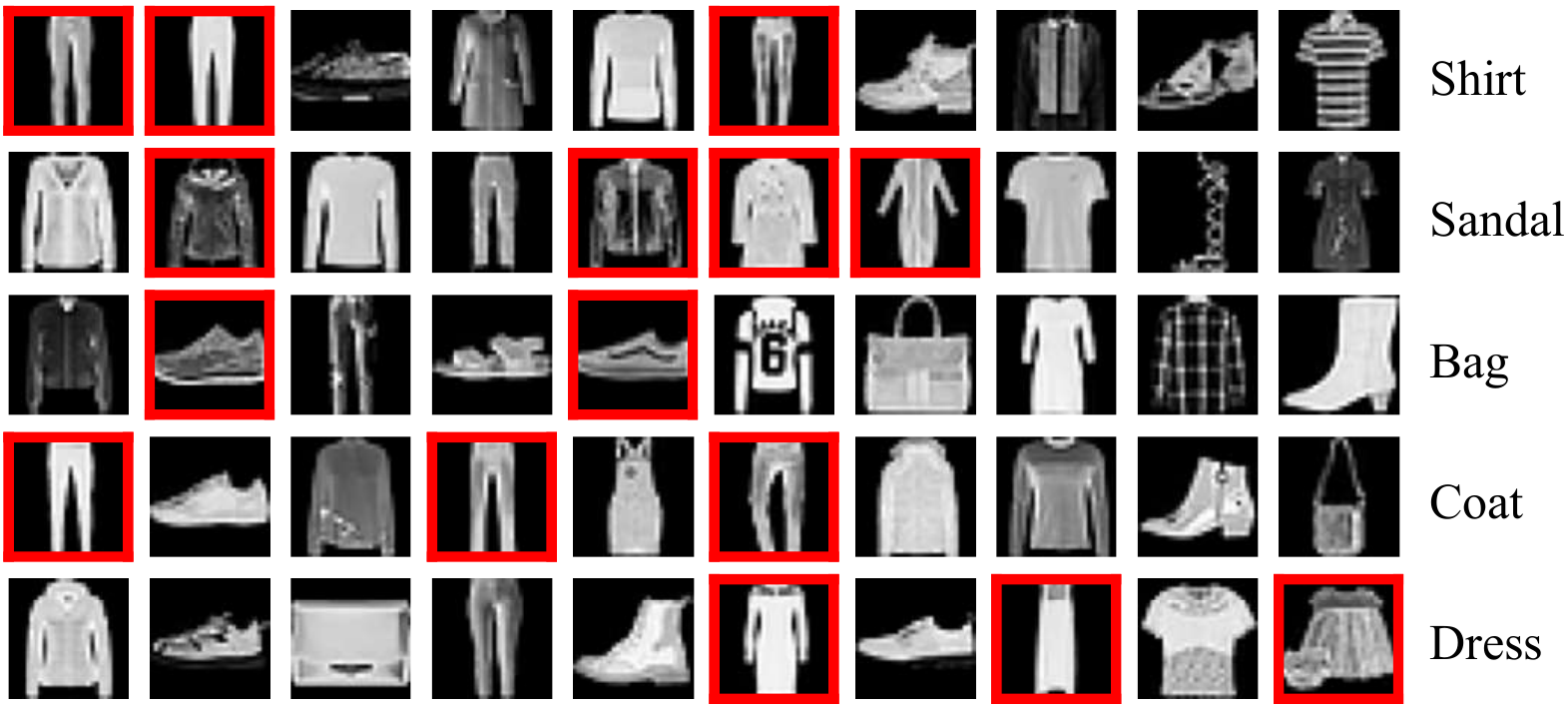}
	\end{minipage}
	\begin{minipage}[!b]{0.5\linewidth}
		\centering
		\includegraphics[width=2.7in]{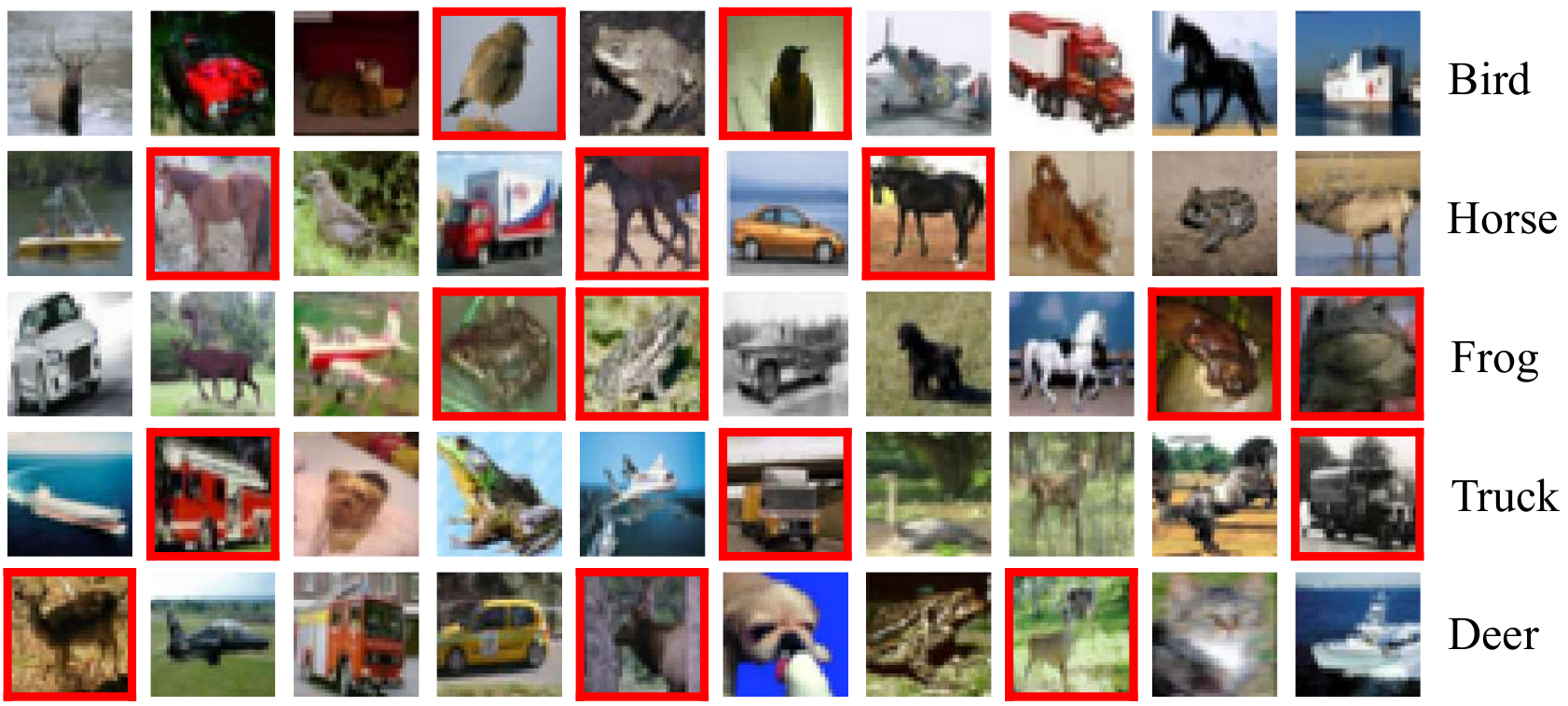}
	\end{minipage}
	\caption{Illustrative samples from the F-MNIST (left) and CIFAR-10 (right) datasets. Each row represents a single sample composed of $|S^*|$ images sharing the same class label (highlighted in red) and $10 - |S^*|$ additional images drawn from distinct classes. The corresponding labels are shown in the rightmost column.}
	\label{fig:fmnist-cifar10}
\end{figure}

\subsection{Detailed Experimental Settings for Compound Selection}
\label{appendix_compount_selection}
In the drug discovery pipeline, identifying compounds that exhibit strong bioactivity~\citep{wallach2015atomnet, li2021structure, ji2023drugood}, sufficient structural diversity, and favorable ADME (absorption, distribution, metabolism, and excretion) profiles~\citep{gimeno2019light} constitutes a fundamental screening stage.
In practice, virtual screening is commonly implemented as a sequence of filtering operations. A typical workflow first prioritizes candidates based on predicted bioactivity, subsequently enforces diversity constraints, and finally excludes compounds with undesirable ADME characteristics, yielding a refined candidate subset.
However, owing to privacy considerations and cost constraints in pharmaceutical development, intermediate supervision signals corresponding to these individual filtering stages are often inaccessible or prohibitively expensive to obtain. Consequently, learning the entire screening procedure in an end-to-end fashion becomes desirable, a setting that naturally aligns with the formulation of optimal subset (OS) oracles.
Despite this alignment, fully end-to-end modeling remains challenging and often demands substantial domain knowledge. In this work, we therefore adopt a simplified yet representative setting by focusing on the prediction of high-bioactivity compounds, following the experimental protocol of \citet{xie2024enhancing}. Experiments are conducted on the PDBBind~\citep{liu2015pdb} and BindingDB~\citep{liu2007bindingdb} datasets using a unified bioactivity-based filtering criterion.
Algorithm~\ref{alg:compound_selection} illustrates the data generation procedure used to simulate the OS oracle for the compound selection task.
Specifically, $\operatorname{random\_choose}(\mathcal{C}, n)$ denotes uniformly sampling $n$ compounds from a database $\mathcal{C}$ (either PDBBind or BindingDB) to construct the ground set $V$,
while $\operatorname{topK\_bioactivity}(V, m)$ selects the top-$m$ compounds exhibiting the highest biological activity from $V$.
These two operations together form a bioactivity-based filtering process, through which the OS oracle $S^*$ is obtained.
In our experiments, we set $(n, m) = (30, 5)$ for PDBBind and $(300, 15)$ for BindingDB, respectively.
This procedure yields a data point $(V, S^*)$, where $V$ represents a candidate compound set and $S^*$ corresponds to the subset of compounds with the highest bioactivity.

\begin{algorithm}[h]
	\caption{OS Oracle Construction Procedure for the Compound Selection Task}
	\label{alg:compound_selection}
	\begin{algorithmic}[1]
		\renewcommand{\algorithmicrequire}{\textbf{Input:}}
		\renewcommand{\algorithmicensure}{\textbf{Output:}}
		\REQUIRE $\mathcal{C}$: compound database; $n$: cardinality of the ground set; $m$: number of highly active compounds
		\ENSURE A data sample $(V, S^*)$
		
		\STATE \textbf{Ground Set Formation:}
		\STATE Uniformly sample $n$ compounds from the database $\mathcal{C}$
		\STATE $V \leftarrow \operatorname{random\_choose}(\mathcal{C}, n)$
		
		\STATE \textbf{Bioactivity-based Selection (OS Oracle):}
		\STATE Select the top-$m$ compounds with the highest bioactivity scores from $V$
		\STATE $S^* \leftarrow \operatorname{topK\_bioactivity}(V, m)$ \hfill 
		
		\RETURN $(V, S^*)$
	\end{algorithmic}
\end{algorithm}

\section{Additional Experiments}\label{appendix:Additional Experiments}

\subsection{Gradient Alignment of the Learned Surrogate}
\label{appendix:gradient_alignment}

We evaluate the learned surrogate $g_\theta$ by comparing its
gradient directions with those of its induced multilinear
extension. Small synthetic sets allow us to compute the reference
gradient exactly and directly assess the surrogate's role in
guiding variational optimization.

\paragraph{Experimental setup.}
We construct synthetic sets of $d\in\{8,10,12\}$ two-dimensional
elements by perturbing three anchor points with Gaussian noise
and randomly permuting the resulting elements.
Optimal-subset labels are obtained by exhaustive search under
a coverage utility with a cardinality penalty, encouraging
compact subsets that cover the three anchor regions.
Each model is trained across these set sizes using one unrolled
inference step and only optimal-subset supervision, without
explicit function-value or gradient-matching losses.

\paragraph{Gradient comparison.}
For each architecture, we freeze a single trained model and
evaluate it on one held-out input set for each cardinality.
For each fixed $V$, define $F_\theta(S):=g_\theta(\mathbf{1}_S)$
and write $f:=f_{\mathrm{mt}}^{F_\theta}$.
We enumerate all $2^d$ subsets to compute $f$ and its gradient
exactly, without Monte Carlo approximation.
For each input, we sample $64$ probability vectors uniformly from
$(0,1)^d$ and measure
\begin{equation}
	C(\bm{\psi})=
	\frac{\left\langle
		\nabla_{\bm{\psi}}g_\theta(\bm{\psi}),
		\nabla_{\bm{\psi}}f(\bm{\psi})\right\rangle}
	{\|\nabla_{\bm{\psi}}g_\theta(\bm{\psi})\|_2
		\|\nabla_{\bm{\psi}}f(\bm{\psi})\|_2}.
	\label{eq:surrogate_gradient_cosine}
\end{equation}
Both models use the same input sets and probability vectors.
The comparison excludes the entropy term to isolate the
alignment between the two objectives.

\paragraph{Results.}
Figure~\ref{fig:gradient_alignment} shows positive gradient
alignment at all sampled points.
For $d=8,10,12$, the median cosine similarities are
$0.952$, $0.878$, and $0.845$ for EquiVSet-$\operatorname{R}$,
and $0.982$, $0.944$, and $0.973$ for INSET-$\operatorname{R}$.
INSET-$\operatorname{R}$ also exhibits less variation across
probability vectors, with standard deviations of
$0.021$--$0.035$, compared with $0.064$--$0.127$ for
EquiVSet-$\operatorname{R}$.
These results support the use of $g_\theta$ to provide useful
local gradient directions for the multilinear objective on these
synthetic inputs, even without explicit gradient supervision.

\begin{figure}[htbp]
	\centering
	\includegraphics[width=\linewidth]{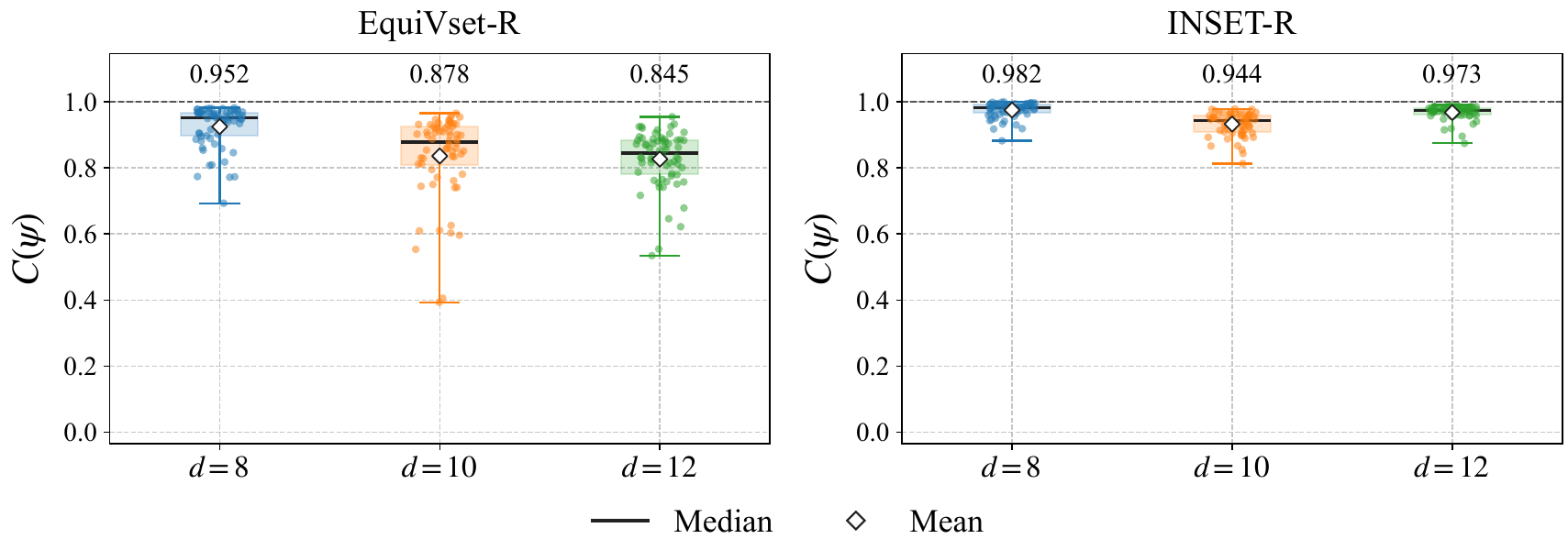}
	\caption{Gradient cosine similarity between the learned surrogate
		$g_\theta$ and its exact induced multilinear extension \(f_{\mathrm{mt}}^{F_\theta}\).
		Each group uses $64$ uniformly sampled probability vectors for
		one held-out set of size $d$, with one trained model per architecture.
		Dots denote individual samples; boxes and whiskers show the
		interquartile and full ranges, respectively.
		Black bars and labels indicate medians, and white diamonds indicate
		means. The dashed line marks perfect alignment.}
	\label{fig:gradient_alignment}
\end{figure}

\subsection{Ablation Study}\label{appendix:Ablation Study}
We examine the effects of multi-head self-attention, Langevin noise, and temperature learning for both EquiVSet and INSET. We augment each baseline with multi-head self-attention, denoted by ``-$\operatorname{MHA}$'', and remove Langevin noise from our $\operatorname{R}$ variants, denoted by ``w/o $\operatorname{L}$''. In the full models, the temperature $T$ is a learnable parameter optimized jointly with the network weights. We additionally evaluate $\operatorname{R}_{(T=1)}$, which fixes $T=1$ throughout training. Table~\ref{table:ablation} reports the results.
Both full $\operatorname{R}$ variants achieve higher mean scores than their original and $\operatorname{MHA}$-augmented baselines on all six datasets. Adding self-attention alone therefore does not reproduce the gains of our framework. Removing Langevin noise yields mixed results: the deterministic variants obtain the highest mean scores on Toys in both model families and on Furniture for INSET, whereas the full variants perform better on Carseats and Diaper. These observations indicate that the benefit of Langevin noise is dataset-dependent, with deterministic updates providing a competitive alternative in several cases.
The variants with learnable temperature achieve higher mean scores than their fixed-temperature counterparts on four of the six datasets in each model family. On Toys, EquiVSet-$\operatorname{R}$ with learnable temperature also exhibits lower variability, with a standard deviation of $1.0$ compared with $11.3$ when $T=1$. Nevertheless, fixing $T=1$ remains competitive and yields the highest mean scores on Media in both families and on Safety for EquiVSet. These results suggest that learning the temperature can improve performance and stability, although its benefits vary across datasets.

\begin{table*}[t]
	\caption{Ablation results on six Amazon datasets. Values are test Jaccard scores (\%), reported as mean $\pm$ standard deviation. Bold indicates the highest mean within each model family.}
	\label{table:ablation}
	\centering
	\resizebox{0.9\textwidth}{!}{
		\begin{tabular}{l|cccccc}
			\toprule
			Variants & Toys & Furniture & Carseats & Diaper & Safety & Media \\
			\midrule
			EquiVSet
			& 68.0 $\pm$ 2.0
			& 17.2 $\pm$ 0.9
			& 21.0 $\pm$ 1.0
			& 83.0 $\pm$ 1.0
			& 25.0 $\pm$ 3.0
			& 57.0 $\pm$ 1.0 \\
			EquiVSet-$\operatorname{MHA}$
			& 75.7 $\pm$ 1.4
			& 15.1 $\pm$ 1.8
			& 20.6 $\pm$ 3.2
			& 88.5 $\pm$ 0.7
			& 13.9 $\pm$ 0.6
			& 62.9 $\pm$ 2.8 \\
			\midrule
			EquiVSet-$\operatorname{R}_{(T=1)}$
			& 73.1 $\pm$ 11.3
			& 35.3 $\pm$ 1.3
			& 23.5 $\pm$ 2.3
			& 89.1 $\pm$ 0.9
			& \textbf{37.9 $\pm$ 5.4}
			& \textbf{67.2 $\pm$ 1.1} \\
			EquiVSet-$\operatorname{R}_{(\text{w/o L})}$
			& \textbf{78.7 $\pm$ 0.7}
			& 33.8 $\pm$ 4.8
			& 26.5 $\pm$ 2.4
			& 89.8 $\pm$ 1.0
			& 37.6 $\pm$ 5.3
			& 64.6 $\pm$ 4.0 \\
			EquiVSet-$\operatorname{R}$
			& 77.9 $\pm$ 1.0
			& \textbf{35.9 $\pm$ 1.3}
			& \textbf{26.9 $\pm$ 1.0}
			& \textbf{90.4 $\pm$ 0.6}
			& 32.1 $\pm$ 9.4
			& 67.0 $\pm$ 1.0 \\
			\midrule
			INSET
			& 76.9 $\pm$ 0.5
			& 16.9 $\pm$ 5.0
			& 23.1 $\pm$ 3.4
			& 88.0 $\pm$ 0.7
			& 23.8 $\pm$ 1.5
			& 62.0 $\pm$ 2.3 \\
			INSET-$\operatorname{MHA}$
			& 77.8 $\pm$ 1.3
			& 17.4 $\pm$ 1.5
			& 23.5 $\pm$ 2.0
			& 90.2 $\pm$ 1.7
			& 14.0 $\pm$ 0.5
			& 65.5 $\pm$ 2.4 \\
			\midrule
			INSET-$\operatorname{R}_{(T=1)}$
			& 79.1 $\pm$ 1.4
			& 34.7 $\pm$ 4.5
			& 28.2 $\pm$ 3.7
			& 90.0 $\pm$ 0.7
			& 35.3 $\pm$ 6.4
			& \textbf{68.7 $\pm$ 1.1} \\
			INSET-$\operatorname{R}_{(\text{w/o L})}$
			& \textbf{79.3 $\pm$ 0.7}
			& \textbf{36.9 $\pm$ 1.0}
			& 29.7 $\pm$ 2.8
			& 89.9 $\pm$ 0.5
			& 37.4 $\pm$ 5.0
			& 67.8 $\pm$ 1.3 \\
			INSET-$\operatorname{R}$
			& 79.0 $\pm$ 0.6
			& 36.5 $\pm$ 1.5
			& \textbf{30.2 $\pm$ 3.9}
			& \textbf{90.4 $\pm$ 0.5}
			& \textbf{38.3 $\pm$ 4.0}
			& 68.1 $\pm$ 1.1 \\
			\bottomrule
		\end{tabular}
	}
\end{table*}

\subsection{Comparison with iDiffMF}
\label{appendix:Comparison with iDiffMF}
Another closely related work is iDiffMF~\citep{ozcan2025learning}.
iDiffMF uses implicit differentiation through the mean-field fixed
point to reduce the cost of backpropagation, while retaining Monte
Carlo estimates of the multilinear gradient.
Our approach learns a differentiable surrogate $g_\theta$ whose
gradients directly guide variational optimization, replacing
these sampling-based gradient estimates.
Table~\ref{tab:idiffmf_comparison} compares performance on the 16
datasets covered in their original paper.
We use the DiffMF, $\mathrm{iDiffMF}_{2}$, and
$\mathrm{iDiffMF}_{*}$ results reported in
the original paper, together with our results.
The entries follow the experimental protocols and reporting
conventions of their respective studies.
Both EquiVSet-$\operatorname{R}$ and INSET-$\operatorname{R}$
achieve higher reported mean MJC than all three baselines on
15 of the 16 datasets, with comparable results on Gaussian Mixture.
The largest differences are observed on Two Moons and Furniture.

\begin{table}[h]
	\centering
	\caption{Comparison with DiffMF and iDiffMF.
		Test MJC (\%) is reported as mean $\pm$ standard deviation.
		DiffMF and iDiffMF results are quoted from
		\citet{ozcan2025learning}.
		Bold indicates the highest reported mean in each row.}
	\label{tab:idiffmf_comparison}
	\begingroup
	\small
	\setlength{\tabcolsep}{6pt}
	\renewcommand{\arraystretch}{1.2}
	\resizebox{\linewidth}{!}{
		\begin{tabular}{l|ccc|cc}
			\toprule
			Dataset
			& DiffMF & $\mathrm{iDiffMF}_{2}$ & $\mathrm{iDiffMF}_{*}$
			& EquiVSet-$\operatorname{R}$ & INSET-$\operatorname{R}$ \\
			\midrule
			CelebA
			& $54.42 \pm 0.70$ & $56.30 \pm 0.58$ & $56.55 \pm 0.49$
			& $58.8 \pm 0.6$ & $\mathbf{59.2 \pm 0.2}$ \\
			Gaussian Mixture
			& $90.96 \pm 0.05$ & $90.95 \pm 0.18$ & $\mathbf{91.03 \pm 0.09}$
			& $91.0 \pm 0.1$ & $90.9 \pm 0.1$ \\
			Two Moons
			& $58.45 \pm 0.15$ & $58.48 \pm 0.15$ & $58.97 \pm 0.04$
			& $99.2 \pm 0.3$ & $\mathbf{99.7 \pm 0.1}$ \\
			\midrule
			Apparel
			& $70.60 \pm 1.35$ & $76.13 \pm 4.65$ & $73.80 \pm 5.71$
			& $82.2 \pm 0.9$ & $\mathbf{83.3 \pm 1.3}$ \\
			Bath
			& $71.87 \pm 0.27$ & $77.68 \pm 0.98$ & $76.43 \pm 0.81$
			& $84.7 \pm 0.9$ & $\mathbf{86.7 \pm 0.6}$ \\
			Bedding
			& $67.66 \pm 0.39$ & $77.88 \pm 0.80$ & $76.94 \pm 1.05$
			& $86.8 \pm 1.2$ & $\mathbf{87.6 \pm 1.7}$ \\
			Carseats
			& $20.15 \pm 0.65$ & $21.94 \pm 1.43$ & $22.42 \pm 1.04$
			& $26.9 \pm 1.0$ & $\mathbf{30.2 \pm 3.9}$ \\
			Diaper
			& $81.74 \pm 1.18$ & $82.76 \pm 0.62$ & $82.07 \pm 0.90$
			& $\mathbf{90.4 \pm 0.6}$ & $\mathbf{90.4 \pm 0.5}$ \\
			Feeding
			& $77.44 \pm 0.46$ & $81.93 \pm 1.00$ & $81.52 \pm 1.84$
			& $87.2 \pm 1.4$ & $\mathbf{88.6 \pm 0.5}$ \\
			Furniture
			& $16.84 \pm 0.05$ & $19.93 \pm 2.68$ & $18.69 \pm 0.93$
			& $35.9 \pm 1.3$ & $\mathbf{36.5 \pm 1.5}$ \\
			Gear
			& $66.06 \pm 2.86$ & $73.90 \pm 10.29$ & $73.57 \pm 6.74$
			& $78.7 \pm 2.6$ & $\mathbf{82.1 \pm 1.1}$ \\
			Health
			& $59.64 \pm 0.81$ & $72.55 \pm 1.10$ & $72.32 \pm 1.03$
			& $81.4 \pm 0.6$ & $\mathbf{82.2 \pm 0.2}$ \\
			Media
			& $51.32 \pm 1.11$ & $56.39 \pm 2.68$ & $55.58 \pm 1.75$
			& $67.0 \pm 1.0$ & $\mathbf{68.1 \pm 1.1}$ \\
			Safety
			& $24.66 \pm 5.56$ & $26.02 \pm 1.68$ & $25.38 \pm 1.88$
			& $32.1 \pm 9.4$ & $\mathbf{38.3 \pm 4.0}$ \\
			Toys
			& $64.39 \pm 1.64$ & $68.53 \pm 1.35$ & $68.91 \pm 1.00$
			& $77.9 \pm 1.0$ & $\mathbf{79.0 \pm 0.6}$ \\
			\midrule
			BindingDB
			& $73.22 \pm 1.08$ & $76.83 \pm 0.50$ & $77.48 \pm 1.04$
			& $\mathbf{78.3 \pm 0.8}$ & $77.8 \pm 1.2$ \\
			\bottomrule
		\end{tabular}
	}
	\endgroup
\end{table}

\subsection{Computational Efficiency}
Table~\ref{tab:efficiency} compares the inference time and memory usage of different methods under increasing sample and set sizes. Across all settings, replacing Monte Carlo sampling with the proposed learned relaxation consistently improves computational efficiency. For EquiVSet, the relaxation variant achieves up to $2.6\times$ faster inference and reduces memory usage by up to $1.6\times$. The improvement is more pronounced for INSET, where INSET-$\operatorname{R}$ obtains up to $3.8\times$ speedup and $2.5\times$ lower memory usage. These gains become especially important as $N$ and $|V|$ grow, suggesting that the proposed relaxation effectively alleviates the computational and memory overhead caused by sampling-based gradient estimation.
\begin{table}[h]
	\centering
	\caption{Computational efficiency on the Two Moons dataset under varying sample and set sizes. The average inference time per epoch (ms) and peak memory usage (MB) are reported.
	}
	\label{tab:efficiency}
	
	\begingroup
	\small
	\setlength{\tabcolsep}{7pt}
	\renewcommand{\arraystretch}{1.2}
	
	\begin{tabular}{l|cc|cc|cc}
		\toprule
		\multirow{2}{*}{Method (\#Params)}
		& \multicolumn{2}{c|}{$N$=1000, $V$=100}
		& \multicolumn{2}{c|}{$N$=2000, $V$=200}
		& \multicolumn{2}{c}{$N$=3000, $V$=300} \\
		
		& Time & Memory
		& Time & Memory
		& Time & Memory \\
		\midrule
		
		EquiVSet (0.76M)
		& 178 & 936
		& 686 & 1942
		& 1575 & 3043 \\
		
		{EquiVSet-$\operatorname{R}$ (0.26M)}
		& \textbf{68} & \textbf{569}
		& \textbf{370} & \textbf{1352}
		& \textbf{1021} & \textbf{2468} \\
		
		\midrule
		
		INSET (1.14M)
		& 479 & 2554
		& 1892 & 5368
		& 4420 & 8505 \\
		
		{INSET-$\operatorname{R}$ (0.52M)}
		& \textbf{125} & \textbf{1021}
		& \textbf{732} & \textbf{2483}
		& \textbf{2025} & \textbf{4396} \\
		
		\bottomrule
	\end{tabular}
	
	\endgroup
\end{table}

\subsection{Comparison with Low-Variance Gradient Estimators}
Table~\ref{tab:low_variance_estimators} compares our method with low-variance gradient estimators by increasing the number of Monte Carlo samples. Although using more samples slightly improves performance, it also substantially increases inference time. In contrast, our method achieves the best performance for both EquiVSet and INSET while requiring significantly less inference time. This suggests that directly learning the relaxation provides a more favorable performance-efficiency trade-off than reducing gradient variance through additional sampling.
\begin{table}[h]
	\centering
	\caption{Comparison with low-variance gradient estimators on the Toys dataset. Performance (\%) and average inference time (ms) are reported.
	}
	\label{tab:low_variance_estimators}
	\begingroup
	\small
	\setlength{\tabcolsep}{6pt}
	\renewcommand{\arraystretch}{1.2}
	\begin{tabular}{ll|cccccc}
		\toprule
		\multicolumn{2}{c|}{Number of Samples}
		& 10 & 20 & 30 & 40 & 50 & Ours \\
		\midrule
		\multirow{2}{*}{EquiVSet}
		& MJC
		& $71.6 {\pm} 0.7$
		& $74.0 {\pm} 0.6$
		& $74.2 {\pm} 0.5$
		& $75.3 {\pm} 0.2$
		& $75.9 {\pm} 0.6$
		& $\mathbf{77.9 {\pm} 1.0}$ \\
		& Time
		& $98$
		& $160$
		& $223$
		& $290$
		& $360$
		& $\mathbf{34}$ \\
		\midrule
		\multirow{2}{*}{INSET}
		& MJC
		& $76.8{\pm}1.0$
		& $77.2{\pm}0.5$
		& $76.9{\pm}1.3$
		& $76.9{\pm}0.6$
		& $77.2{\pm}0.4$
		& $\mathbf{79.0{\pm}0.6}$ \\
		& Time
		& $165$
		& $300$
		& $423$
		& $547$
		& $669$
		& $\mathbf{54}$ \\
		\bottomrule
	\end{tabular}
	\endgroup
\end{table}

\subsection{Sensitivity Analysis of Hyperparameters}
\label{appendix:Sensitivity Analysis of Hyperparameters}

\paragraph{Impact of the Gradient Ascent Step Size $\alpha$.}
In $\operatorname{ReSet}$, the step size $\alpha$ for gradient ascent in Eq.~(\ref{add noise})  serves as a critical hyperparameter governing both convergence velocity and numerical stability, directly impacting the final performance.
An excessively large step size may lead to divergence, while a disproportionately small one risks trapping the optimization in suboptimal local minima and unnecessarily prolonging inference time.
We evaluate the performance of our two proposed OS Oracle variants across different datasets under various $\alpha$ configurations, with results illustrated in Figures~\ref{fig:alpha_Equiv} and~\ref{fig:alpha_INSET}. 
Furthermore, in our practical implementation, we set $\alpha$ as a learnable parameter to achieve an adaptive step size. 
This approach alleviates the burden of manual hyperparameter tuning for unseen datasets and enables the model to dynamically calibrate its optimization dynamics for diverse tasks. 
The performance of this self-adaptive configuration is denoted as ``Adaptive" in the corresponding figures.
The empirical results indicate that a fixed step size near $0.5$ generally yields favorable outcomes.
Notably, the adaptive strategy consistently outperforms the fixed-parameter configurations in most scenarios, demonstrating its superior capability in navigating the optimization landscape.
\begin{figure}[h]
	\centering
	\includegraphics[width=1\linewidth]{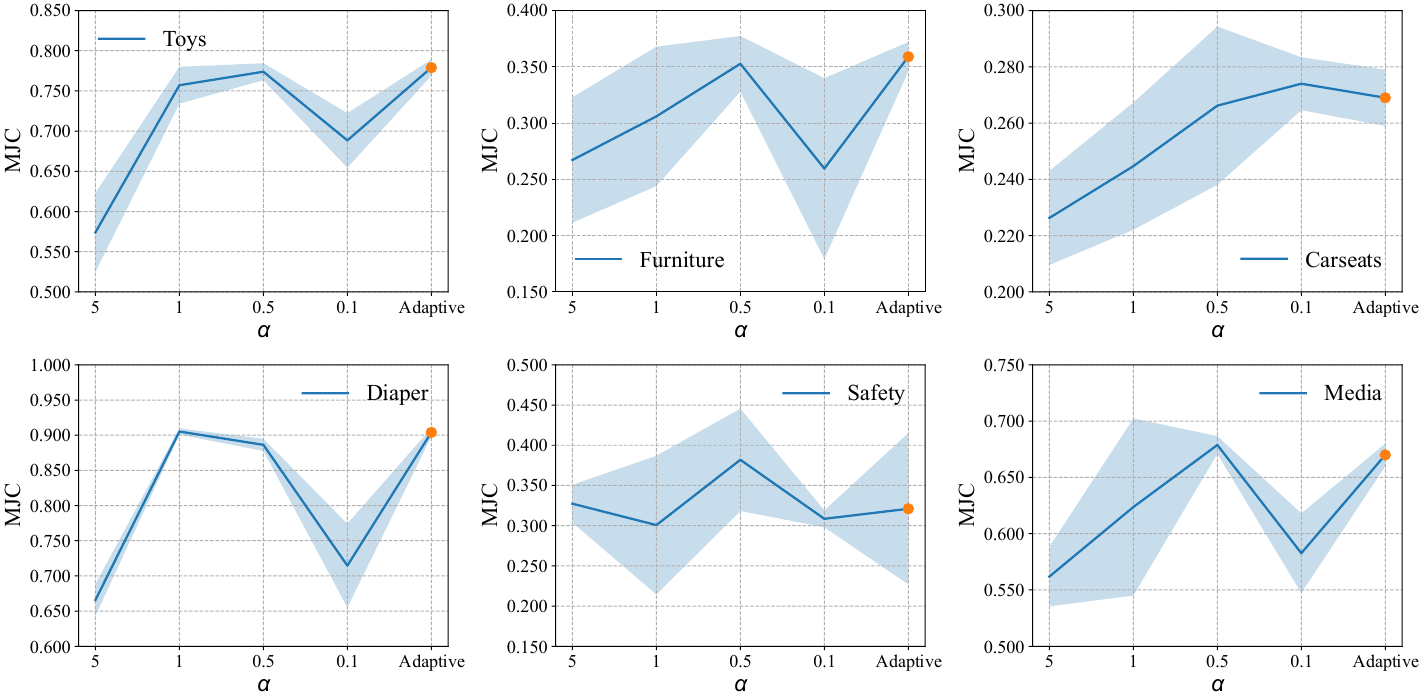}
	\caption{Sensitivity analysis of EquiVSet-$\operatorname{R}$ performance under different values of $\alpha$, where “Adaptive" denotes the variant with $\alpha$ as a learnable parameter.}
	\label{fig:alpha_Equiv}
\end{figure}

\begin{figure}[h]
	\centering
	\includegraphics[width=1\linewidth]{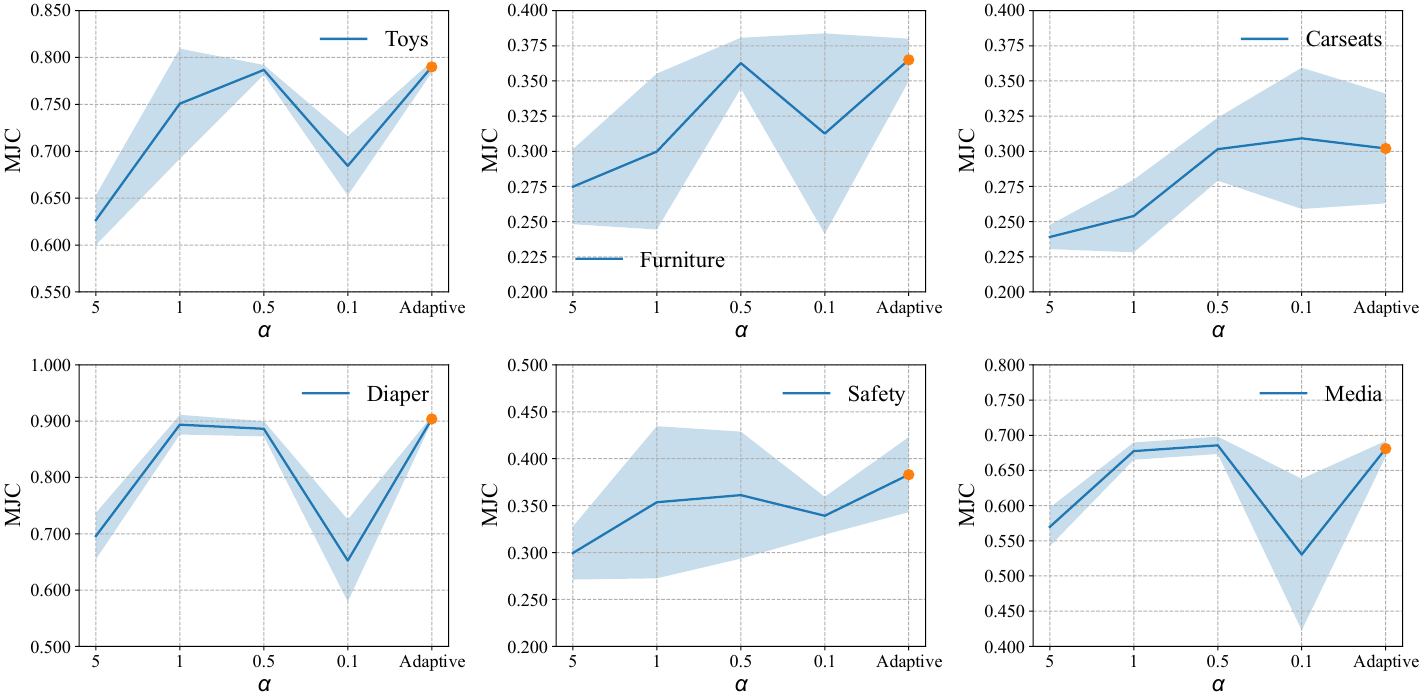}
	\caption{Sensitivity analysis of INSET-$\operatorname{R}$ performance under different values of $\alpha$, where “Adaptive" denotes the variant with $\alpha$ as a learnable parameter.}
	\label{fig:alpha_INSET}
\end{figure}

\paragraph{Impact of the Number of Gradient Ascent Steps $K$.}
In $\operatorname{ReSet}$, the number of gradient ascent steps $K$
determines the extent of iterative refinement in
Eq.~(\ref{add noise}) and directly affects the computational cost.
We evaluate our two proposed OS oracle variants on the six Amazon
datasets with $K\in\{0,1,2,3,4,5\}$, keeping all other settings
fixed and repeating each configuration over five random seeds.
The results are illustrated in Figures~\ref{fig:steps_equiv}
and~\ref{fig:steps_inset}.
When $K=0$, predictions rely solely on the recognition network,
yielding performance close to the DeepSet baseline on Toys,
Diaper, and Media. Notably, a single refinement step guided by
$\nabla_{\bm{\psi}}g_\theta$ increases mean MJC by approximately
$24$--$44$ percentage points on these datasets.
These one-step gains exceed the improvements of both EquiVSet
and INSET over DeepSet reported in Table~\ref{Product Recommendation},
highlighting the benefit of refining the initial predictions.
The empirical results indicate that $K=1$ achieves the highest
mean MJC in eight of the twelve settings, while the remaining
settings favor $K=2$ or $K=3$.
Increasing $K$ further provides no consistent improvement and can
lead to lower performance or greater variability across seeds.\footnote{
	Theorem~\ref{theorem1} analyzes optimization of a fixed surrogate;
	here, the model parameters are learned jointly with the $K$
	unrolled updates.}
Table~\ref{tab:step_runtime} further reports the per-epoch runtime
(ms/epoch), which increases approximately linearly with $K$.
At $K=5$, the runtime is $3.2$--$4.6$ times that at $K=1$
across the evaluated settings.
In our practical implementation, we follow the $K=1$ setting
of EquiVSet~\citep{ou2022learning} and
INSET~\citep{xie2024enhancing} for a fair comparison.
This choice also provides a favorable trade-off between
predictive performance and computational cost.

\begin{table}[h]
	\centering
	\caption{Per-epoch runtime (ms/epoch) under different numbers
		of gradient ascent steps $K$ on six Amazon datasets.}
	\label{tab:step_runtime}
	\begingroup
	\small
	\setlength{\tabcolsep}{6pt}
	\renewcommand{\arraystretch}{1.2}
	\begin{tabular}{ll|cccccc}
		\toprule
		\multicolumn{2}{c|}{Number of Steps $K$}
		& $0$ & $1$ & $2$ & $3$ & $4$ & $5$ \\
		\midrule
		\multirow{6}{*}{EquiVSet-$\operatorname{R}$}
		& Toys
		& $21.76$ & $61.46$ & $108.07$
		& $160.98$ & $188.88$ & $239.91$ \\
		& Furniture
		& $1.90$ & $8.94$ & $12.37$
		& $17.63$ & $23.20$ & $28.53$ \\
		& Carseats
		& $5.24$ & $17.43$ & $26.47$
		& $37.37$ & $48.60$ & $59.86$ \\
		& Diaper
		& $48.27$ & $127.68$ & $235.56$
		& $337.22$ & $442.61$ & $547.42$ \\
		& Safety
		& $2.11$ & $8.63$ & $12.06$
		& $17.43$ & $22.71$ & $27.94$ \\
		& Media
		& $16.51$ & $35.77$ & $57.20$
		& $82.78$ & $108.44$ & $134.37$ \\
		\midrule
		\multirow{6}{*}{INSET-$\operatorname{R}$}
		& Toys
		& $27.39$ & $94.22$ & $186.31$
		& $258.39$ & $340.35$ & $423.19$ \\
		& Furniture
		& $2.12$ & $11.76$ & $21.02$
		& $30.75$ & $40.25$ & $49.84$ \\
		& Carseats
		& $5.20$ & $24.40$ & $44.18$
		& $64.58$ & $85.42$ & $106.11$ \\
		& Diaper
		& $48.51$ & $217.94$ & $413.51$
		& $610.61$ & $804.38$ & $1001.57$ \\
		& Safety
		& $2.27$ & $11.20$ & $20.86$
		& $30.30$ & $39.83$ & $49.31$ \\
		& Media
		& $15.25$ & $53.91$ & $101.21$
		& $149.41$ & $197.64$ & $247.89$ \\
		\bottomrule
	\end{tabular}
	\endgroup
\end{table}

\begin{figure}[htbp]
	\centering
	\includegraphics[width=\linewidth]{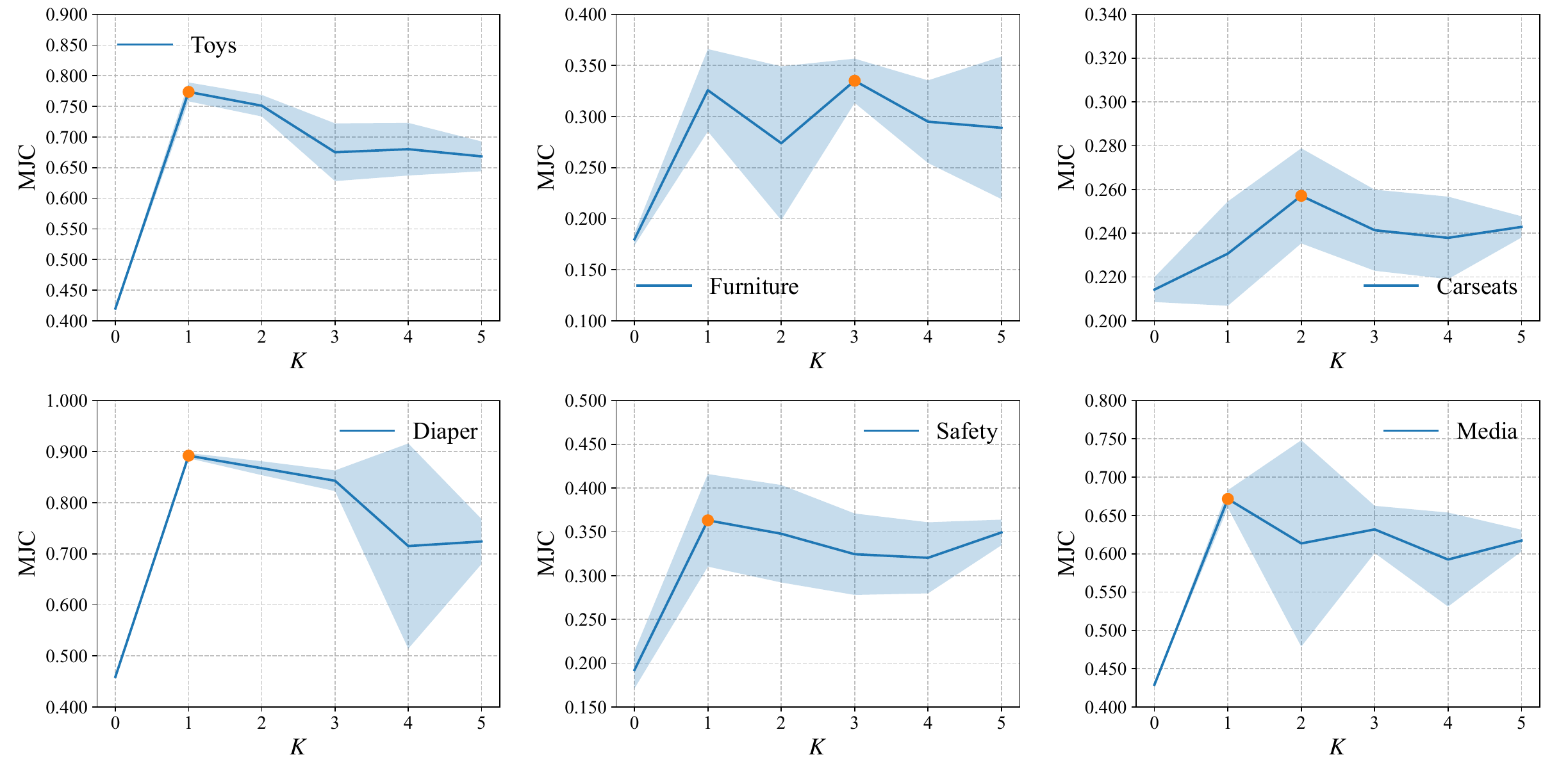}
	\caption{Sensitivity analysis of EquiVSet-$\operatorname{R}$
		performance under different values of $K$.
		Curves and shaded regions show the mean MJC and one standard
		deviation over five random seeds. Orange markers indicate
		the highest mean in each dataset.}
	\label{fig:steps_equiv}
\end{figure}

\begin{figure}[htbp]
	\centering
	\includegraphics[width=\linewidth]{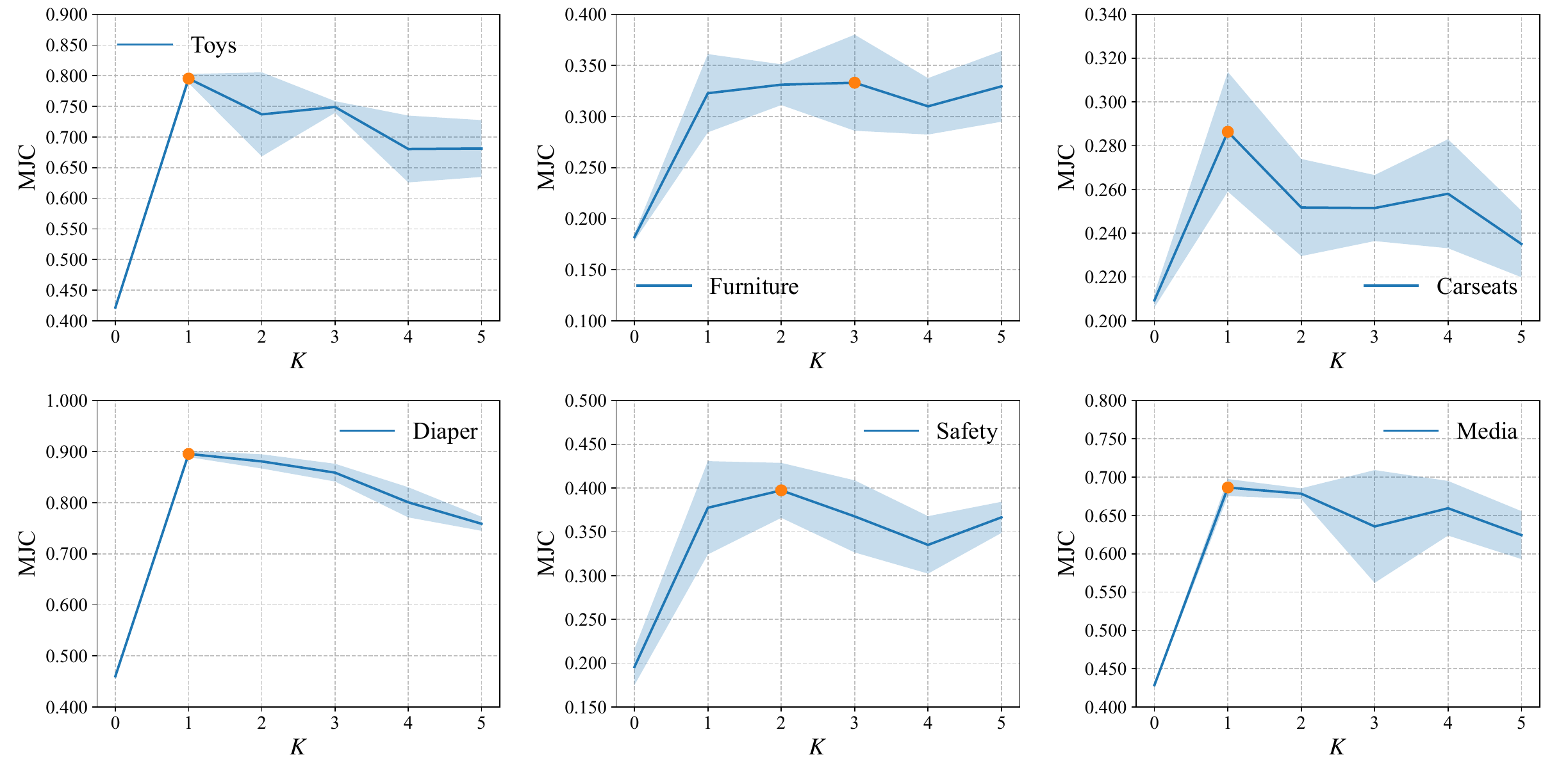}
	\caption{Sensitivity analysis of INSET-$\operatorname{R}$
		performance under different values of $K$.
		Curves and shaded regions show the mean MJC and one standard
		deviation over five random seeds. Orange markers indicate
		the highest mean in each dataset.}
	\label{fig:steps_inset}
\end{figure}

\paragraph{Evolution of the Learnable Temperature $T$.}
In $\operatorname{ReSet}$, the temperature $T$ in
Eq.~(\ref{reformulated Helmholtz free energy}) controls the balance
between expected energy and entropy.
In our implementation, $T$ is treated as a learnable parameter
and optimized jointly with the network weights.
We examine the evolution of the learned temperature during training
for our two proposed OS oracle variants on Toys, Carseats, and Media,
as shown in Figure~\ref{fig:temperature_trends}.
For EquiVSet-$\operatorname{R}$, the temperature decreases slightly
at the beginning of training across all three datasets.
For INSET-$\operatorname{R}$, it decreases throughout training on all
three datasets.
Overall, the learned temperature remains close to its initialization
while exhibiting distinct trends across datasets and model variants.
These modest adjustments illustrate how temperature learning
fine-tunes the balance between expected energy and entropy during training.

\begin{figure}[htbp]
	\centering
	\includegraphics[width=\linewidth]{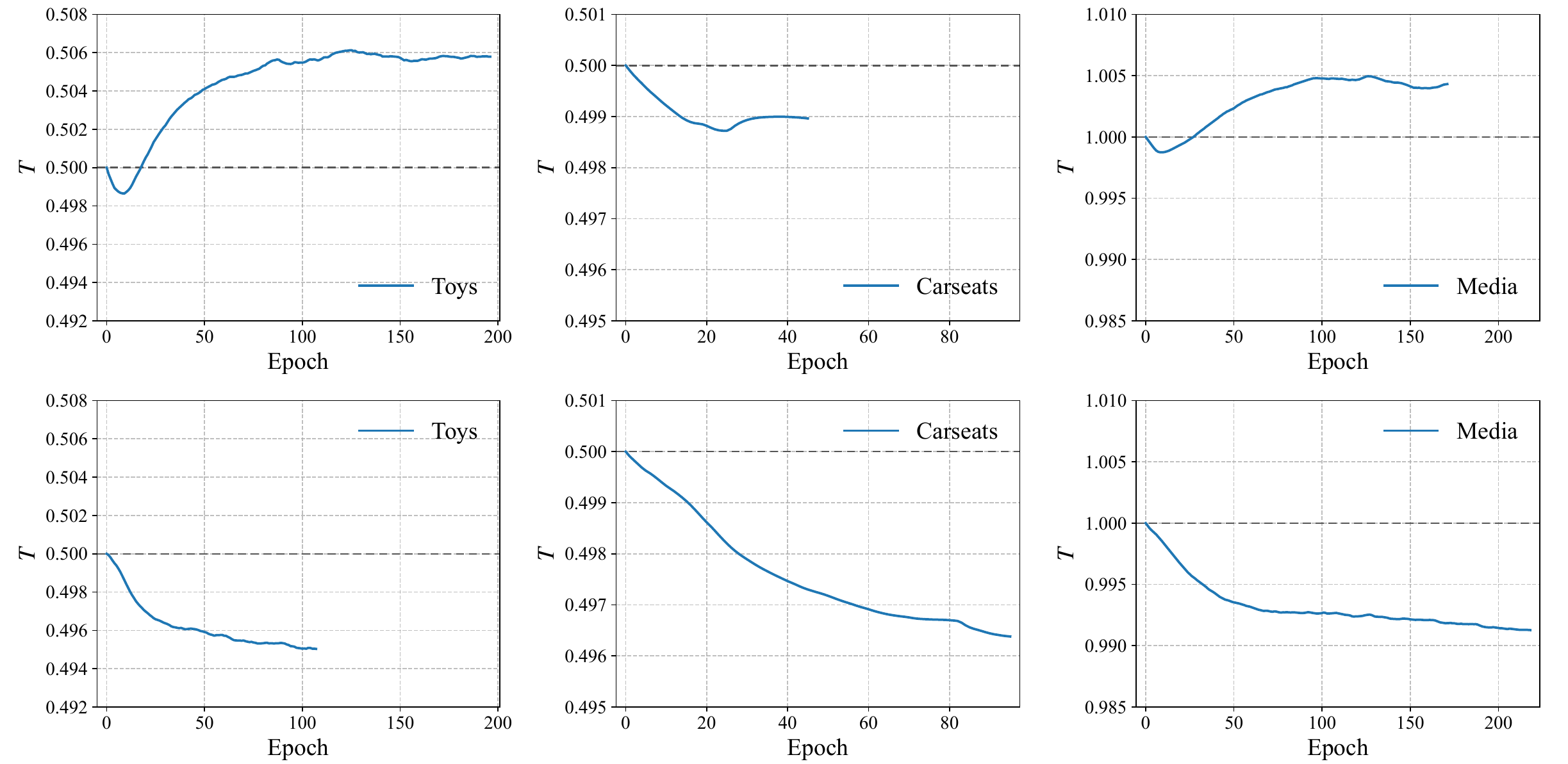}
	\caption{Evolution of the learnable temperature $T$ during
		training. The top and bottom rows correspond to
		EquiVSet-$\operatorname{R}$ and INSET-$\operatorname{R}$,
		respectively;
		Dashed horizontal lines indicate the initial temperature.}
	\label{fig:temperature_trends}
\end{figure}

\subsection{Set Size Transferability Analysis}\label{appendix:Set Size Transferability Analysis}
Following \citet{ou2022learning},  we experiment to understand the pattern of set size transfer ability.
In this experiment, we train the models using fixed sizes of the ground set but test the trained model on different sizes. 
We fix the size of OS oracle $S^*$ to be $10$, and train the model with ground set $V$ of size $100$. After training, we test it using varying sizes of ground set in the range of $\{200, 400, 600, 800, 1000\}$. 
The experiments with two OS Oracle variants are conducted on the Two-Moons and Gaussian-Mixture datasets, respectively, with the results shown in Figures~\ref{fig:varying_Equiv} and~\ref{fig:varying_INSET}. 
As illustrated in the results, both proposed OS Oracle variants exhibit an observable degree of set-size transferability. When trained on a fixed ground set size and evaluated on increasingly larger sets, the models experience a gradual performance degradation, which aligns with theoretical expectations regarding the increased complexity of the search space.
This decline is relatively marginal on the Gaussian Mixture dataset but becomes more pronounced on the Two Moons dataset. We attribute this discrepancy to the intrinsic task complexity; the non-linear, intertwined geometry of the Two-Moons distribution likely poses a greater challenge for zero-shot size transfer compared to the more clustered Gaussian-Mixture distribution.

\begin{figure}[h]
	\centering
	\includegraphics[width=1\linewidth]{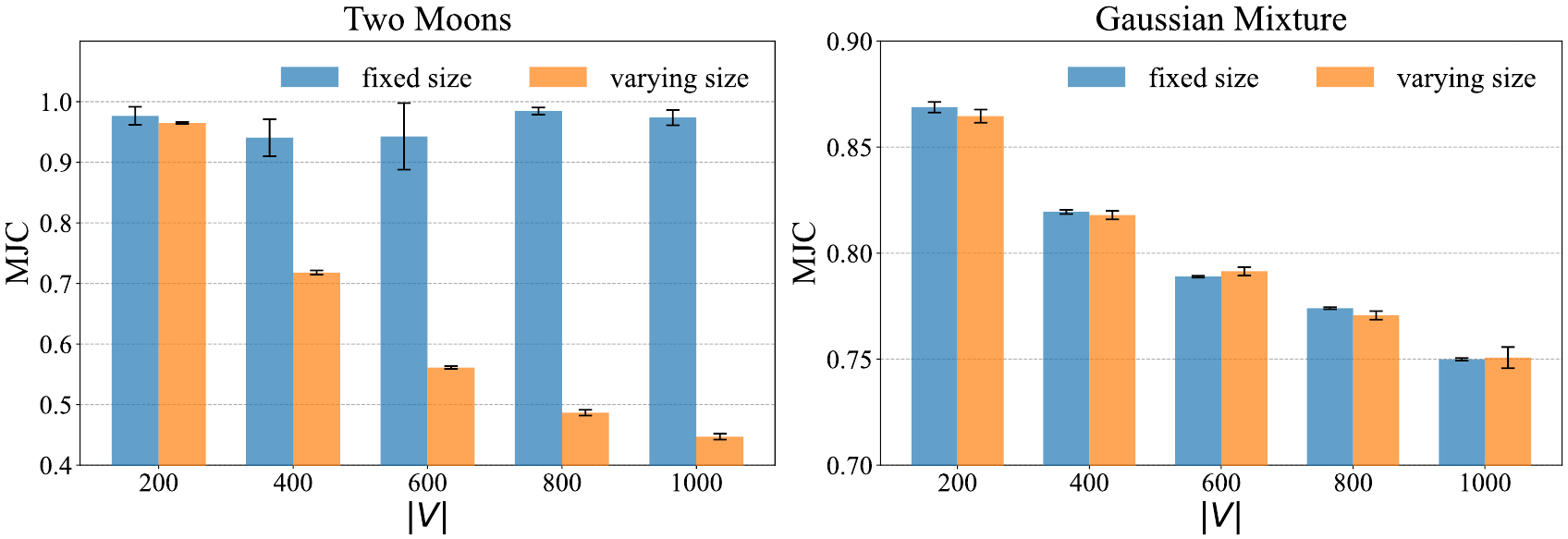}
	\caption{Synthetic results of EquiVSet-$\operatorname{R}$ under the set-size transferability setting are reported. The blue bars correspond to scenarios where the ground set size used at test time matches that employed during training, whereas the yellow bars indicate evaluations conducted with mismatched ground set sizes at inference.}
	\label{fig:varying_Equiv}
\end{figure}

\begin{figure}[h]
	\centering
	\includegraphics[width=1\linewidth]{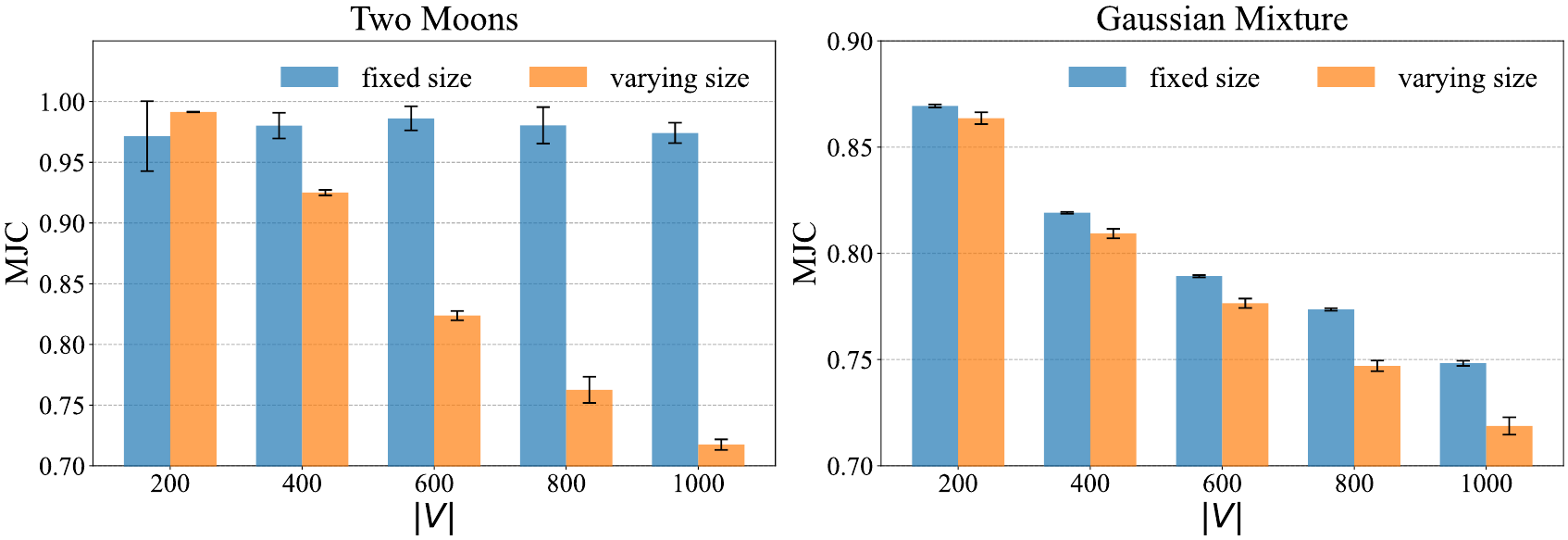}
	\caption{Synthetic results of INSET-$\operatorname{R}$ under the set-size transferability setting are reported. The blue bars correspond to scenarios where the ground set size used at test time matches that employed during training, whereas the yellow bars indicate evaluations conducted with mismatched ground set sizes at inference.}
	\label{fig:varying_INSET}
\end{figure}

\section{Related Work}\label{appendix:Related Work}
\subsection{Set Function Learning}
Many existing studies focus on learning an unknown function to predict the value of a given set \citep{wendler2019powerset, wendler2021learning, de2022neural}, commonly referred to as a function-value (FV) oracle.
\citet{zaheer2017deep} introduced the DeepSets architecture to model permutation-invariant and permutation-equivariant functions for set-valued prediction.
Building on DeepSets, \citet{lee2019set} incorporated the Transformer architecture to capture relational dependencies among elements within a set.
Additionally, several works have explored deep architectures for submodular set functions \citep{bilmes2017deep, dolhansky2016deep, balcan2018submodular}.
It is important to note that these methods typically require a substantial amount of supervision, which poses challenges for their practical deployment.

\textbf{Optimal Subset Oracle.}
\citet{ou2022learning} propose to learn latent set functions using optimal subsets rather than explicit function values as supervision, formulating the problem under an energy-based modeling framework with parameters updated via variational inference.
Building upon this framework, \citet{xie2024enhancing, xie2024horse} further incorporate contextual information from the superset and hierarchical attention mechanisms.
\citet{ozcan2025learning} further improve the fixed-point iterations in variational inference through implicit differentiation.
However, these methods require unrolled Monte Carlo sampling for gradient estimation during iterative optimization, rendering the training process inefficient.

\subsection{Energy-based Modeling}
Energy-based learning \citep{lecun2006tutorial} provides a classical framework for modeling underlying data distributions, in which model parameters are learned via principled techniques such as contrastive divergence \citep{carreira2005contrastive, tieleman2008training, du2019implicit} and score matching \citep{hyvarinen2005estimation, vincent2011connection, song2020sliced}.
In the framework of \citet{ou2022learning}, energy-based model is employed to parameterize the set mass function.
Building upon this formulation, we reinterpret the associated variational inference procedure from the perspective of free energy estimation \citep{he2025feat} and energy minimization \citep{gladstone2025energy}.
This viewpoint aligns closely with a broad line of work that casts learning and inference as energy minimization problems \citep{du2022learning, du2024learning, wang2025equilibrium}.

\subsection{Submodular Function Learning}
In practical applications, learning set functions typically necessitates the incorporation of structural priors. The underlying objectives are often assumed to be submodular, \textit{i.e.}, exhibiting the diminishing returns property. Unlike prior work that predominantly relies on function value oracles for parameter estimation \citep{dolhansky2016deep, bilmes2017deep, djolonga2017differentiable, kothawade2020deepsubmodularnetworksextractive, de2022neural, bhatt2024deepsubmodularperipteralnetworks}, this work learns latent set functions characterized by more general weak submodular properties through a more practical and computationally efficient optimal subset oracle.

\end{document}